\documentclass[a4,journal]{IEEEtran}
\IEEEoverridecommandlockouts 
\usepackage[utf8]{inputenc}
\usepackage{amsmath}
\usepackage{amssymb}
\usepackage{amsthm}
\usepackage[position=top]{subfig}
\usepackage{amsfonts}
\usepackage{mathtools}
\usepackage{url}
\usepackage{subcaption}
\usepackage{bm}
\usepackage{enumerate}
\usepackage{xcolor}
\usepackage{graphicx}
\usepackage{subfig}
\usepackage[font=small]{caption}

\usepackage{algorithm}
\usepackage{algpseudocode}
\usepackage{optidef} 
\usepackage{physics} 
\usepackage{siunitx} 
\usepackage{pgfplots}
\pgfplotsset{compat=newest}
\usepgfplotslibrary{groupplots}
\usepgfplotslibrary{dateplot}

\newlength\figureheight
\newlength\figurewidth

\title{Cyclic Nullspace Coordination: Perpetual Flight of Aerial Carriers for Static Suspension}

\author{Chiara Gabellieri$^{*}$, Yaolei Shen$^*$, Martina Paolucci$^\dagger$, and Antonio Franchi$^{*,\dagger}$
\thanks{$^*$ Robotics and Mechatronics Department, Electrical Engineering,  Mathematics, and Computer Science (EEMCS) Faculty, University of Twente, 7500 AE Enschede, The Netherlands. {\tt\footnotesize c.gabellieri@utwente.nl}, {\tt\footnotesize schol@r-franchi.eu}}
\thanks{$^\dagger$ Department of Computer, Control and Management Engineering, Sapienza University of Rome, 00185 Rome, Italy, {\tt\footnotesize schol@r-franchi.eu}}
\thanks{This work was partially funded by the Horizon Europe agreement no. 101120732 (Autoassess).}}
\begin{document}

\newcommand\red[1]{{\textcolor{red}{#1}}}
\newcommand\blue[1]{{\textcolor{blue}{#1}}}
\newcommand\green[1]{{\textcolor{green}{#1}}}


\newcommand{\vect}[1]{\bm{#1}}		
\newcommand{\matr}[1]{\bm{#1}}		
\newcommand{\nR}[1]{\mathbb{R}^{#1}}		
\newcommand{\nN}[1]{\mathbb{N}^{#1}}
\newcommand{\SO}[1]{SO({#1})}		
\newcommand{\define}{:=}			
\newcommand{\modulus}[1]{\left| #1 \right|}	
\newcommand{\matrice}[1]{\begin{bmatrix} #1 \end{bmatrix}}	
\newcommand{\smallmatrice}[1]{\left[\begin{smallmatrix} #1 \end{smallmatrix}\right]}	
\newcommand{\cosp}[1]{\cos \left( #1 \right)}	
\newcommand{\sinp}[1]{\sin \left( #1 \right)}	

\newcommand{\sgn}[1]{\text{sgn}\left( #1 \right)}			
\newcommand{\atanTwo}[1]{{\rm atan2}\left( #1\right)}		
\newcommand{\acotTwo}[1]{{\rm acot2}\left( #1\right)}		
\newcommand{\upperRomannumeral}[1]{\uppercase\expandafter{\romannumeral#1}}	
\newcommand{\lowerromannumeral}[1]{\romannumeral#1\relax}
\newcommand{\vSpace}{\;\,}
\newcommand{\image}[1]{\text{Im}\left( #1 \right)}
\newcommand{\pinv}{\dagger}
\newcommand{\diag}[1]{\text{diag}\left( #1 \right)}
\newcommand{\unitOfMeasure}[1]{\; [ \rm #1]}
\newcommand{\Ker}{{\rm Null}}
\newcommand{\ith}[1]{#1^{\rm{th}}}
\newcommand{\fig}{Fig.~}	
\newcommand{\eqn}{Eq.~}	
\newcommand{\tab}{Tab.~}	
\newcommand{\cha}{Chap.~}	
\newcommand{\sect}{Sec.~}	
\newcommand{\theo}{Theorem~}	


\renewcommand{\frame}{\mathcal{F}}		
\newcommand{\origin}{O}						
\newcommand{\vX}{\vect{x}}					
\newcommand{\vY}{\vect{y}}					
\newcommand{\vZ}{\vect{z}}					
\newcommand{\vE}[1]{\vect{e}_{#1}}			
\newcommand{\vV}{\vect{v}}					
\newcommand{\pos}{\vect{p}}				
\newcommand{\dpos}{\dot{\vect{p}}}		
\newcommand{\ddpos}{\ddot{\vect{p}}}	
\newcommand{\rotMat}{\matr{R}}			
\newcommand{\drotMat}{\dot{\matr{R}}}			
\newcommand{\rotMatVectAngle}[2]{\rotMat_{#1}(#2)}	
\newcommand{\angVel}{\vect{\omega}}				
\newcommand{\angAcc}{\dot{\vect{\omega}}}		
\newcommand{\vZero}{\vect{0}}				
\newcommand{\gravity}{\vect{g}}			
\renewcommand{\skew}[1]{\matr{S}(#1)}				
\newcommand{\eye}[1]{\matr{I}_{#1}}		
\newcommand{\roll}{\phi}		
\newcommand{\pitch}{\theta}		
\newcommand{\pitchDes}{\bar{\theta}}		
\newcommand{\pitchEq}{\theta^{eq}}		
\newcommand{\yaw}{\psi}		
\newcommand{\yawDes}{\bar{\psi}}		
\newcommand{\yawEq}{\psi^{eq}}		
\newcommand{\eulerAngles}{\vect{\eta}}

\newcommand{\frameW}{\frame_W}			
\newcommand{\originW}{\origin_W}		
\newcommand{\xW}{\vX_W}				
\newcommand{\yW}{\vY_W}				
\newcommand{\zW}{\vZ_W}				

\newcommand{\frameB}{\frame_B}			
\newcommand{\originB}{\origin_B}			
\newcommand{\xB}{\vX_B}				
\newcommand{\yB}{\vY_B}				
\newcommand{\zB}{\vZ_B}				
\newcommand{\pL}{\pos_L}			
\newcommand{\pLW}{\prescript{W}{}{\pL}} 
\newcommand{\dpL}{\dpos_L}			
\newcommand{\ddpL}{\ddpos_L}		
\newcommand{\rotMatL}{\rotMat_L}
\newcommand{\drotMatL}{\drotMat_L}

\newcommand{\rotMatLEquilib}{\rotMat^{eq}_{L}}
\newcommand{\rotMatLInc}{\hat{\rotMat}_L}
\newcommand{\rotMatLW}{\prescript{W}{}{\rotMat_L}}	
\newcommand{\angVelL}{\angVel_L}		
\newcommand{\angAccL}{\angAcc_L}		
\newcommand{\massL}{{m_L}}
\newcommand{\massLU}{\hat{{m}}_L}				
\newcommand{\inertiaL}{\matr{J}_L}	
\newcommand{\InertiaL}{\matr{M}_L}	
\newcommand{\coriolisL}{\vect{c}_L}	
\newcommand{\gravityL}{\vect{g}_L}	
\newcommand{\graspL}{\matr{G}}		
\newcommand{\dampingL}{\matr{B}_L}	
\newcommand{\f}{\vect{f}}


\newcommand{\length}[1]{{l}_{0#1}}
\newcommand{\lengthU}[1]{\hat{{l}}_{0#1}}	
\newcommand{\springCoeff}[1]{{k}_{#1}}
\newcommand{\springCoeffU}[1]{\hat{{k}}_{#1}}
\newcommand{\cableForce}[1]{\vect{f}_{#1}}
\newcommand{\cableForceEquilib}[1]{\vect{f}^{eq}_{#1}}
\newcommand{\cableForceInc}[1]{\hat{\vect{f}}_{#1}}			
\newcommand{\cableForceU}[1]{\hat{\vect{f}}_{#1}}
\newcommand{\cableAttitudeNorm}[1]{\vect{n}_{#1}}
\newcommand{\cableAttitude}[1]{\vect{l}_{#1}}
\newcommand{\dcableAttitude}[1]{\dot{\vect{l}}_{#1}}	

\newcommand{\condZero}{\xi}
\newcommand{\anchorPoint}[1]{B_{#1}}			
\newcommand{\anchorPos}[1]{\vect{b}_{#1}}		
\newcommand{\anchorLength}[1]{{b}_{#1}}	
\newcommand{\anchorLengthU}[1]{\hat{{b}}_{#1}}		
\newcommand{\anchorPosL}[1]{\prescript{L}{}{\vect{b}}_{#1}}		
\newcommand{\anchorPosLU}[1]{\prescript{L}{}{\hat{\vect{b}}}_{#1}}		

\newcommand{\robotPoint}[1]{A_{#1}}				
\newcommand{\robotPos}[1]{\vect{a}_{ #1}}		
\newcommand{\robotPosP}[1]{\prescript{P}{}{\vect{a}}_{#1}}		
\newcommand{\angleCable}[1]{\alpha_{#1}}		
\newcommand{\angleCables}{\vect{\alpha}}		
\newcommand{\tension}[1]{t_{#1}}				
\newcommand{\tensionMax}[1]{\overline{f}_{L#1}}				
\newcommand{\tensionMin}[1]{\underline{f}_{L#1}}				
\newcommand{\cableForces}{\cableForce{}}		

\newcommand{\frameR}[1]{\frame_{R #1}}			
\newcommand{\originR}[1]{O_{R #1}}					
\newcommand{\xR}[1]{\vX_{R #1}}								
\newcommand{\yR}[1]{\vY_{R #1}}								
\newcommand{\zR}[1]{\vZ_{R #1}}								
\newcommand{\pR}[1]{\pos_{R #1}}						
\newcommand{\dpR}[1]{\dpos_{R #1}}					
\newcommand{\ddpR}[1]{\ddpos_{R #1}}				
\newcommand{\uR}[1]{\vect{u}_{R #1}}				
\newcommand{\pRW}[1]{\prescript{W}{}{\pR{#1}}} 	
\newcommand{\rotMatR}[1]{\rotMat_{R #1}}			
\newcommand{\thrust}[1]{\vect{f}_{R #1}}		
\newcommand{\maxThrust}[1]{h_{#1}}				
\newcommand{\thrustIJ}{\thrust{ij}}				
\newcommand{\maxThrustIJ}{\maxThrust{ij}}		
\newcommand{\gravityIJ}{\vect{g}_{ij}}			
\newcommand{\massR}[1]{m_{R#1}}					
\newcommand{\inertiaR}[1]{\vect{J}_{R#1}}		

\newcommand{\dampingA}[1]{\matr{B}_{A#1}}		
\newcommand{\springA}[1]{\matr{K}_{A#1}}		
\newcommand{\inertiaA}[1]{\matr{M}_{A#1}}		

\newcommand{\paramA}[1]{\vect{\pi}_{A#1}}			

\newcommand{\config}{\vect{q}}					
\newcommand{\dconfig}{\vect{v}}			
\newcommand{\ddconfig}{\dot{\dconfig}}			
\newcommand{\configR}{\config_R}					
\newcommand{\dconfigR}{\dconfig_R}				
\newcommand{\ddconfigR}{\ddconfig_R}			
\newcommand{\configL}{\config_L}					
\newcommand{\dconfigL}{\dconfig_L}				
\newcommand{\ddconfigL}{\ddconfig_L}			
\newcommand{\state}{\vect{x}}						
\newcommand{\dynamicModelFun}{m}					

\newcommand{\configEq}{\bar{\config}}			
\newcommand{\configLEq}{\bar{\config}_L}		
\newcommand{\configREq}{\bar{\config}_R}		
\newcommand{\paramAEq}[1]{\bar{\vect{\pi}}_{A#1}}			
\newcommand{\paramAEqInc}[1]{\hat{\bar{\vect{\pi}}}_{A#1}}
\newcommand{\paramAEqU}[1]{\hat{\bar{\vect{\pi}}}_{A#1}}
\newcommand{\pLEq}{\bar{\pos}_L}

\newcommand{\pLEquilib}{{\pos}^{eq}_L}
\newcommand{\rotMatLEq}{\bar{\rotMat}_L}		
\newcommand{\paramASetEq}{{\Pi}_{A}(\configLEq)}				
\newcommand{\paramASetEqPrime}{{\Pi}_{A}(\configLEq')}				
\newcommand{\configRSetEq}{\mathcal{P}_{R}}			
\newcommand{\cableForceEq}[1]{\bar{\vect{f}}_{#1}}
\newcommand{\cableForceEqInc}[1]{\hat{\bar{\vect{f}}}_{#1}}				
\newcommand{\cableForcesEq}{\bar{\vect{f}}}	
\newcommand{\cableForcesEqInc}{\hat{\bar{\vect{f}}}}					
\newcommand{\cableForcesSetEq}{\mathcal{F}_{L}}		
\newcommand{\nullGrasp}{\vect{r}_L}				
\newcommand{\internalTension}{t_L}				
\newcommand{\internalForceDir}{\vect{n}_L}	
\newcommand{\internalForceDirL}{\prescript{L}{}{\vect{n}}_L}	
\newcommand{\pREq}[1]{\bar{\pos}_{R #1}}	
\newcommand{\pREqInc}[1]{\hat{\bar{\pos}}_{R #1}}		

\newcommand{\pREquilib}[1]{\pos^{eq}_{R #1}}
\newcommand{\configRparamASetEq}{\mathcal{S}(\configLEq)}
\newcommand{\configRparamAEq}{\bar{\vect{s}}}
\newcommand{\configRparamA}{{\vect{s}}}
\newcommand{\configSetEq}{\mathcal{Q}(\internalTension,\configLEq)}
\newcommand{\configSetEqZero}{\mathcal{Q}(0,\configLEq)}
\newcommand{\configSetEqZeroi}[1]{\mathcal{Q}_{#1}(0,\configLEq)}

\newcommand{\configSetEqPlus}{\mathcal{Q}^+(\internalTension,\configLEq)}
\newcommand{\configSetEqMinus}{\mathcal{Q}^-(\internalTension,\configLEq)}

\newcommand{\configRLSetEq}{\mathcal{R}(\internalTension,\configLEq)}
\newcommand{\configRLSetEqZero}{\mathcal{R}(0,\configLEq)}

\newcommand{\errorpREq}[1]{\vect{e}_{R#1}}	
\newcommand{\errorPL}{{\vect{e}_{p}}_L}

\newcommand{\screwJacobian}{\matr{J}(\genCoord)}		
\newcommand{\vectI}{\vect{v}_i}
\newcommand{\vectII}{{v}_{i,i}}
\newcommand{\minTensionResolving}{\underline{t}_i(\genCoord)}		
\newcommand{\maxTensionResolving}{\overline{t}_i(\genCoord)}		

\newcommand{\stateEq}{\bar{\state}}
\newcommand{\invariantSet}{\Omega_{\alpha}}		
\newcommand{\invariantSetZero}{\Omega_{0}}		
\newcommand{\dVZeroSet}{\mathcal{E}}				
\newcommand{\stateSetEq}{\mathcal{X}(\internalTension,\configLEq)}
\newcommand{\stateSetEqZero}{\mathcal{X}(0,\configLEq)}
\newcommand{\stateSetEqZeroi}[1]{\mathcal{X}_{#1}(0,\configLEq)}
\newcommand{\stateSetEqPlus}{\mathcal{X}^+(\internalTension,\configLEq)}
\newcommand{\stateSetEqMinus}{\mathcal{X}^-(\internalTension,\configLEq)}
\newcommand{\stateSetEqPlusPrime}{\mathcal{X}^+(\internalTension',\configLEq)}
\newcommand{\lyapunovFun}{V(\state)}				

\newcommand{\Vadd}{V_R(\state)}				
\newcommand{\dVadd}{\dot{V}_1(\state)}	

\newcommand{\dlyapunovFun}{\dot{V}(\state)}				
\newcommand{\maxInvariantSet}{\mathcal{M}}

\newcommand{\inp}{\vect{u}}
\newcommand{\out}{\vect{y}}
\newcommand{\outputFunction}{\vect{\Phi}(\out)}

\newcommand{\displacement}{\vect{d}}
\newcommand{\pREqIncRef}[1]{\pREqInc{#1}^r}
\newcommand{\rotMatRRef}[1]{\rotMatR{#1}^r}
\newcommand{\ut}{u(t)}

\newcommand{\ioneton}{i=1,\ldots,n}

\newtheorem{problem}{Problem}
\newtheorem{problem*}{Problem}
\newtheorem{prop}{Proposition}
\newtheorem{lemma}{Lemma}
\newtheorem{example}{Example}

\newtheorem{fact}{Fact}
\newtheorem{result}{Result}
\newtheorem{assumption}{Assumption}

\theoremstyle{remark}
\newtheorem{remark}{Remark}
\maketitle


\begin{abstract}
    This work demonstrates that the non-stop flights of three or more carriers are compatible with holding a constant pose of a cable-suspended load. It also presents an algorithm for generating the carriers' coordinated non-stop trajectories.
The proposed method builds upon two pillars: (1) the choice of $n$ special linearly independent directions of internal forces within the $3n-6$-dimensional nullspace of the grasp matrix of the load, chosen as the edges of a Hamiltonian cycle on the graph that connects the cable attachment points on the load. Adjacent pairs of directions are used to generate $n$ forces evolving on distinct 2D affine subspaces, despite the attachment points being generically in 3D; (2) the construction of elliptical trajectories within these subspaces by mapping, through appropriate graph coloring, each edge of the Hamiltonian cycle to a periodic coordinate while ensuring that no adjacent coordinates exhibit simultaneous zero derivatives. Combined with conditions for load statics and attachment point positions, these choices ensure that each of the $n$ force trajectories projects onto the corresponding cable constraint sphere with non-zero tangential velocity, enabling perpetual motion of the carriers while the load is still.
The work provides a scalable constructive design for any $n\!\ge\!3$ with tuning guidelines, quantifies sensitivity and single-carrier failures, and provides a fixed-wing–compatible planner that preserves load statics under speed/bank/flight-path constraints.
The theoretical findings are validated through simulations and laboratory experiments with quadrotor UAVs. 
\end{abstract}

\maketitle

\section{Introduction}\label{sec:intro}
Aerial manipulation has been intensely studied thanks also to the potential impact in many different real-world use cases. Interesting applications are, e.g., parcel delivery and assembly tasks in the construction field, to name a few~\cite{ruggiero2018aerial}. 
Various aerial manipulation concepts have been proposed in the literature ~\cite{ollero2021past}, but  
lightweight manipulation tools such as simple low-cost cables have often been preferred ~\cite{drones8020035}. Not only are they especially suitable due to the flying carriers' payload limitations, but, if attached ideally to the carrier's Center of Mass (CoM), they do not perturb the attitude dynamics of the carrier itself.

When disregarding the full-pose control of the suspended load, a single carrier may be considered. In the literature, we find examples of a slung load suspended below a single multi-rotor~\cite{sreenath2013trajectory, pereira2016slung} or a small helicopter~\cite{bernard2009generic}.

\textit{Cooperative} aerial manipulation allows for going beyond the single-carrier payload and avoids single points of failure. 
 Among the examples of cooperative aerial manipulation with cables, uncrewed helicopters have been used in~\cite{bernard2011autonomous}, while multirotor UAVs (Uncrewed Aerial Vehicles) have been more generally considered~\cite{michael2011cooperative, fink2011planning, li2021cooperative,gassner2017dynamic,wahba2024efficient, goodman2022geometric}.  The two-multirotor scenario has been largely studied, especially for the manipulation of bar-shaped loads~\cite{gabellieri2023equilibria,pereira2018asymmetric,8995928}. 
 However, three is the minimum number of carriers that can control the full pose of a cable-suspended generic rigid body~\cite{jiang2012inverse}; such a 3-carrier system has been studied, e.g., in~\cite{9508879,michael2011cooperative, li2021cooperative,masone2016}.

As it has been shown, e.g., in~\cite{leutenegger2016flying}, fixed-wing UAVs have a longer flight endurance than multirotor UAVs. However, they cannot stop in mid-air. 
 Convertible UAVs, such as tail sitters, have intermediate flight endurance. They have the feature of transitioning between VTOL platforms and fixed-wing ones, so as to exploit the benefit of the latter and multirotors, as well. An example of such a design is in proposed~\cite{leutenegger2016flying}. Note that, when flying horizontally, also these platforms are also constrained to follow non-stop trajectories. 

The compatibility of non-stop flights with cable manipulation tasks has been theorized in \cite{williams2009dynamics}, and \textit{point-mass-load} lifting using fixed-wing aircraft has been shown in \cite{quenneville2023experimental}. \cite{foss2026energy} demonstrates the potential energy efficiency of non-stop trajectories in manipulating a point-mass load even with standard quadrotor UAVs. More recently, the compatibility of non-stop flights with the cooperative manipulation
of a \textit{rigid-body} pose has been explored. These systems might enable energy-efficient aerial manipulation, seamlessly integrating long-distance transport with precise control of the load's pose.

This work focuses on the challenging scenario of maintaining a constant rigid-body-load pose while the carriers move. Ensuring a fixed load pose during transportation is critical in many practical situations, such as at the destination or to enhance obstacle avoidance and safety. An abstract depiction of the future scenario made possible by this research is illustrated in Fig.~\ref{fig:fig1}.

The non-stop flight dynamics of the carriers make it challenging to determine whether they can hold a suspended load in static equilibrium while all loitering with a nonzero forward speed.
In previous  work~\cite{gabellieri2024existence}, we presented preliminary results demonstrating the theoretical feasibility of three UAVs achieving this task, while establishing the impossibility of achieving it with only one or two carriers. However, the proof presented in~\cite{gabellieri2024existence} was specifically tailored to the three-UAV scenario.
Consequently,~\cite{gabellieri2024existence} left unanswered the fundamental question of whether non-stop flights are compatible with cable manipulation in the general case of multiple UAVs. 

This work constitutes a substantial leap forward beyond the minimal findings of~\cite{gabellieri2024existence} by introducing the first method for generating non-stop flight trajectories for any $n \geq 3$ flying carriers. This method enables the carriers to maintain a cable-suspended load in static equilibrium while simultaneously executing their respective non-stop flight paths. This work also provides the first proof of the existence of such trajectories for any $n \geq 3$, as extending the results of~\cite{gabellieri2024existence} to the case of $n>3$ is non-trivial.
The introduction of a novel methodology, based on the selection of Hamiltonian cycles in the graph associated with the manipulation system, was crucial. This work also reformulates the theoretical framework using graph theory, resulting in a more general and rigorous approach. Furthermore, this work presents the first complete algorithm for generating compatible trajectories, along with guidelines for tuning the trajectory's free parameters to optimize performance based on the unique capabilities of each UAV.

\textcolor{black}{We analyze for the first time how parameter tuning (frequency, amplitude, edge coloring) and Hamiltonian-cycle choice affect performance, including a uniformly attached planar-load case where orthogonal internal-force directions yield constant-speed profiles, whereas non-orthogonality increases speed variability and reduces margins. We demonstrate a translating-load scenario (piecewise-static/constant-attitude) and add a fixed-wing–compatible optimization layer that enforces bounds on speed, bank, and flight-path angles while preserving the static load.} 
This work also provides the first comprehensive validation of the theoretical findings through extensive numerical simulations assessing the influence of various parameters, investigating the impact of failures, and offering valuable insights into the system's behavior. 

\textcolor{black}{Finally, experiments with quadrotors validate the cable abstraction: with our cyclic nullspace trajectories, the load remains essentially static, while without our planner, it becomes clearly non-static with linear/angular velocities which are  orders of magnitude larger.}

The work is organized as follows. Sec.~\ref{sec:model} formalizes the notation and problem; Sec.~\ref{sec:method} derives a general method for non-stop trajectories with $n\!\ge\!3$ carriers. Sec.~\ref{sec:algo} presents the algorithm, \textcolor{black}{parameter-tuning guidance (frequency, amplitude, edge coloring) and Hamiltonian-cycle selection, including the uniformly attached planar-load case and its constant-speed condition}. Sec.~\ref{sec:sim} reports numerical validation \textcolor{black}{covering sensitivity to uncertainties/tracking, cycle choice, and single-carrier failures, a translating-load scenario (piecewise-static/constant-attitude), and a fixed-wing–compatible optimization/planning layer}. \textcolor{black}{Sec.~\ref{sec:exp} provides laboratory evidence that the cable abstraction holds in practice: our cyclic nullspace planner keeps the load essentially static, whereas without it the load exhibits orders-of-magnitude larger velocities.} Sec.~\ref{sec:conclusion} concludes with future directions.


\section{Problem Statement}\label{sec:model}
\subsection{Kinematics}
\begin{figure}[t]
    \centering
    \includegraphics[trim=0 0 0 0.1cm, clip, width=0.99\linewidth]{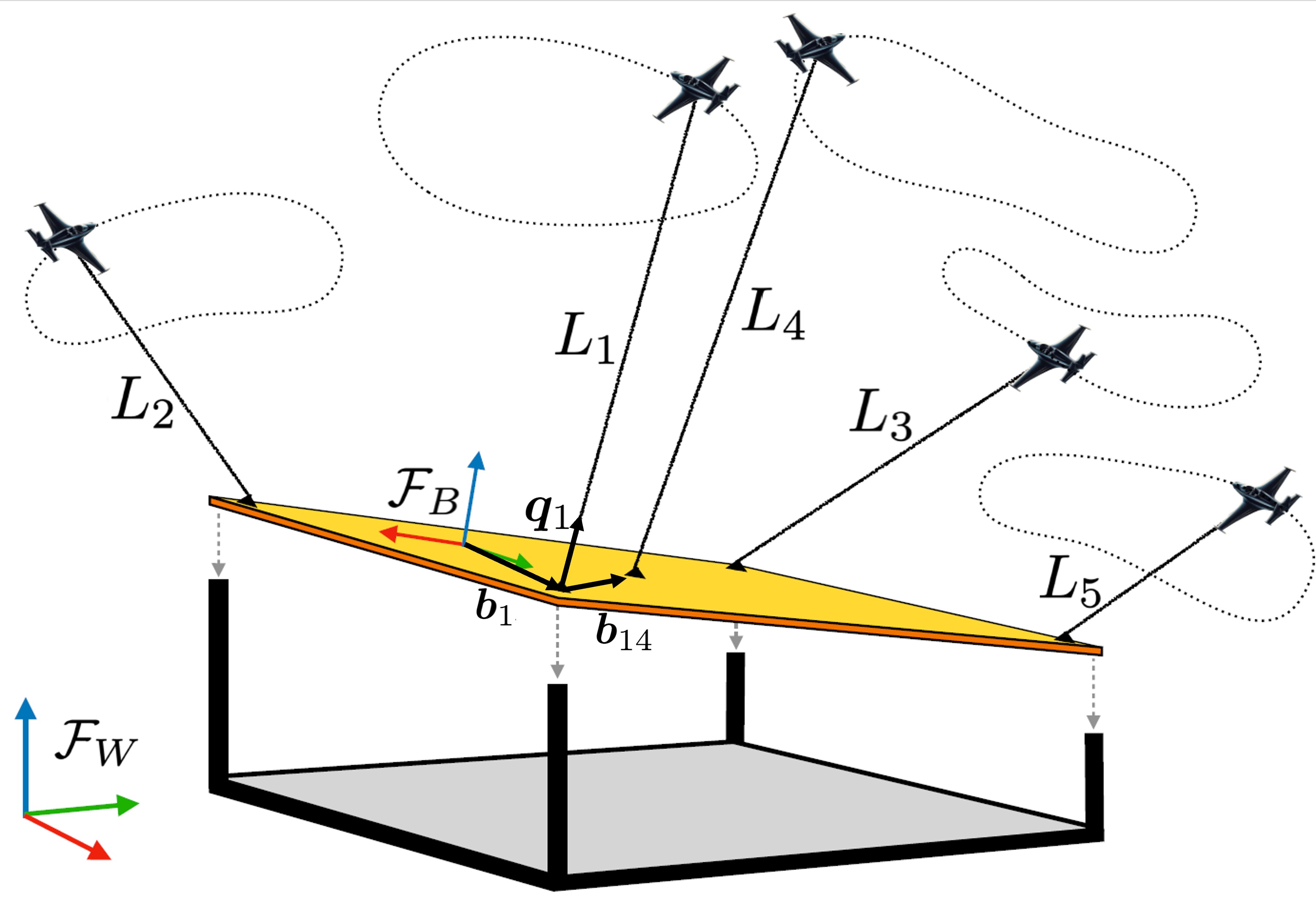}
    \caption{A team of $n\geq3$ non-stop flying carriers maintains a cable-suspended load in a fixed position to enable a potential construction scenario, while continuing on their non-zero-speed flight paths.}
    \label{fig:fig1}
\end{figure}
Denote with ${\frameW=\{\originW,\xW, \yW, \zW\},}$ the inertial world-fixed frame where $\originW$ is the origin, and $\xW, \yW, \zW$ are the x-, y-, and z- axes, respectively, see Fig.~\ref{fig:fig1}. 
Denote with ${\frameB=\{\originB, \xB, \yB, \zB\}}$ a frame attached to the manipulated \emph{load}, modeled as a rigid body, where $\originB$ is the  Center of Mass (CoM) of the load. 
The position of $\originB$ in $\frameW$ is denoted with ${\pL\in \nR{3}}$ and $\rotMat_L\in SO(3)$ expresses the orientation of $\frameB$ w.r.t. $\frameW$. The angular velocity of $\frameB$ w.r.t. $\frameW$, expressed in $\frameB$ is indicated as $^B\angVel_L$\footnote{The upper left superscript expresses the reference frame, and if omitted it is considered equal to $W$ unless differently specified.}.

The load is manipulated by the flying carriers through $n$ cables, each one of constant length $L_i>0$ 
and attached to a point $B_i$ on the load and a point $C_i$ to the carrier, with $\ioneton$. Denote with $^B\vect{b}_i\in\nR{3}$  the constant position of $B_i$ expressed in $\frameB$ and with $\pR{i}\in\nR{3}$ the position of $C_i$ in $\frameW$. 
Let the unit vector $\vect{q}_i\in S^2=\{\vect{x}\in\nR{3}\mid  \|\vect{x}\|=1\}$ pointing from $B_i$ to $C_i$
represent the direction of the $i$-th cable in $\frameW$.  
The kinematics of the points $C_i$'s, with $\ioneton$., is given by
\begin{align}
\pR{i}&=\pL+\rotMatL{^B\vect{b}}_i+L_i\vect{q}_i\label{eq:kinematics}\\
\dpR{i}&=\dpL+\drotMatL{^B\vect{b}}_i+L_i\dot{\vect{q}}_i.\label{eq:diff_kin}
\end{align}
\subsection{Dynamics}
The cables' mass and inertia are considered negligible compared to the carriers' and load's. This is a reasonable assumption which has been validated in practice~\cite{sanalitro2020full, gabellieri2023equilibria}.

Denote with $T_i\in \nR{}$ the tension of the $i$-th cable, where $\ioneton$, and with $\vect{f}_i$ the coordinates in $\frameW$ of the force that the $i$-th cable exerts on the load at~$B_i$. We have the following relation between the cable force, its tension, and direction: 
\begin{equation}
    \vect{f}_i=T_i\vect{q}_i.\label{eq:fi}
\end{equation} 
The dynamics of the load is described by
\begin{align}
    \massL\ddpL &=-\massL g \vE{3} + \sum_{i=1}^n{\vect{f}_i}\label{eq:load_dyn_tr}\\
    \inertiaL{^B\angAcc}_L &= - S(\inertiaL^B\angVel_L) {^B\angVel}_{L} + \sum_{i=1}^n{S(^B\vect{b}_i) \rotMatL^ \top \vect{f}_i}\label{eq:load_dyn_rot}\\
    \drotMatL &= S(^B\angVel_L)\rotMatL
\end{align}
where $\massL, \inertiaL$ are the mass and the rotational inertia of the load, $S(\star)$ indicates the skew operator that implements the cross product between two vectors, and $\vE{i}$ is the $i$-th column of the $3\times3$ identity matrix $\eye{3}$. 

To compactly rewrite the load dynamics, we use the matrix $\matr{G}\in\nR{6\times3n}$ that maps the cable forces to the wrench applied at the load's center of mass, referred to as the grasp matrix in the literature~\cite{yoshikawa1999virtual, tognon2018aerial}  
\begin{equation}
    \matr{G}=\begin{bmatrix}
        \eye{3} & \eye{3} & \ldots & \ \eye{3}\\ S(^B\vect{b}_1)\rotMatL^ \top& S(^B\vect{b}_2)\rotMatL^ \top&\ldots&S(^B\vect{b}_n)\rotMatL^ \top,\label{eq:grasp}
    \end{bmatrix}.
\end{equation}  

The load dynamics~\eqref{eq:load_dyn_tr}--\eqref{eq:load_dyn_rot} can be compactly rewritten as
\begin{align}
\vect{w}:=\begin{bmatrix}
         \massL(\ddpL+ g \vE{3})  
        \\
        \inertiaL{^B\angAcc}_L+ \inertiaL{^B\angVel}_L\times {^B\angVel}_{L}
            \end{bmatrix}
= \matr{G}\f
\label{eq:compact-dynamics}
\end{align}
where $\f=[\cableForce{1}^\top\ \cableForce{2}^\top\ \cdots\ \cableForce{n}^\top]^\top\in\nR{3n}$ stacks all forces that the cables apply to the load and $\vect{w}\in \nR{6}$ represents the coordinates of the resulting wrench applied by the all cables on the load.

\subsection{Statics}

Consider a static equilibrium of the suspended load characterized by constant position $\pLEq$ and orientation $\rotMatLEq$, i.e.,  
\begin{align}
\begin{cases} 
\pL(t)= \pLEq\\ 
\rotMatL(t)= \rotMatLEq
\end{cases}
\quad \forall t\in  [0,\infty).
\label{eq:equilib}
\end{align} 
Imposing~\eqref{eq:equilib} on~\eqref{eq:compact-dynamics} and~\eqref{eq:kinematics}--\eqref{eq:diff_kin} and highlighting the time-varying quantities we obtain
\begin{align}
\vect{w}_0 &:=\begin{bmatrix}
         \massL g \vE{3}  
        \\
        \vect{0}
    \end{bmatrix}
= \matr{G}(\rotMatLEq)\f(t)&
\label{eq:compact-statics}\\
\pR{i}(t)&=\pLEq+\rotMatLEq{^B\vect{b}}_i+L_i\vect{q}_i(t)&\label{eq:kinematics-const}  \forall i=1,\ldots,n\\
\dpR{i}(t) &= L_i\dot{\vect{q}}_i(t) &\forall i=1,\ldots,n.\label{eq:diff_kin-const}
\end{align}

If $\vect{f}_i$ is of class $C^1$ (i.e., it is at least continuously differentiable)
then we can differentiate~\eqref{eq:fi} with respect to time obtaining
\begin{equation}
\dot{\vect{f}}_i(t)=\dot{T}_i(t)\vect{q}_i(t)+T_i(t)\dot{\vect{q}}_i(t).\label{eq:fi_dot}
\end{equation} 
If  $T_i(t) = \|\vect{f}_i(t)\| \neq 0$, $\forall t \in [0, \infty)$, then $\vect{q}_i(t)$ and $\dot{\vect{q}}_i(t)$ can be derived from $\vect{f}_i(t)$, using~\eqref{eq:fi} and~\eqref{eq:fi_dot} respectively. Substituting the obtained $\vect{q}_i(t)$ and $\dot{\vect{q}}_i(t)$ into~\eqref{eq:kinematics-const}--\eqref{eq:diff_kin-const}, one can compute $\pR{i}(t)$ and $\dpR{i}(t)$. This demonstrates the intuitive idea that in static equilibrium—where the load's position and orientation remain constant over time—knowledge of the force trajectories is sufficient to determine how $\pR{i}(t)$ and $\dpR{i}(t)$ evolve over time. Consequently, it is reasonable to formulate the following problem.

\begin{problem*}[Coordinated Trajectories for Non-stop Flying
Carriers Holding a Cable-Suspended Load]\label{prob:nonstop}
Assume that the load is held at a constant position and orientation defined by~\eqref{eq:equilib}. Find, if it exists, a coordinated trajectory of the forces 
\begin{align}
\vect{f}_i: [0,\infty)\to \nR{3} \quad \forall i=1,\ldots,n
\end{align}
which satisfies the following conditions, for $\ioneton$
\begin{itemize}
    \item they are of class $C^1$ (at least continuously differentiable),
    \item they respect the equilibrium constraint~\eqref{eq:compact-statics} for the constant orientation, i.e., 
        \begin{align}
            \matr{G}(\rotMatLEq)\vect{f}(t) = \vect{w}_0 \quad \forall t\in  [0,\infty),
            \label{eq:compact-statics_time}
        \end{align}
        where $\vect{f}(t) = [\vect{f}_1^\top(t) \cdots \vect{f}_n^\top(t)]^\top$,
    \item they don't vanish and do not grow unbounded, i.e., $\exists \,\underline{T}>0$ and $\exists\,\overline{T}<\infty$ such that the cable tension  $T_i(t)=\| \vect{f}_i(t) \|$ satisfies 
    \begin{align}\label{eq:problem3}
        \underline{T} < T_i(t) < \overline{T}  \quad   \forall t\in  [0,\infty),
    \end{align}
\end{itemize}
and, finally, such that the carriers are never stopping and the norm of their velocity is lower bounded by a non-zero speed, i.e., for which $\exists \underline{v}>0$ such that
\begin{align}\label{eq:problem4}
\|\dpR{i}(t)\| \geq \underline{v}, \quad   \forall t\in  [0,\infty).
\end{align}
\end{problem*} 


\section{Proposed Methodology}\label{sec:method}

We know from~\cite{gabellieri2024existence} that the problem is not solvable for $n< 3$, and that there is a way to solve it in the particular case of $n=3$. Therefore, in the following, we aim at solving the general problem for $n\geq 3$.
Any coordinated trajectory of the  collective forces that maintains the  equilibrium is given by the following generic solution of~\eqref{eq:compact-statics}
\begin{align}
\f(t)=\matr{G}^\dagger\vect{w}_{0}+\matr{N}\vect{\lambda}(t)
    \label{eq:f_t_form}
\end{align}   
where $\matr{G}^\dagger\in\nR{3n\times 6}$ is the Moore-Penrose inverse of $\matr{G}(\rotMatLEq)$, $\matr{N}$ is \emph{any} constant matrix with $m$ columns which are linearly independent and  belong to $\operatorname{null}(\matr{G})$, and $\vect{\lambda}:[0,\infty)^m \to \nR{m}$ is an array of time-varying coefficients $\vect{\lambda}(t)=[\lambda_1(t)\cdots \lambda_m(t)]^\top$ for the linear combination of the columns of $\matr{N}$ expressed by $\matr{N}\vect{\lambda}(t)$.

{ \color{black}
Let $\mathcal{K}_n=(\mathcal{V},\mathcal{V}^2)$ denote the complete directed graph with $n$ vertices, where $\mathcal{V}=\{1,\ldots,n\}$. Each vertex $i \in \mathcal{V}$ is associated with the $i$-th carrier. In this work, we propose a design for the nullspace matrix $\matr{N}$ consisting of $n$ columns, where each column corresponds to an edge of $\mathcal{K}_n$ (i.e., a specific pair of carriers).

The design procedure proceeds by first selecting a Hamiltonian cycle $H$ of $\mathcal{K}_n$, and then generating the columns of $\matr{N}$ based on the edges of $H$. This means that this method allows $(n-1)!/2$ possible choices of $\matr{N}$, one for each possible choice of a Hamiltonian cycle $H$ of $\mathcal{K}_n$.

The precise definition of $\matr{N}$ given a selected $H$ is detailed in Sec.~\ref{sec:definition_N}. However, to facilitate this definition, we first introduce the necessary preliminary notation regarding Hamiltonian cycles and the indexing of their vertices and edges.}

\subsection{Preliminaries on Hamiltonian Cycles}

Consider any directed Hamiltonian cycle $H$ of \textcolor{black}{$\mathcal{K}_n$}.
The cycle $H$ is composed of a loop of $n$ edges connecting one after the other all the $n$ vertices and coming back to the starting vertex. Denote with $\matr{H}\in\{0,1,-1\}^{n\times n}$ the incidence matrix of $H$.  Denote the  list of consecutive edges of $H$ with $e_1,\ldots,e_n,e_{n+1}$, where $e_{n+1}:=e_1$ is introduced for convenience. Denote with $e_i^1$ and $e_i^2$ the index of the first and the second vertex of $e_i$, respectively, i.e., if $e_i=(j,k)$ then $e_i^1=j$ and $e_i^2=k$. 
Denote with $h_i$ the index of the edge of $H$ incoming to the vertex $i$, and consequently, $h_i+1$ represents the index of the edge of $H$ outgoing from the vertex $i$. This is due to the way the edges of $H$ are indexed, see before.  For example if $e_l=(*,i)$ then $h_i=l$ and $e_{h_i+1}=(i,*)$, where $*$ stands for an unspecified vertex number. 

\begin{example}
For example, if $n=4$ and  $H$ is the cycle that connects cyclically the vertices with indexes $1\to 3\to4\to2\to1$ in such order. We have $e_1=(1,3)$, $e_2=(3,4)$, $e_3=(4,2)$, $e_4=(2,1)$, $e_5=e_1=(1,3)$. Additionally, 
for example, $e_1^1=1$, $e_2^1=3$,  $e_3^1=4$,  $e_4^2=1$, and so on; and $h_1 = 4$, $h_1+1 = 1$, $h_3+1 = 2$, $h_4 = 2$, $h_2+1 = 4$, and so on. The incidence matrix in this case is $
\matr{H}=
\left[\begin{smallmatrix}
    1& 0    & 0     & -1\\
    0    & 0    & -1 & 1\\
    -1& 1 & 0     & 0\\
    0    & -1& 1  & 0
\end{smallmatrix}\right],
$
where the $i$-th column corresponds to the edge $e_i$ and the $j$-th row corresponds to the vertex $j$.
\label{example:Hamilton}
\end{example}

\begin{remark}\label{rem:structure_H}
    Notice the  structure of  the square matrix $\matr{H}$:
\begin{itemize}
    \item the $i$-th column (which represents the edge $e_i$) has a $1$ corresponding to the $e_i^1$-th row , a $-1$ corresponding to the $e_i^2$-th row, and the rest $n-2$ entries are all zero's;  
    \item the $i$-th row (which represents the vertex $i$) has a $-1$ corresponding to the $h_i$-th column, a $1$ corresponding to the ${h_i+1}$-th column, and the rest $n-2$ entries are all zero's. 
\end{itemize}
\end{remark}

\subsection{Decision of the null-space matrix \textbf{\textit{N}}}
\label{sec:definition_N}

Define the vector $\bm{b}_{ij}\in\mathbb{R}^3$ as follows:
$$\bm{b}_{ij}:=\rotMatLEq\frac{{^B\vect{b}}_j-{^B\vect{b}}_i}{||^B\vect{b}_j-{^B\vect{b}}_i||_2}$$ 
for any $i,j\in\{1,\ldots,n\}$. \textcolor{black}{In other terms, $\bm{b}_{ij}$ is a vector applied in the cable attachment point $B_i$ and ending in $B_j$.} Notice that $\bm{b}_{ij}=-\bm{b}_{ji}$. 

{\color{black}
In order to systematically design the nullspace matrix $\matr{N}$, we propose a graph-based construction. First, we arbitrarily select a Hamiltonian cycle $H$ from the $(n-1)!/2$ possible cycles available in $\mathcal{K}_n$. Based on this selection, we \textit{define} our proposed design for the $3n \times n$ matrix $\matr{N}$ as follows:
\begin{align}
\matr{N}(H) = (\matr{H}\otimes \matr{I}_3) 
\cdot \operatorname{diag}(\bm{b}_{e_1^1e_1^2},\ldots,\bm{b}_{e_n^1e_n^2}),
\label{eq:choice_of_N}
\end{align}
    where $\otimes$ denotes the Kronecker product and $\operatorname{diag}$ denotes a block diagonal matrix.\footnote{For the sake of completeness we note that the matrix $\matr{N}(H)$ is the transposed of the Rigidity matrix associated to the framework $\left(H, (\pR{1},\ldots, \pR{n})\right)$, see e.g.,~\cite{Zelazo2015} for a definition.} It will be shown in the subsequent analysis that this specific construction ensures that all columns of $\matr{N}$ lie strictly within the nullspace of the grasp matrix $\matr{G}$, thereby effectively generating internal forces only.
}

\begin{example}
With reference to the Example~\ref{example:Hamilton} we have

$$
\operatorname{diag}(\bm{b}_{e_1^1e_1^2},\ldots,\bm{b}_{e_n^1e_n^2})
=
\left[\begin{smallmatrix}
    \bm{b}_{13} & \vect{0}    & \vect{0}     & \vect{0} \\
    \vect{0}    & \bm{b}_{34}    & \vect{0}  & \vect{0} \\
    \vect{0} & \vect{0}  & \bm{b}_{42}   & \vect{0}\\
    \vect{0}    & \vect{0} & \vect{0}   & \bm{b}_{21}
\end{smallmatrix}\right]\in\mathbb{R}^{12 \times 4}.
$$

$$
 \matr{N}(H)=
\left[\begin{smallmatrix}
    \bm{b}_{13} & \vect{0}    & \vect{0}     & -\bm{b}_{21}\\
    \vect{0}    & \vect{0}    & -\bm{b}_{42} & \bm{b}_{21}\\
    -\bm{b}_{13}& \bm{b}_{34} & \vect{0}     & \vect{0}\\
    \vect{0}    & -\bm{b}_{34}& \bm{b}_{42}  & \vect{0}
\end{smallmatrix}\right]\in\mathbb{R}^{12 \times 4}.
$$
\end{example}

For simplicity, in the following, we omit the dependency of $\matr{N}$ on the particular $H$  when the choice of $H$ is either implicit or non-discriminatory.

We can easily check that the $i$-th column of the chosen $\matr{N}$, denoted with $\matr{N}_{(:,i)}$, belongs to $\operatorname{null}(\matr{G})$, in fact: 
$$\matr{G}\matr{N}_{(:,i)} = 
\begin{bmatrix}
   \bm{b}_{e_i^1 e_i^2} - \bm{b}_{e_i^1e_i^2} \\
   S(^B\vect{b}_{e_i^1})\rotMatL^\top  \bm{b}_{e_i^1 e_i^2} 
   -
   S(^B\vect{b}_{e_i^2})\rotMatL^\top  \bm{b}_{e_i^1e_i^2} 
\end{bmatrix}
=
$$
$$
\begin{bmatrix}
   \vect{0} \\
   S(^B\vect{b}_{e_i^1})({^B\vect{b}}_{e_i^2} -{^B\vect{b}}_{e_i^1})  
   -
   S(^B\vect{b}_{e_i^2})({^B\vect{b}}_{e_i^2} - {^B\vect{b}}_{e_i^1})
\end{bmatrix}=
$$
$$
\begin{bmatrix}
   \vect{0} \\
   S(^B\vect{b}_{e_i^1}){^B\vect{b}}_{e_i^2}
   +
   S(^B\vect{b}_{e_i^2}){^B\vect{b}}_{e_i^1}
\end{bmatrix}
=
\begin{bmatrix}
   \vect{0} \\
   \vect{0} 
\end{bmatrix},
$$
where we exploited the fact that $S(\vect{x})\vect{x}=\vect{0}$ and $S(\vect{x_1})\vect{x_2}=-S(\vect{x_2})\vect{x_1}$ for any $\vect{x},\vect{x}_1,\vect{x}_2\in\mathbb{R}^3$.

\textcolor{black}{Note that the Hamiltonian cycle is a tool that allows selecting each cable (so that no cable is left with no internal forces) and ensuring that its internal force is generated by exactly two independent directions. Note also that the directions into which the internal force of each cable is decomposed depend on the other cables' and cannot be chosen at will but must belong to the nullspace of $\bm{G}$. For any number of contact points on a rigid body, admissible directions are the ones connecting the contact points themselves \cite{yoshikawa1999virtual}.}
\subsection{Force of the $i$-th cable for the chosen \textbf{\textit{N}}}

From~\eqref{eq:f_t_form}, the force applied by the 
$i$-th cable is
$$ \f_i(t)=\vect{f}_{0i}+\matr{N}_i\vect{\lambda}(t)$$
where 
$\vect{f}_{0i}\in\mathbb{R}^3$ is the vector obtained by the three components of $\matr{G}^\dagger\vect{w}_{0}$ at the positions $3(i-1)+\{1,2,3\}$,
and $\matr{N}_i\in \mathbb{R}^{3\times n}$ is made by the $3$ rows of  $\matr{N}$ located at the same positions, which induces the following block partition $
\matr{N}=
\left[\begin{smallmatrix}
    \matr{N}_1^\top & \cdots &
    \matr{N}_n^\top
\end{smallmatrix}\right]^\top.
$

Using the definition of $h_i$ and ${h_i+1}$ and remembering what said in Remark~\ref{rem:structure_H} regarding the rows of $\matr{H}$, we have that
$$
\matr{N}_i \vect{\lambda}(t) =
-\lambda_{h_i}(t)\;\bm{b}_{e_{(h_i)}^1e_{(h_i)}^2} + 
\lambda_{{h_i+1}}(t)\;\bm{b}_{e_{({h_i+1})}^1e_{({h_i+1})}^2}.  
$$
For the sake of brevity, we define
$\vect{\delta}_i:=-\bm{b}_{e_{h_i}^1e_{h_i}^2}$,
$\bar{\vect{\delta}}_i:=\bm{b}_{e_{{h_i+1}}^1e_{{h_i+1}}^2}$, $\mu_i(t):=\lambda_{h_i}(t)$, and  $\bar{\mu}_i(t):=\lambda_{{h_i+1}}(t)$.
Thus we can rewrite the previous equation as 
\begin{align}
\matr{N}_i \vect{\lambda}(t) =
\mu_i(t)\;\vect{\delta}_i + 
\bar{\mu}_i(t)\;\bar{\vect{\delta}}_i. 
\label{eq:time_var_force_Ni}
\end{align}


\begin{assumption}
For the selected Hamiltonian cycle, we have that for any $i=1,\ldots,n$, the three attachment points  $B_i$, $B_{e_{h_i}^1}$ and $B_{e_{h_i+1}^2}$ are not aligned, or equivalently, the vectors 
${^B\vect{b}}_i - {^B\vect{b}}_{e_{h_i}^1}$ and 
${^B\vect{b}}_{e_{h_i+1}^2}-{^B\vect{b}}_i$ are linearly independent.
\label{assumpt:not_aligned}
\end{assumption}

Assumption~\ref{assumpt:not_aligned} ensures that for any $i=1,\ldots,n$ the vectors  $\vect{\delta}_i$ and $\bar{\vect{\delta}}_i$ are not parallel and therefore 
\begin{itemize}
    \item $\operatorname{span}\{\vect{\delta}_i,\bar{\vect{\delta}}_i\}$ is a 2D subspace of $\mathbb{R}^3$, and
    \item $\operatorname{rank}(\matr{N})=n$, i.e., all the columns of $\matr{N}$ are linearly independent.
\end{itemize}
To assess the truth of the second statement, consider that for all $i = 1, \dots, n$, we have $\matr{N}_i \vect{\lambda} = \vect{0}$ if and only if both $\mu_i = 0$ and $\bar{\mu}_i = 0$, see~\eqref{eq:time_var_force_Ni}. Therefore, the only $\vect{\lambda}\in\mathbb{R}^n$ which makes $\matr{N} \vect{\lambda}=\vect{0}$ is $\vect{\lambda}=\vect{0}$.

We can then recap the main results obtained so far in the following statement.
\begin{result}
After arbitrarily selecting one of the Hamiltonian cycles of the complete graph with $n$ vertices, construct a matrix $\matr{N}(H)\in\mathbb{R}^{3n\times n}$ as described in~\eqref{eq:choice_of_N}. If Assumption~\ref{assumpt:not_aligned} is valid, then:
\begin{itemize}
    \item  the $n$ columns of $\matr{N}$ belong to the nullspace of $\matr{G}$ and are linearly independent.

    \item The force applied by the $i$-th carrier has the following affine structure: \begin{align}\label{eq:forces_mu}
    \f_i(t) = \vect{f}_{0i} + \underbrace{\mu_i(t)\;\vect{\delta}_i + \bar{\mu}_i(t)\;\bar{\vect{\delta}}_i}_{=:\tilde{\vect{f}}_i(t) \in \operatorname{span}\{\vect{\delta}_i, \bar{\vect{\delta}}i\}}. 
    \end{align} 
    This expression shows that at any time $t$, $\f_i(t)$ lies in the affine 2D subspace of $\mathbb{R}^3$ that passes through $\vect{f}_{0i}$ and is parallel to $\operatorname{span}\{\vect{\delta}_i, \bar{\vect{\delta}}_i\}$. The quantities $\mu_i(t)$ and $\bar{\mu}_i(t)$ represent the coordinates of the projection of $\f_i(t)$ along $-\vect{f}_{0i}$ onto  this subspace w.r.t. $\vect{\delta}_i$ and $\bar{\vect{\delta}}_i$, respectively (see Fig.~\ref{fig:geom_deltas}).

    \item The time derivative of the force is given by 
    \begin{align} \label{eq:fidot_mudot}
    \dot{\f}_i(t) = \dot{\mu}_i(t) \;\vect{\delta}_i + \dot{\bar{\mu}}_i(t) \;\bar{\vect{\delta}}_i, 
    \end{align} 
    which means that $\dot{\f}_i(t)$ lies in the 2D subspace of $\mathbb{R}^3$ spanned by ${\vect{\delta}_i, \bar{\vect{\delta}}_i}$. The quantities $\dot{\mu}_i(t)$ and $\dot{\bar{\mu}}_i(t)$ represent the coordinates of $\dot{\f}_i(t)$ in this subspace along $\vect{\delta}_i$ and $\bar{\vect{\delta}}_i$, respectively.
    \end{itemize}
\end{result}

\begin{figure}[t]
\centering
\includegraphics[width=0.7\columnwidth]{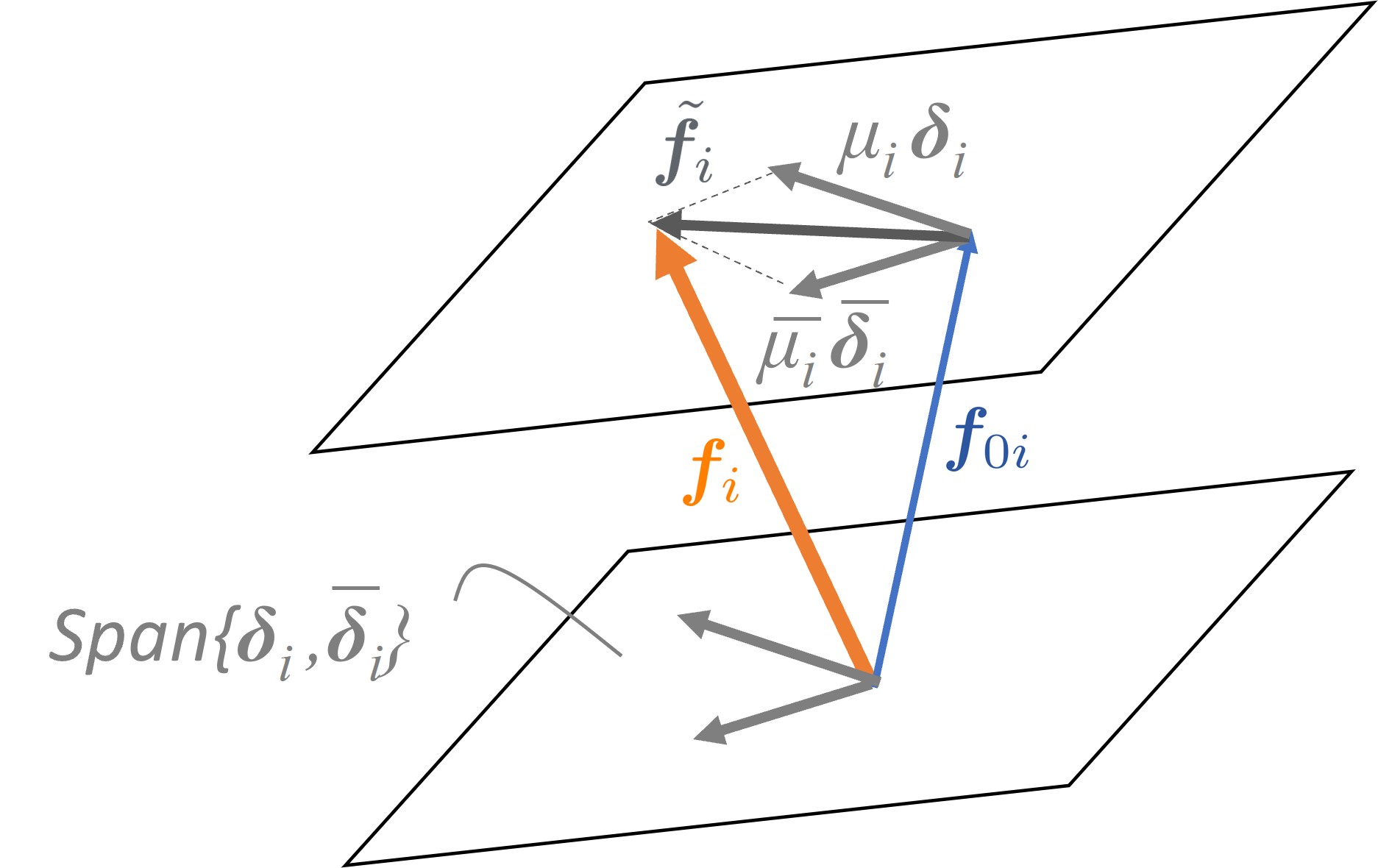}\caption{Geometric visualization of the forces in the $i$-th cable.}\label{fig:geom_deltas}
\end{figure}

\subsection{Relation between force derivative and  carrier velocity}

So far, we have seen how the particular choice of $\matr{N}$ simplifies the evolution of the force applied by the $i$-th cable and its time derivative, by restricting the latter to evolve within a 2D subspace, and the former to evolve within a parallel 2D affine subspace passing through $\vect{f}_{0i}$.
\begin{remark}
Notice the fact that with this choice, the $n$ forces belong to $n$ different 2D affine subspaces despite the fact that the attachment points are, in general, distributed in 3D.     
\end{remark}
In this section, we will explore how this result is instrumental in drawing important consequences for relating the time derivative of the force to the velocity of the $i$-th carrier.

Starting from and further analyzing~\eqref{eq:fi_dot} the time derivative of the $i$-cable force is given by
$$
\dot{\vect{f}}_i=\dot{T}_i\vect{q}_i+T_i\dot{\vect{q}}_i = 
\dot{\vect{f}}_i^{\parallel}+\dot{\vect{f}}_i^{\perp}
$$
where $\dot{\vect{f}}_i^{\parallel}=\dot{T}_i\vect{q}_i$ represents the change along the direction parallel to the one of the force, i.e., the change of the force intensity, and $\dot{\vect{f}}_i^{\perp}=T_i\dot{\vect{q}}_i$ represents the change orthogonal to the force, i.e., the change of the force direction. 

Recalling~\eqref{eq:diff_kin-const}, we have in general the following
\begin{result}\label{result_2}
    With the assumption of static load, and as long as the tension on the $i$-th cable $T_i$ is positive and bounded,
    \begin{itemize}
        \item the velocity of the $i$-th carrier is collinear and proportional to $\dot{\vect{f}}_i^{\perp}$ through the following relation
        \begin{equation}
        \dpR{i}(t) = \tfrac{L_i}{T_i(t)}\dot{\vect{f}}_i^{\perp}(t).\label{eq:dpr_df}
      \end{equation}
    \item The velocity of the $i$-th carrier $\dpR{i}(t)$ is nonzero as long as the direction of the force changes, i.e., as long as $\dot{\vect{f}}_i^{\perp}(t)$ is nonzero.
    \end{itemize}
\label{res:velocity_force_derivative}
\end{result}

We can now state the main result of this section, which is a consequence of our particular choice of $\matr{N}$ as an $\matr{N}(H)$ for a Hamiltonian cycle $H$.

\begin{lemma}\label{prop_1}
    After arbitrarily selecting one of the Hamiltonian cycles of the complete graph with $n$ vertices, denoted with $H$,  construct a matrix $\matr{N}(H)\in\mathbb{R}^{3n\times n}$ as described in~\eqref{eq:choice_of_N}. 
    Then, the following conditions are sufficient to have $\|\dpR{i}(t)\|>0$
    \begin{enumerate}
        \item Assumption~\ref{assumpt:not_aligned} holds,
        \item $\vect{f}_{0i}\not\in \operatorname{span}\{\vect{\delta}_i,\bar{\vect{\delta}}_i\}$,
        \item $\mu_i(t)$ and $\bar{\mu}_i(t)$ are bounded,
        \item $|\dot{\bar{\mu}}_i(t)| + |\dot{\mu}_i(t)| \neq 0$. 
    \end{enumerate}
\label{prop:sufficient_cond_nonzer_vel}
\end{lemma}
\begin{proof}
The assumptions of Result~\ref{res:velocity_force_derivative} are satisfied because of the following two facts:
    \begin{itemize}    
        \item since $\vect{f}_{0i}\not\in \operatorname{span}\{\vect{\delta}_i,\bar{\vect{\delta}}_i\}$ 
        the affine 2D subspace  where $\vect{f}_{i}(t)$ belongs to does not pass through the origin and therefore  $\vect{f}_{i}(t)$ is nonzero for any $t$ and so is $T_i(t)$.
        \item Since $\mu_i(t)$ and $\bar{\mu}_i(t)$ are bounded then $\|\vect{f}_{i}(t)\|$ is bounded and so is $T_i(t)$.  
    \end{itemize}

From Result~\ref{res:velocity_force_derivative} we know that in order to show that $\|\dpR{i}(t)\|>0$ it is enough to show that $\|\dot{\vect{f}}_i^{\perp}(t)\| > 0$, and we do this by showing the next two points 
    \begin{itemize}
        \item (non vanishing derivative) $\|\dot{\vect{f}}_i(t)\|\neq 0$, and
        \item (non-collinear force and derivative) $\dot{\vect{f}}_i(t)$ is not collinear with $\vect{f}_i(t)$.
    \end{itemize}

\emph{Non vanishing} From $|\dot{\bar{\mu}}_i(t)|+|\dot{\mu}_i(t)| \neq 0$ it follows that $\dot{\bar{\mu}}_i(t)$ and $\dot{\mu}_i(t)$ are not simultaneously zero and therefore $\|\dot{\vect{f}}_i(t)\|\neq 0$. 

\emph{Non-collinear}. By contradiction assume that $\vect{f}_i(t)$ is collinear with  $\dot{\vect{f}}_i(t)$, therefore also their sum  
    $$
    \dot{\vect{f}}_i(t) + \vect{f}_i(t) = \vect{f}_{0i}+\tilde{\vect{f}}_{i}(t)+\dot{\vect{f}}_i(t)
    $$
    is collinear with $\dot{\vect{f}}_i(t)$, and therefore their sum is contained in $\operatorname{span}\{\vect{\delta}_i,\bar{\vect{\delta}}_i\}$. However on the right hand side of the previous equation, $\tilde{\vect{f}}_{i}(t)+\dot{\vect{f}}_i(t)$ is contained in $\operatorname{span}\{\vect{\delta}_i,\bar{\vect{\delta}}_i\}$ while $\vect{f}_{0i}\not\in\operatorname{span}\{\vect{\delta}_i,\bar{\vect{\delta}}_i\}$, which implies that the sum is not contained in $\operatorname{span}\{\vect{\delta}_i,\bar{\vect{\delta}}_i\}$, which brings to the sought contradiction.

\smallskip
We have shown that the assumptions of Result~\ref{res:velocity_force_derivative} are satisfied, and so the proposition is proven by applying the second point in Result~\ref{res:velocity_force_derivative}.
%
\end{proof}

\subsection{Decision of the time-varying coefficients ${\lambda}(t)$}\label{subsec:coloring}

In Lemma~\ref{prop:sufficient_cond_nonzer_vel}, we have seen that two important conditions to be satisfied are that,  for every carrier $i=1,\ldots,n$,
\begin{itemize}
    \item $\mu_i(t)$ and $\bar{\mu}_i(t)$ are bounded, and
    \item $|\dot{\bar{\mu}}_i(t)| + |\dot{\mu}_i(t)|\neq 0$, i.e., $\dot{\bar{\mu}}_i$ and $\dot{\mu}_i$ are not simultaneously zero. 
\end{itemize}
In order to satisfy the first condition, it is enough to chose smooth bounded functions for the $\lambda_i(t)$'s, for example, a natural choice is to choose periodic functions which oscillate between a maximum and a minimum, as, e.g., the combination of some trigonometric functions of time.

Restating the second condition for the original $\dot\lambda_i(t)$'s variables, we see that for the second condition, we must guarantee that
$\dot\lambda_i(t)$ and $\dot\lambda_{i+1}(t)$ are not simultaneously zero for all $i=1,\ldots,n$ (with $\dot\lambda_{n+1}:=\dot\lambda_1$). Since in order to satisfy the boundedness condition we have opted for smooth oscillatory functions, we need to choose carefully those functions such that the pair $\lambda_i(t)$ and $\lambda_{i+1}(t)$ do not attain their stationary points simultaneously, for every consecutive pair of edges obtained when $i=1,\ldots,n$. 

Assume that we define a finite library $C$ of periodical functions of time which do not have any simultaneous stationary point in common, denote such library with $C = \{c_1(t), c_2(t), \ldots,c_{n_c}(t)\}$. Formally, we are imposing that   $\{t \in \mathbb{R} \;|\; \dot c_i(t)=0\}\cap \{t \in \mathbb{R} \;|\; \dot c_j(t)=0\}=\emptyset$ for every $i,j=1,\ldots,n_c$, with $j\neq i$. 

The problem is then to assign to each edge of the Hamiltonian cycle $H$ one function in $C$ such that   
two incident (consecutive) edges, $e_i$ and $e_{i+1}$, do not have the same function assigned. This is a type of problem called \emph{edge coloring} in graph theory. The minimum number of colors for edge coloring in a cycle graph such as $H$ is two if the number of edges is even and three if the number of edges is odd~\cite{soifer2008mathematical}. Therefore, to account for any  possible case, it is enough to have a library with three functions, e.g., 
\begin{align}\label{eq:c1c2c3}
    \begin{split}
        c_1(t)&=A\cos(\xi t)\\
        c_2(t)&=A\cos(\xi t + \tfrac{\pi}{3})\\
        c_3(t)&=A\cos(\xi t + \tfrac{2\pi}{3})
    \end{split}
\end{align}
where $A>0$ and $\xi>0$ are any amplitude and angular frequency, respectively, and then equate each $\lambda_i(t)$ to one of such functions, taking care of avoiding that the same function is assigned both to $\lambda_i(t)$ and $\lambda_{i+1}(t)$ for any $i=1,\ldots,n$, with $\lambda_{n+1}(t):=\lambda_1(t)$. 
With this choice, and following an appropriate assignment corresponding to the coloring of $ H $, we obtain that when $ \dot{\lambda}_i(t) = 0 $, it follows that $ \dot{\lambda}_{i+1}(t) = \pm A\xi\frac{\sqrt{3}}{2} $.  
Therefore, by Lemma~\ref{prop_1}, we conclude that $ \| \dpR{i} \| > 0 $.

Of course, an infinite number of other choices are possible, e.g., in the case of an even number of carriers, one could use a simplified library with just two functions out of phase by $\frac{\pi}{2}$ radians
\begin{align}\label{eq:c1c2}
        c_1(t)&=A\cos(\xi t),\ 
        c_2(t)=A\cos(\xi t + \tfrac{\pi}{2}).   
\end{align}
By assigning $ c_1 $ and $ c_2 $ alternatively to consecutive edges of $ H $, we have that when $ \dot{\lambda}_i(t) = 0 $, then $ \dot{\lambda}_{i+1}(t) = \pm A\xi $. 

We can now state the conclusive result of this section. 

\begin{prop}\label{prop_2}
After arbitrarily selecting one of the Hamiltonian cycles of the complete graph with $n$ vertices, denoted with $H$, construct a matrix $\matr{N}(H)\in\mathbb{R}^{3n\times n}$ as described in~\eqref{eq:choice_of_N}. 
    Furthermore,   assume that the following conditions are true:
    \begin{enumerate}
        \item Assumption~\ref{assumpt:not_aligned} holds,
        \item $\vect{f}_{0i}\not\in \operatorname{span}\{\vect{\delta}_i,\bar{\vect{\delta}}_i\}$,        
    \item the $\lambda_i(t)$'s are chosen via edge coloring of the selected Hamiltonian cycle using as a library of functions~\eqref{eq:c1c2} for the case of an odd~$n$ or~\eqref{eq:c1c2c3} for the case of an even~$n$;

    \end{enumerate}
then, the Problem of Coordinated Trajectories for Non-stop Flying Carriers holding a Cable-Suspended Load is solved.
\end{prop}

\begin{proof} 
We will show that all the requirements of the Problem are met by the proposed choice of coordinated trajectories.
Choosing  $\lambda_i(t)$'s via edge coloring from the libraries~\eqref{eq:c1c2} or~\eqref{eq:c1c2c3} ensures that the forces~\eqref{eq:forces_mu} belong to class~$C^1$.
    Furthermore~\eqref{eq:compact-statics_time} holds thanks to~\eqref{eq:time_var_force_Ni}. 
    Additionally, the forces do not vanish 
    for what is said 
    in the proof of Lemma~\ref{prop:sufficient_cond_nonzer_vel}. Furthermore,
    for any bounded $A$ in~\eqref{eq:c1c2} or~\eqref{eq:c1c2c3}, the forces are upper-bounded because each force consists of a constant term and a variable term with bounded periodic coordinates. Therefore,~\eqref{eq:problem3} is satisfied.

Finally, the next part of the proof is aimed at showing that~\eqref{eq:problem4} holds.
    Combining~\eqref{eq:fidot_mudot} and the time derivatives of~\eqref{eq:c1c2} or~\eqref{eq:c1c2c3}, we can write \begin{align}
        \dot{\bm{f}}_i&=
        \begin{bmatrix}
            \bm{\delta}_i &\bar{\bm{\delta}}_i
        \end{bmatrix}
        \begin{bmatrix}
            \dot{\mu}_i\\ \dot{\bar{\mu}}_i  
        \end{bmatrix}=\\ &=
        \begin{bmatrix}
            \bm{\delta}_i &\bar{\bm{\delta}}_i
        \end{bmatrix} 
        \xi A
        \left[
        \begin{smallmatrix}
            -\sin{\beta_i} & \;-\cos{\beta_i}\\ 
            -\sin{\bar{\beta}_i} & \; -\cos{\bar{\beta}_i}
        \end{smallmatrix}
        \right]
        \begin{bmatrix}
            \cos{(\xi t)}\\ \sin{(\xi t)}
        \end{bmatrix}:= \\ &:=
        \xi\Delta_i \bm{M}_i(\beta_i,\bar{\beta}_i,A)   \begin{bmatrix}
            \cos{(\xi t)}\\ \sin{(\xi t)} 
        \end{bmatrix},\label{eq:ellipses}
    \end{align}
    where $\Delta_i:=\begin{bmatrix}
            \bm{\delta}_i &\bar{\bm{\delta}}_i
        \end{bmatrix}$, 
        $\bm{M}_i:=
        A\left[
        \begin{smallmatrix}
            -\sin{\beta_i} & \;-\cos{\beta_i}\\ 
            -\sin{\bar{\beta}_i} & \; -\cos{\bar{\beta}_i}
        \end{smallmatrix}
        \right]$
        where $\beta_i\in\{0,\pi/3,2\pi/3\}$  and  $\bar{\beta}_i\in\{0,\pi/3,2\pi/3\}/\{\beta_i\}$ for an odd $n$ (or $\beta_i\in\{0,\pi/2\}$  and  $\bar{\beta}_i\in\{0,\pi/2\}/\{\beta_i\}$ for an even $n$) are the constant phases of the sinusoidal functions  in~\eqref{eq:c1c2}, ~\eqref{eq:c1c2c3} corresponding to the two edges adjacent to the $i$-th vertex. 
        Thanks to the edge coloring, the two constant phases are different; therefore, $\bm{M}_i$ is full-rank.
        For Assumption~\ref{assumpt:not_aligned}, $\Delta_i$ is also full-rank. Hence,  $\Delta_i \bm{M}_i 
        \left[
        \begin{smallmatrix}
            \cos{(\xi t)}\\ \sin{(\xi t)}, 
        \end{smallmatrix}
        \right]$ is a parametrization of a 2D ellipse whose minor axis length is nonzero and is denoted with $\underline\gamma_i>0$.  In other words, for \eqref{eq:ellipses} we have that  $\forall t, ||\dot{\bm{f}}_i||\geq\xi \underline\gamma_i$. For the same argument used in Lemma~\ref{prop_1}, we know that $\bm{f}_i$ and $\dot{\bm{f}}_i$ are not collinear, therefore we can state that for any non-zero $\dot{\bm{f}}_i$, there is a non-zero $\dot{\bm{f}_i^\perp}$. More specifically there exists $\alpha_i>0$ such that $||\dot{\bm{f}_i^\perp}||\geq\alpha_i ||\dot{\bm{f}}_i||.$ Finally, we have that $||\dot{\bm{f}_i^\perp}||\geq\xi \alpha_i\underline\gamma_i$. Applying~\eqref{eq:dpr_df}, 
        we have that  
        $\|\dpR{i}(t)\| \geq \tfrac{L_i}{T_i(t)}\xi \alpha_i\underline\gamma_i\geq\tfrac{L_i}{\overline{T_i}}\xi \alpha_i\underline\gamma_i$.
        Finally, we can conclude that there exists 
        \begin{align}
            \underline{v}=\frac{L_i}{\overline{T}}\xi\alpha_i\underline\gamma_i
            \label{eq:vel_lower_bound}
        \end{align} such that $\forall t, ||\dpR{i}||\geq \underline{v}$. Therefore, all the requirements of the Problem are met by the proposed choice of coordinated trajectories. Note also that, with analogous reasoning and calling the major axis of the 2D ellipses $\overline\gamma_i,$ we have $\|\dpR{i}(t)\| \leq \tfrac{L_i}{T_i(t)}\xi \alpha_i\overline\gamma_i\leq\tfrac{L_i}{\underline{T_i}}\xi \alpha_i\overline\gamma_i$. 
\end{proof}

\begin{algorithm}[t]
\small
\caption{Computation of the non-stop carrier trajectories}\label{alg:1}
$[\pR{1},\dpR{1},\ldots,\pR{n},\dpR{n}]=$\\compute\_carrier\_trajectories$(n,\massL, \rotMatLEq, \pLEq, {^B\bm{b}}_1,\ldots,{^B\bm{b}}_n)$
\begin{algorithmic}

\State $\bm{G}(\rotMatLEq, ^B\bm{b}_i) \gets$ \eqref{eq:grasp}
\State $\bm{f}_{0} \gets\bm{G}^\dagger\begin{bmatrix}
    \massL g\vE{3}\\\bm{0}_{3\times1}
\end{bmatrix}$

\Comment{\textit{Find a Hamiltonian cycle satisfying conditions 1) and 2) of Proposition \ref{prop_2}}}
\State $\mathcal{H}\gets $ all Hamiltonian cycles.
\While{TRUE}
\State $h\gets$ draw a cycle from $\mathcal{H}$

\State $\mathcal{H}\gets \mathcal{H}\backslash h$ 

\For{$j\neq n$}

$\vect{\delta}_i = \rotMatLEq({^B\vect{b}}_{e_{h_i}^2}-{^B\vect{b}}_{e_{h_i}^1})$

 $\bar{\vect{\delta}}_i =  \rotMatLEq({^B\vect{b}}_{e_{{h_i+1}}^2}-{^B\vect{b}}_{e_{{h_i+1}}^1})$
 
\EndFor

\If{rank($[\vect{f}_{0i} \; \vect{\delta}_i \; \bar{\vect{\delta}}_i])=3$, for all $\ioneton$}

\textbf{Break}
\EndIf
\EndWhile

 $\bm{H}\gets$ incidence matrix of $h$

$\bm{N}\gets$ \eqref{eq:choice_of_N}

\Comment{\textit{Apply edge coloring using \eqref{eq:c1c2} or \eqref{eq:c1c2c3}. }}

    \If{$n$ is even}
    \For{$j\leq n-2$}
        $\lambda_j(t)\gets c_1(t)$ in \eqref{eq:c1c2}
    
        $\lambda_{j+1}(t)\gets c_2(t)$ in \eqref{eq:c1c2}
    
        $j\gets j+2$
      \EndFor
    \Else
    \For{$j\leq n-3$}
    $\lambda_j(t)\gets c_1(t)$ in  \eqref{eq:c1c2c3}
    
        $\lambda_{j+1}(t)\gets c_2(t)$ in \eqref{eq:c1c2c3}

             $\lambda_{j+2}(t)\gets c_2(t)$ in \eqref{eq:c1c2c3}
    \State $j\gets j+3$
        \EndFor
    \EndIf

   \Comment{\textit{Compute the cable forces and use kinematics to compute the carriers' trajectories}}
   \For{$j=1:n$, $\forall t>0$}
   
   $\bm{f}_j(t)\gets$\eqref{eq:forces_mu}
   
   $T_j\gets ||\bm{f}_j||$
   
   $\bm{q}_j\gets \frac{\bm{f}_j}{T_j}$ 
   
   $\dot{\bm{f}}_j\gets$\eqref{eq:fidot_mudot}

   $\dot{\bm{q}}_j\gets (\dot{\bm{f}}_j- \bm{q}_j\bm{q}_j^\top\dot{\bm{f}}_j)/ T_j$

   $\pR{j}\gets$\eqref{eq:kinematics-const}

   $\dpR{j}\gets$\eqref{eq:diff_kin-const}
    \EndFor 
\end{algorithmic}
\end{algorithm}

\section{Algorithm and Remarks}
\label{sec:algo}

\subsection{Algorithm to generate Coordinated Trajectories for Non-stop Flying Carriers Holding a Cable-Suspended Load}
Algorithm~\ref{alg:1} presents a conceptual pseudo-code of the solution formally derived in Sec.~\ref{sec:method}. This algorithm computes the parameters of a set of $n$ coordinated trajectories for the $n$ carriers. These trajectories are periodic, never stopping, continuously differentiable, and ensure compatibility with the forces required to hold the load stationary.

\textcolor{black}{The main steps of the algorithm can be summarized as follows. The algorithm determines the cable forces; from those and their time derivatives, it computes, using the system's kinematics, the positions and velocities of the carriers over time. The steps to determine the cable forces are as follows. The component of each cable force that compensates the external wrench is computed by inverting \eqref{eq:compact-statics} as in \eqref{eq:f_t_form}. Then, the other component of each cable force, namely the internal force, is computed. To do so, a Hamiltonian cycle is selected to determine the nullspace $N$ of the grasp matrix in \eqref{eq:f_t_form}. The internal forces are expressed as in \eqref{eq:time_var_force_Ni}, where ${\vect{\delta}}_i$ and $\bar{\vect{\delta}}_i$ are given once the Hamiltonian cycle is selected, and $\mu_i$ and $\bar{\mu}_i$ are chosen with the procedure given in Sec. \ref{subsec:coloring}.}

\subsection{Qualitative Remarks about the Choice of Free Parameters}

\begin{remark} \emph{(Choice of the Hamiltonian cycle among the $\frac{(n-1)!}{2}$ possible for $n$ carriers).} 
There may be certain Hamiltonian cycles that are inadmissible because they would violate two of the assumptions of Proposition~\ref{prop_2}. 
First, there might exist Hamiltonian cycles for which Assumption 1 does not hold. In that case, 
$\operatorname{span}\{\vect{\delta}_i,\bar{\vect{\delta}}_i\}$ is a 1D subspace of 
$\mathbb{R}^3$, and 
$\nexists{\bm{\lambda}}(t)$ such that $\|\dpR{i}(t)\|\geq\underline{v}$ for  $\underline{v}>0,$ as we showed in \cite{gabellieri2024existence}. Note that if there exist no three collinear cable attachment points on the load, then Assumption 1 holds for any Hamiltonian cycle. 

Secondly, any Hamiltonian cycle for which $\vect{f}_{0i}\in \operatorname{span}\{\vect{\delta}_i,\bar{\vect{\delta}}_i\}$ is not admissible. That would imply that the total cable force $\vect{f}_i(t),$ and hence $T_i(t)$, may be equal to zero. This invalidates our model in~\eqref{eq:dpr_df}.
Secondly, $\vect{f}_{0i}\in \operatorname{span}\{\vect{\delta}_i,\bar{\vect{\delta}}_i\}$ would allow $\vect{f}_i(t)$ and $\dot{\vect{f}}_i(t)$ to be collinear (see Proof of Lemma~\ref{prop_1}), and hence a variation of the force $\vect{f}_i(t)$ (i.e., a non-zero  $\dot{\vect{f}}_i(t)$ would not necessarily imply a non-zero $\dpR{i}(t)$.

From what has been said, we hypothesize that it is desirable to choose a Hamiltonian cycle for which 
$\vect{\delta}_i$ and $\bar{\vect{\delta}}_i$ are as much as possible orthogonal. In that way, the carrier's motion is as far as possible from being constrained on a 1D space  
($\vect{\delta}_i$ and $\bar{\vect{\delta}}_i$ linearly dependent). Specifically, the 2D ellipses represented by~\eqref{eq:ellipses} becomes more and more skew as $\vect{\delta}_i$ and $\bar{\vect{\delta}}_i$ are more aligned. Consequently, the minimum of the carriers' velocities $\underline{v}$ in~\eqref{eq:vel_lower_bound} is lower (with everything else remaining unchanged).
Moreover, we hypothesize that it is 
 better to choose Hamiltonian cycles such that $\vect{f}_{0i}$ is as much orthogonal as possible to 
$\operatorname{span}\{\vect{\delta}_i,\bar{\vect{\delta}}_i\}$.  In that way, we  
avoid that, even for large $\dot{\mu}_i$ and  $\dot{\bar{\mu}}_i$, the value of 
$\alpha_i$, and hence $\dot{\vect{f}}_i^{\perp}(t)$, is 
small; that would result in a small 
minimum value $\underline{v}
(\alpha_i)$ of $\dpR{i}$ (see~\eqref{eq:vel_lower_bound}). 
\end{remark}
\begin{remark} \emph{\textcolor{black}{(Failure of one carrier).}} 
One of the advantages of multi-robot systems is the tolerance to a single point of failure. 
When one robot fails, e.g., it detaches from the system, \eqref{eq:compact-statics_time} is not exact anymore: \textcolor{black}{the grasp matrix has changed, and so has its nullspace; hence, the cable forces and the corresponding carriers' trajectories are not valid anymore to keep the load unperturbed. However, because the carriers are not controlled to generate a force on the cable but to follow a trajectory, we expect that} the load pose error remains bounded as the carriers keep tracking their periodic trajectories in a closed loop. \textcolor{black}{Since the load weight is compensated for by the $N-1$ remaining carriers, the cable forces must adapt to the missing carrier.} As the variable part of the cable forces depends on the Hamiltonian cycle selection, we may expect a different tolerance to cable detachment depending on the chosen cycle.
\textcolor{black}{If the cables are symmetrically attached around the CoM of the load and the Hamiltonian cycle is such that all cable forces are the same (more on this can be found in Sec.\ref{sec:uniform}), we expect that the error of the load pose in case of any carrier's failure is the same.}
\end{remark}

\begin{remark} \emph{(Choice of the frequency $\xi$)}
    From the proof of Proposition~\ref{prop_2}, we have that the minimum carriers' velocity $\underline{v}$ can be increased at will by increasing the frequency $\xi$ in~\eqref{eq:c1c2},~\eqref{eq:c1c2c3}. Note that increasing instead the amplitude A would increase $\overline{T}$ as well, thus defeating the purpose of increasing $\underline{v}$ in~\eqref{eq:vel_lower_bound}.
\end{remark}
\begin{remark}\label{rmk:coloring} \emph{(Choice of the library of periodic functions C)}
In Section~\ref{subsec:coloring}, we gave a library of 2 (for $n$ even) or 3 (for $n$ odd) sinusoidal functions to assign to $\lambda_i(t)$. 
Note that a library of $n$ functions $c_i(t)=A\cos{(\xi t+ \phi_i)}$ for $i=1,\cdots, n$ with phases all spaced by $\pi/n$, namely $\phi_i=\frac{\pi}{n}(i-1)$, can be used to ensure that the same assignment of $\lambda_i(t)$ is valid for any possible Hamiltonian cycle. As a drawback, the phases $\lambda_i(t)$ paired together are closer to each other, leading to a larger eccentricity of the force derivative ellipses in the plane $\text{span}(\delta_i, \overline{\delta}_i)$ and hence to large force (and carrier velocity)  variations.  Note also that a library of different bounded periodic functions that do not simultaneously attain stationary points may also be used. \textcolor{black}{Ultimately, optimizing the amplitude, frequency, and phase based on a desired objective function, e.g., to regulate the eccentricity of the ellipses and so the variations of the carriers' speed norm, is an interesting direction for future work.} 
\end{remark}
\subsection{\textcolor{black}{Uniformly Attached cables on a planar horizontal load.}}\label{sec:uniform}
\textcolor{black}{This section focuses more in depth on a case that attains special practical relevance: that of a payload where the $n$ cables are attached uniformly on a circle centered at the load center of mass, and for which the desired attitude is horizontal. Consider $n$ equal to an even number and apply the proposed coloring procedure. The complete analysis for any number $n$ of cables is more complex and is left for future work.  
Recalling that $\cos{(x+\pi/2)}=-\sin{(x)}$, and considering without loss of generality $\bm\delta_i=[1\,0\,0]^\text{T}$ and 
$\bar{\bm{\delta}}_i=[\cos{\alpha_i}\, \sin{\alpha_i}\,0]^\text{T}$, namely an angle $\alpha_i$ between the directions of the internal forces, the $\ith{i}$ cable force is equal to 
$$\bm{f}_i=\bm{f}_{0i}+\begin{bmatrix}A\cos{(\xi t)}\\0\\0\end{bmatrix}+\begin{bmatrix}
-A\sin{(\xi t)}\cos{\alpha_i}\\ -A\sin{(\xi t)}\sin{\alpha_i}\\0\end{bmatrix}$$ 
where $\bm{f}_{0i}=\begin{bmatrix}0&0&\bar{\star}\end{bmatrix}^\text{T}$ is constant by definition. 
Also, $$\dot{\bm{f}}_{i}=A\xi\begin{bmatrix}-\sin{(\xi t)}\\0\\0\end{bmatrix}+A\xi\begin{bmatrix}
-\cos{(\xi t)}\cos{\alpha_i}\\ -\cos{(\xi t)}\sin{\alpha_i}\\0\end{bmatrix}.$$ 
It holds that \begin{equation}||\bm{f}_i||_2=A\sqrt{\star^2 +1 - \cos{\alpha_i}\sin{(2\xi t)}}\label{norm_f}\end{equation} and 
$||\dot{\bm{f}}_i||_2=A\xi\sqrt{1+\cos{\alpha_i}\sin{(2\xi t)}}$, 
with $\star=\frac{\bar{\star}}{A}$. We recall that, for \eqref{eq:dpr_df} of the manuscript, 
\begin{equation}
 \label{eq:dpr_norm_}
||\dpR{i}(t)||_2=\frac{L_i}{||\bm{f}_i(t)||_2}||\dot{\bm{f}}^\perp_i||_2.
\end{equation}}

\textcolor{black}{\begin{prop}\label{prop_3}
Let be given a planar load with an even number of cables attached uniformly around its center of mass; let be given the coloring procedure \eqref{eq:c1c2}. If $\alpha_i=\pi/2$ then $||\dpR{i}||$ is constant 
while if $\alpha_i\neq \pi/2$ then $||\dpR{i}||$ is time varying and its minimum value is smaller  than the constant value it has for $\alpha_i=\pi/2$.
\end{prop}
\begin{proof}
    For $\alpha_i=\pi/2$, the force has a constant module, and its value is lower than the maximum value of $||\bm{f}_i(t)||_2$ for any $\alpha_i\neq\pi/2.$ Similarly, $||\dot{\bm{f}}_i||_2$ is constant, and its value is higher than the minimum value of $||\dot{\bm{f}}_i||_2$ for any $\alpha_i\neq\pi/2$. Since the force does not change in norm, $\dot{\bm{f}}_i=\dot{\bm{f}}_i^\perp$ for $\alpha_i=\pi/2$. As a consequence, and for \eqref{eq:dpr_norm_}, the norm of the vehicle's velocity is also constant for $\alpha_i=\pi/2$, and its value is higher than the minimum it can assume for any $\alpha_i\neq \pi/2$. 
\end{proof}
\begin{prop}\label{prop_4}
    Given the same assumptions of Prop. \ref{prop_3}. The minimum value of $||\dot{\bm{f}}_i||_2$  is the lower the further $\alpha_i$ is from $\pi/2$. Moreover, such minimum value happens for $t=\frac{3\pi}{4\xi}$ if $\cos{(\alpha_i)}>0$ and for $t=\frac{\pi}{4\xi}$ for $\cos{(\alpha_i)}<0$.
\end{prop}
\begin{proof}
From \eqref{norm_f}, $||\bm{f}_i||_2$ has its maximum value at $t=\frac{3\pi}{4\xi}$ for $\cos{\alpha_i}>0$ at $t=\frac{\pi}{4\xi}$ for  $\cos{\alpha_i}<0$, and such value is higher the further $\alpha_i$ is from $\pi/2$.
We write $\bm{q}_i=\bm{f}_i/||\bm{f}_i||$ and compute the component of $\dot{\bm{f}}_i$ parallel to $\bm{q}_i$ as \begin{align}&\bm{q}_i^\text{T}\dot{\bm{f}}_i=
&=\frac{A\xi\cos{\alpha_i}(\sin{\xi t}^2-\cos{\xi t}^2)}{\sqrt{\star^2 +1 - 
\cos{\alpha_i}
\sin{(2\xi t)}}}=||\dot{\bm{f}}_i^\parallel||
\end{align}
Using Pitagora's theorem, we have
\begin{align}
&||\dot{\bm{f}}_i^\perp||^2 =||\dot{\bm{f}}_i||^2 _2-||\dot{\bm{f}}_i^\parallel||^2=\nonumber \\&=A^2\xi^2\left(1+\sin{2\xi t}\cos{\alpha_i} -\frac{\cos{\alpha_i}^2(1-\sin{2\xi t}^2)}{\star^2+1-\sin{2\xi t}\cos{\alpha_i}}\right) \label{eq:fperp}
\end{align}
Defining $\sin{(2\xi t)}=\ut$, \eqref{eq:fperp} is a composite function $||\dot{\bm{f}}_i^{\perp}||^2=g(\ut)$, so its minima are those of $\ut$ and of $g()$. 
By defining $C=\cos{(\alpha_i)}$, $B=\cos{(\alpha_i)}^2$, and $D=1+\star^2$, from \eqref{eq:fperp} we have $g(\ut)=1 + C\ut -\frac{B(1-\ut^2)}{D-Cu}$. To find the minima of $g(\ut)$, we look at the zeros of its derivative
\begin{align}
    \frac{\partial g(\ut)}{\partial \ut}=\frac{C(D^2-B)}{(D-C\ut)^2}\label{eq:partial_g}
\end{align}
whose numerator is always different than zero.
In conclusion, the only minima of $||\dot{\bm{f}}_i^\perp||^2$ are those of $\ut=\sin{(2\xi t)}$. Hence, $||\dot{\bm{f}}_i^\perp||^2$, and for the monotonicity of the square root function also $||\dot{\bm{f}}_i^\perp||^2$ has its minimum value for $t=\frac{3\pi}{4\xi}$ if $\cos{(\alpha_i)}>0$ and for $t=\frac{\pi}{4\xi}$ for $\cos{(\alpha_i)}<0$. These correspond to the maxima of $||\bm{f}_i||$. The minimum value of  $||\dot{\bm{f}}_i^\perp||^2$ is lower the further $\alpha_i$ is from $\pi/2$. For what said and from \eqref{eq:dpr_norm_}, we conclude that the minimum value of $||\dpR{i}|| $ is for $||\dot{\bm{f}}_i^\perp||$  at its minimum, and hence $||{\bm{f}}_i||$ at its maximum, and it decreases as $\alpha_i$ gets further from $\pi/2$.   
\end{proof}}

\section{Numerical Results}\label{sec:sim}

\begin{figure}[t]
    \centering\includegraphics[width=0.44\linewidth,trim={4cm 0 4cm 0},clip ]{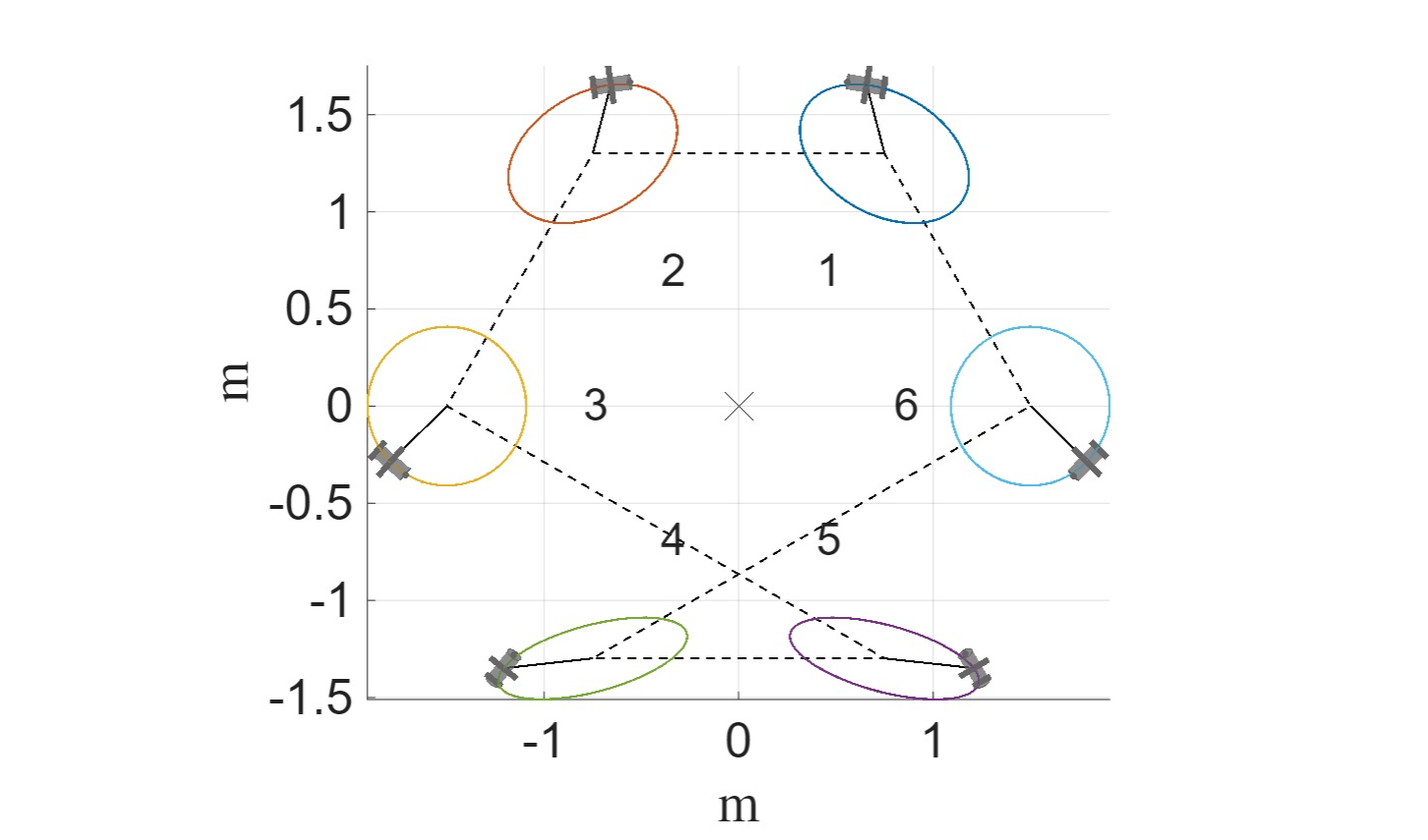}
\includegraphics[width=0.54\linewidth]{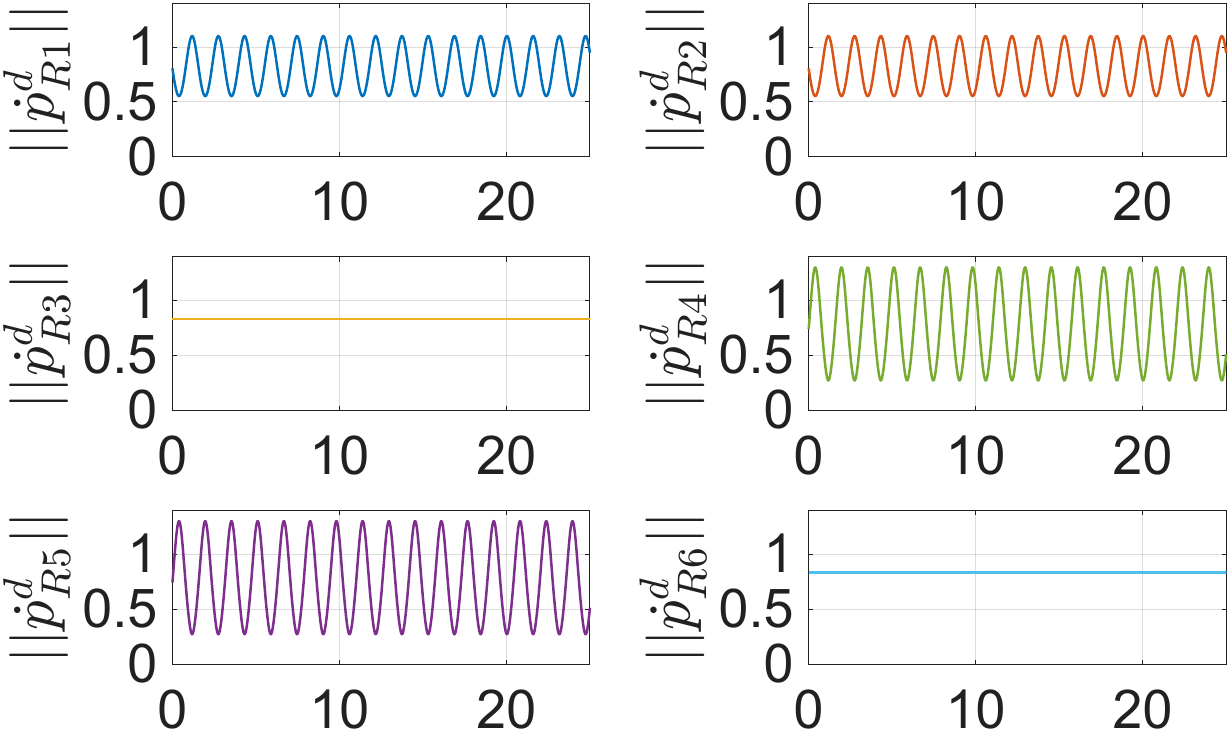}
    \caption{\textcolor{black}{Values of $||\dot{\bm{p}}_{Ri}^d||$ for a system with $n=6$ cables attached uniformly around a planar load, reported schematically on top. The carriers are numbered as indicated on the left, where a top view of the system is schematically shown. The carriers follow the colored planned paths; the cross indicates the load CoM, and dotted lines the edges of the Hamiltonian cycle. In the selected Hamiltonian cycle, the value of $\alpha_1=\alpha_2=\frac{3\pi}{2}$; $\alpha_3=\alpha_6=\frac{\pi}{2}$, and, as expected, the two vehicles numbered 3 and 6 have desired velocities with constant norms;  $\alpha_4=\alpha_5=\frac{\pi}{6}$, namely, they are the most distant from $\frac{\pi}{2}$, and the norm of these vehicle's velocity has the smallest minimum value. }}
    \label{fig:6N-alpha}
\end{figure}

\begin{figure}[t]
\includegraphics[width=0.48\columnwidth,trim={3cm 0 4cm 0},clip ]{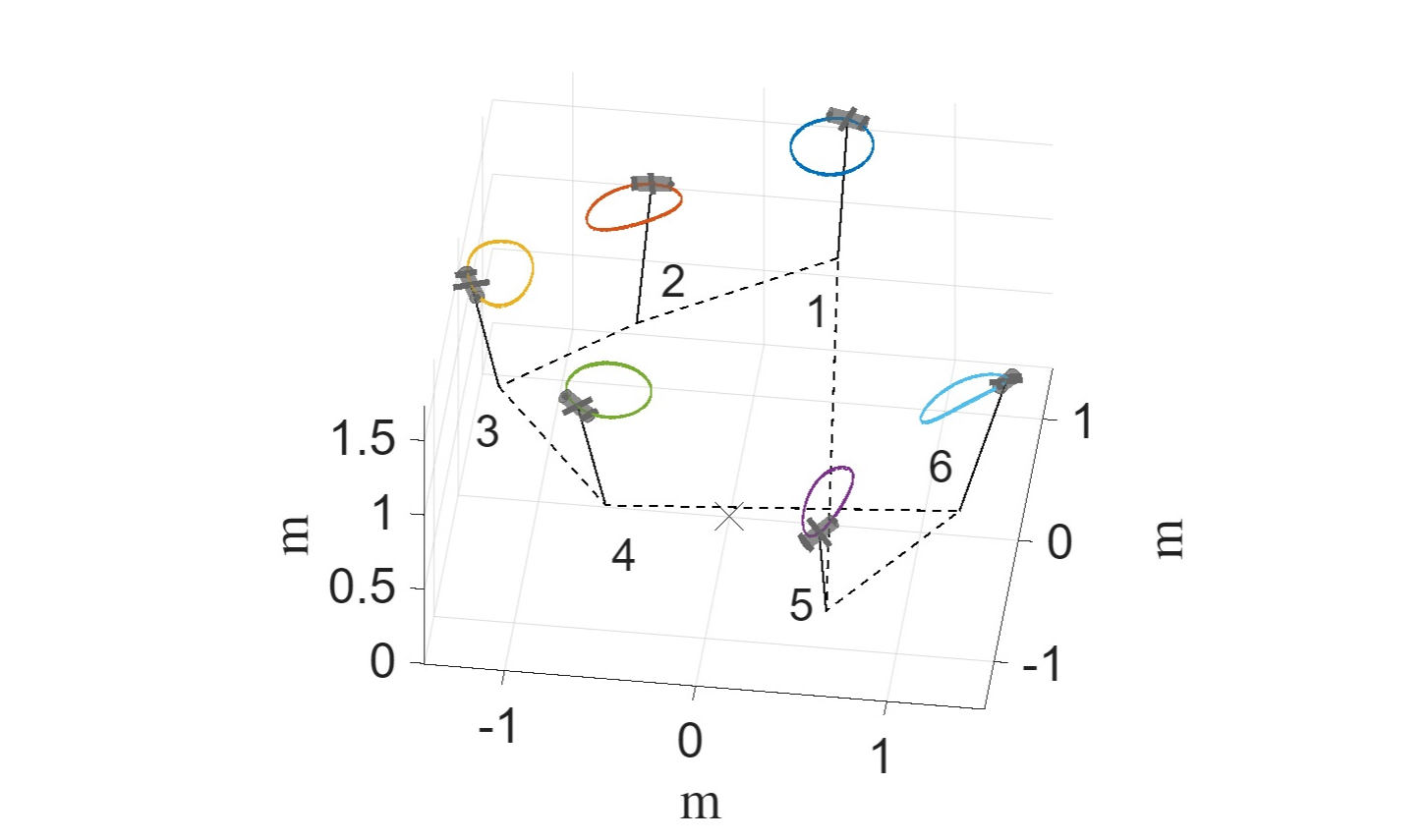}
\includegraphics[width=0.5\columnwidth]{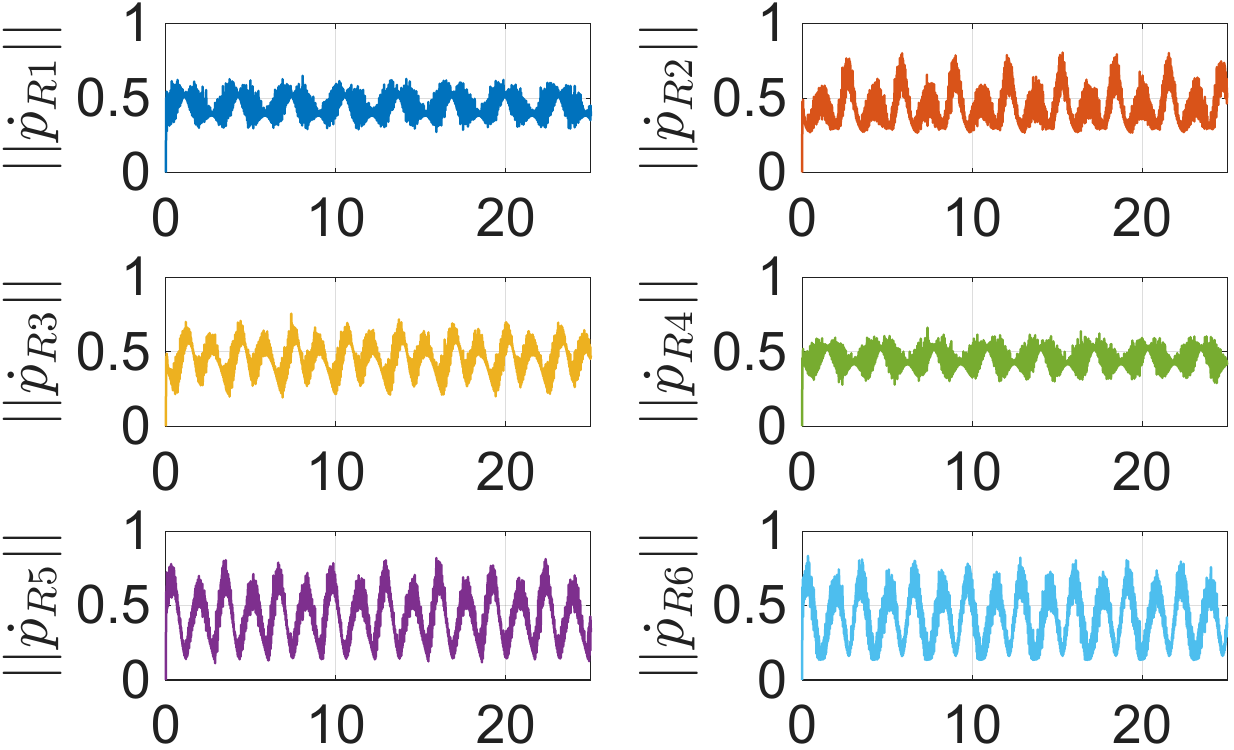}\\\includegraphics[width=0.5\columnwidth]{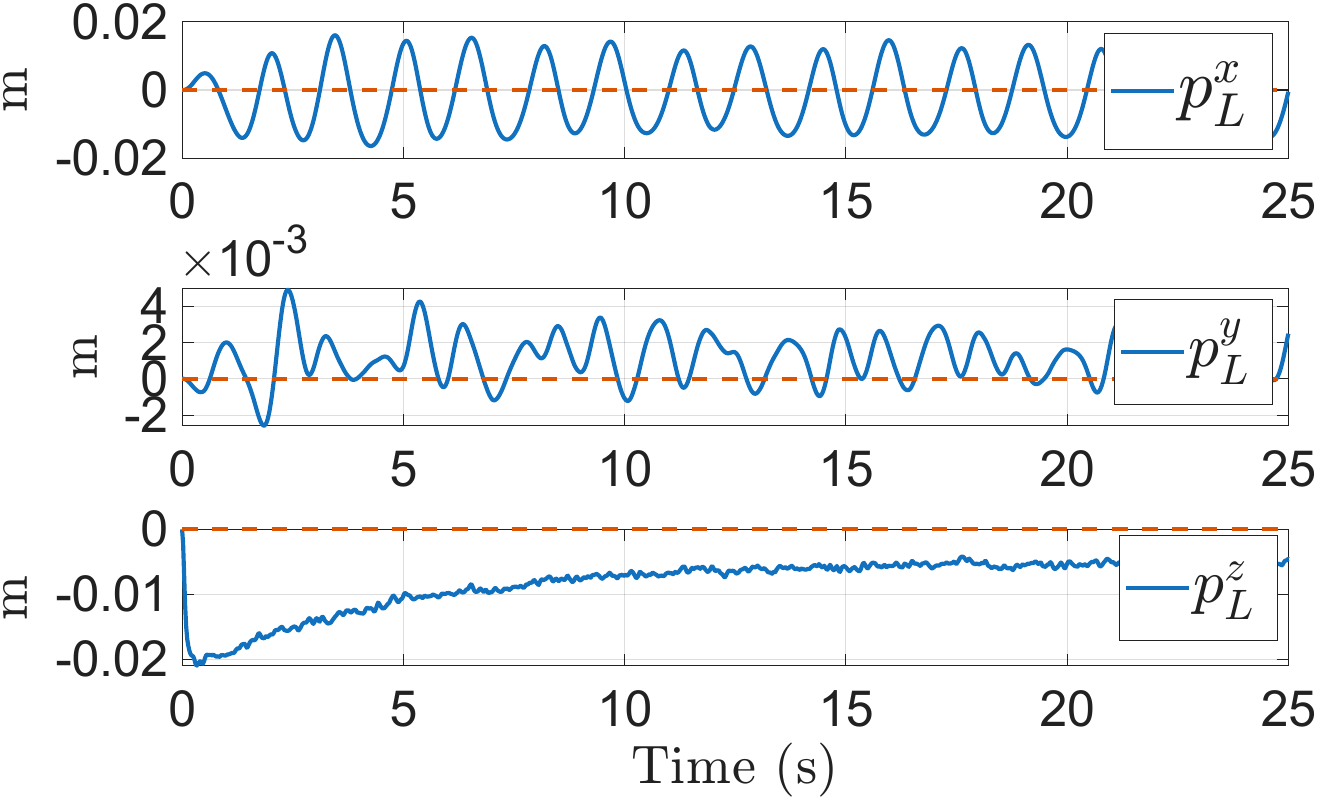}\includegraphics[width=0.5\columnwidth]{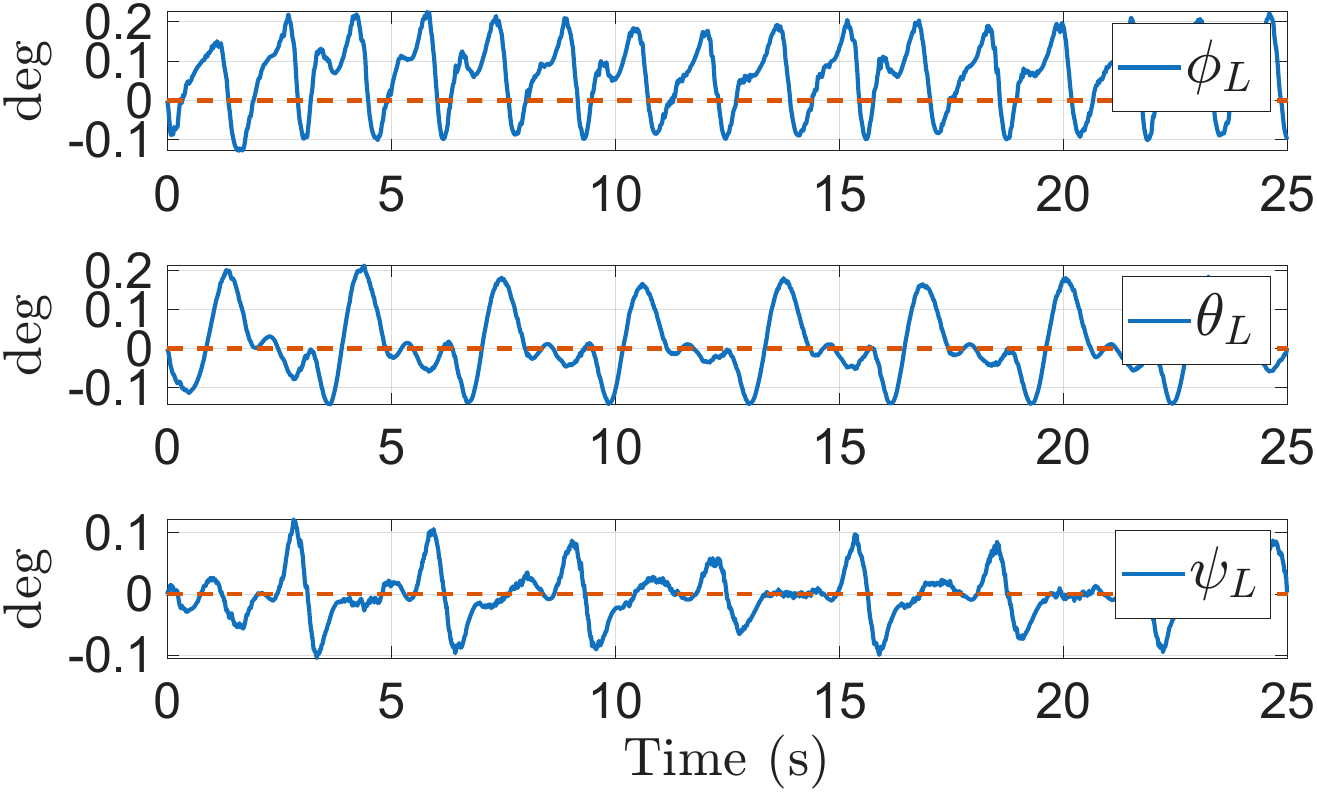}
\caption{\textcolor{black}{6 carriers manipulating a load. In the first column, a schematic representation.
 In the second column, $|\dpR{i}(t)|>0$. In the third and fourth columns, the position and attitude of the load are reported. A small error in the vertical position of the load is caused by the elastic cables' elongation, unknown to the carrier.}}\label{fig:sim_4-6}
\end{figure}

\begin{figure}[t]
    \centering    \includegraphics[width=0.65\linewidth]{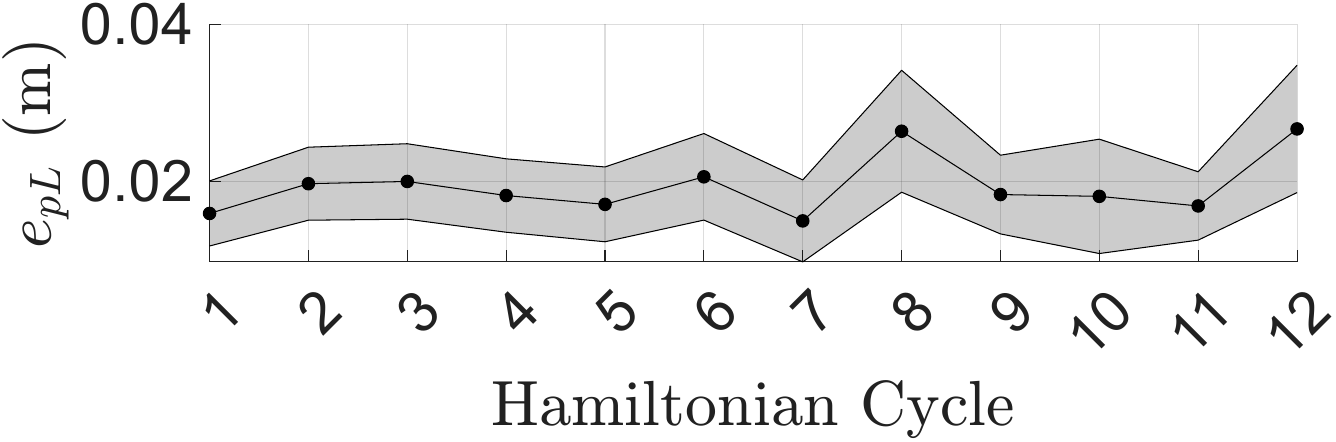}
\includegraphics[width=0.65\linewidth]{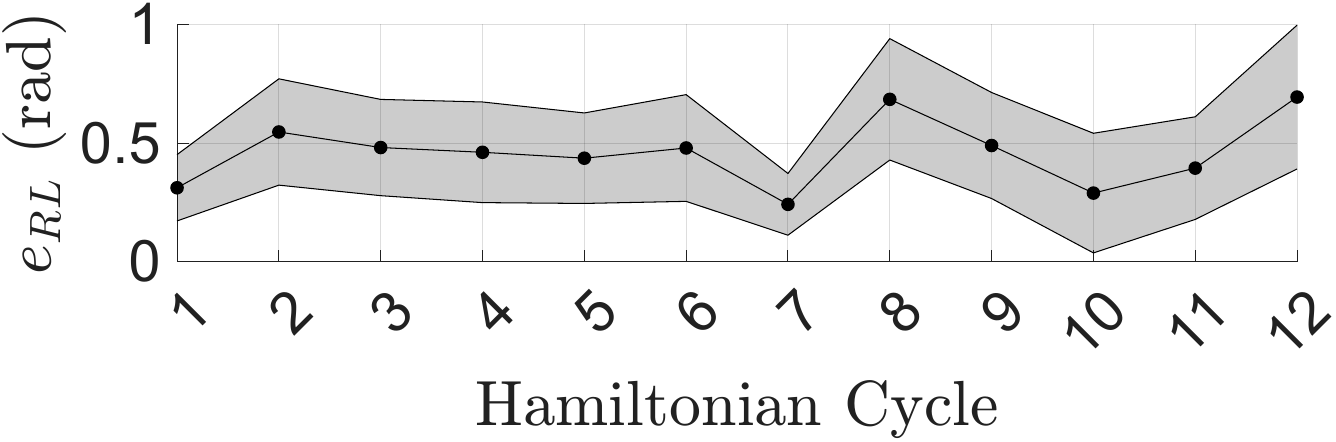}
    
    \caption{\textcolor{black}{Each point corresponds to the mean value of the position(left) and attitude(right) load error when one of the 12 Hamiltonian cycles is selected for a system of $N=5$ carriers. The colored bands correspond to the standard deviation. The choice of the Hamiltonian Cycle has a minimum effect on the resulting load pose errors.}}
    \label{fig:ham_n5}
\end{figure}

In this section, we show simulation results that validate and reinforce the analysis discussed earlier.

The simulations have been carried out in Matlab-Simulink. Each flying carrier is simulated as a double integrator controlled via a PID feedback law to follow a desired trajectory. 
The carriers' mass is $m=$ 0.1 kg, and the proportional, derivative, \textcolor{black}{and integral} gain matrices are diagonal matrices with values on the diagonal of $100$\,N/m and $10$\,Ns/m, \textcolor{black}{$15\,$ Ns$^2$/m} respectively. To simulate sensor noise, a Gaussian random signal is added to the position and velocity of the carriers, with a standard deviation of $0.005$\,m and $0.01$\,m/s, respectively. The suspended rigid body load has a mass of $1$\,kg and diagonal rotational inertia with diagonal elements equal to 0.01 N$\cdot$m/s. 
Viscous friction with a coefficient of $0.1$\,Ns/m is added to the load's translational and rotational dynamics to simulate friction with the air. The cables are modeled as linear springs with a rigidity of $K_c=500$\,N/m, damping coefficient $B_c=1$\,Ns/m, and a rest length of $0.8$\,m. The load equilibrium configuration is chosen without loss of generality in $\pLEq=[0,\,0,\,0]^\top$, $\rotMatLEq{}=\eye{3}$.

\subsection{The role of the Hamiltonian cycle on the carriers' velocity}

\textcolor{black}{In this section, we validate the results of Sec \ref{sec:uniform}.
We simulate a system with $n=6$ evenly distributed cable attachment points around the center of mass of a planar horizontal load, on a circumference of radius 1.2 m. 
Figure~\ref{fig:6N-alpha} shows how $||\dpR{i}||$ is constant for $\alpha_i=\pi/2$, namely for orthogonal directions of the internal force of the $\ith{i}$ cable. It also shows how its constant value is higher than the minimum value assumed for $\alpha_i\neq\pi/2$, and how such minimum is lower the further $\alpha_i$ gets from $\pi/2$.}

\subsection{\textcolor{black}{Load pose error}}
From now on, the attachment points of the cables on the load are randomly generated in 3D around the position of its CoM and at different altitudes as follows: 
\begin{equation}^B\bm{b}_i=[[\rotMat{\text{z}}(2\pi i/n+\zeta_i)[1.2,0]^\top]^\top,\hat{z}_i]^\top,\label{eq:b_selection_sim}
\end{equation} 
where $\rotMat{\text{z}}()$ is the elementary rotation around $\vE{3}$, $\zeta_i$ is a random angle in $[0,\, 0.2]$\,rad, and $\hat{z}_i$ a random number in $[0,\,1]$\,m, and $n=5$.

This section focuses on the objective of leaving the pose of the suspended load unperturbed. 
The proposed algorithm provides nominal carriers' trajectories that do not perturb the load's equilibrium pose, in a similar way in which internal joint motions of a redundant robotic manipulator leave the end-effector still. However, there may be non-idealities such as noise, tracking errors, etc. that produce a motion of the load. Figures~\ref{fig:sim_4-6} report the results of 6 non-stop carriers keeping the load at its constant pose, where the roll, pitch, and yaw angles describing the load's attitude are respectively indicated as $\phi_L,$ $\theta_L,$ and $\psi_L.$  Note that the norm of each carrier velocity is always greater than zero during the task execution. \textcolor{black}{The load pose error is under 2 degrees for the roll, pitch, and yaw, and the components of the position error are below a few centimeters. A more extensive analysis of the load pose error is provided in the following.}

We may expect that the load's pose error sensitivity to the non-idealities is different for different choices of the Hamiltonian cycle. 
\textcolor{black}{To assess this, we ran 12 simulations,  one for} each of the $(5-1)!/2=12$ Hamiltonian cycles with $N=5$. 
The phases of the periodic functions used to assign $\lambda_i(t)$ have been chosen as in Remark \ref{rmk:coloring} to be able to test any Hamiltonian cycle. 

In Fig.~\ref{fig:ham_n5}, the results are reported. \textcolor{black}{The load position error is defined as 
$\bm{e}_{pL}=||\pL-\pLEq||$ and the load attitude error as $\bm{e}_{RL}=|\psi_L|+|\theta_L|+|\phi_L|$. The dots represent the error mean value and the colored band its standard deviation.}
These results suggest that the selection of the Hamiltonian cycle has an effect on the performance of the proposed method, but that, despite it, the load pose error is kept low.

\textcolor{black}{Moreover, we tested the sensitivity of the load pose to parameter uncertainties besides the noise already present in all the simulations, as described earlier. Specifically, the load mass $m_L$, the carriers' masses $m$, the cable attachment points on the load with positions $\bm{b_i}$ for $i=1,...,N$, and the cable rest length $l_0$ as known to the carriers are perturbed by a relative error as follows. If we consider the parameter $p$, its uncertain value is computed as $p=p(1+\Delta_{\%p})$. Relative errors $\Delta_{\%p}\in[-40\%$, $+40\%]$ have been considered. The results for a system with $N=6$ carriers are reported in Figure~\ref{fig:uncert}.}

\begin{figure}[t]
\includegraphics[width=0.24\textwidth]{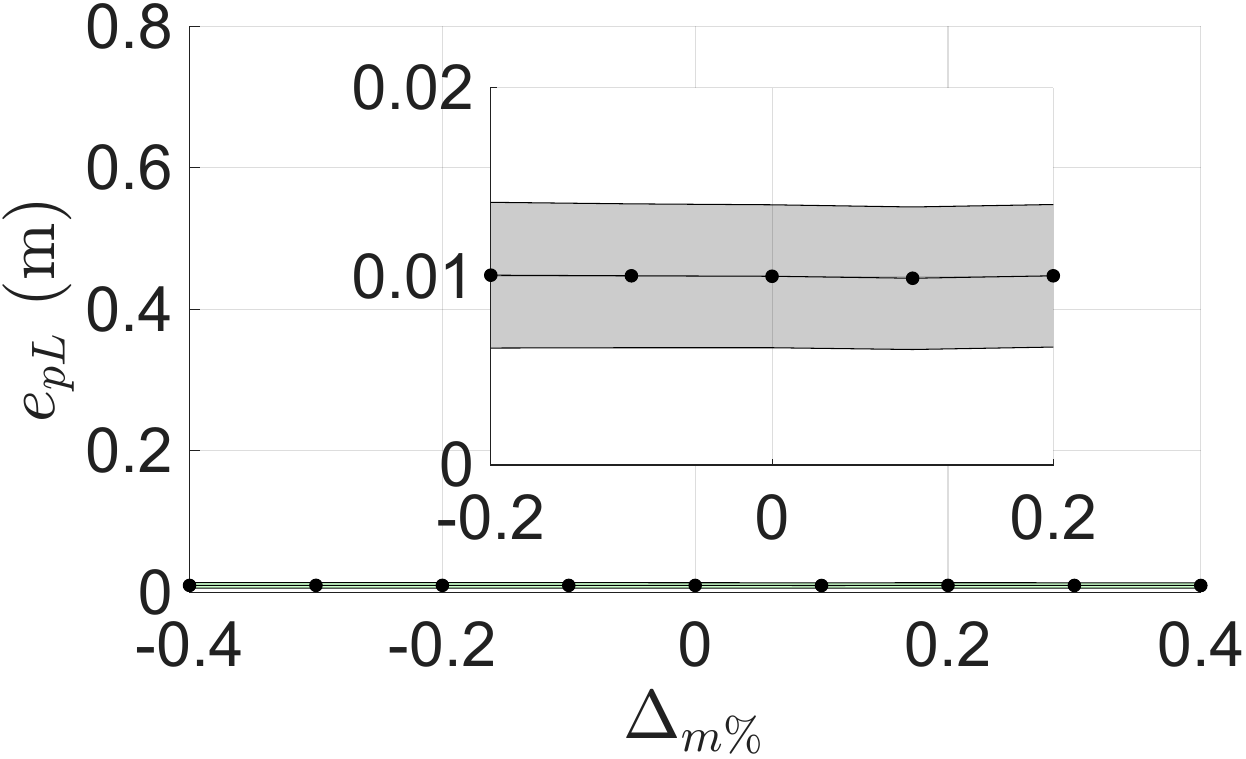} \includegraphics[width=0.24\textwidth,]{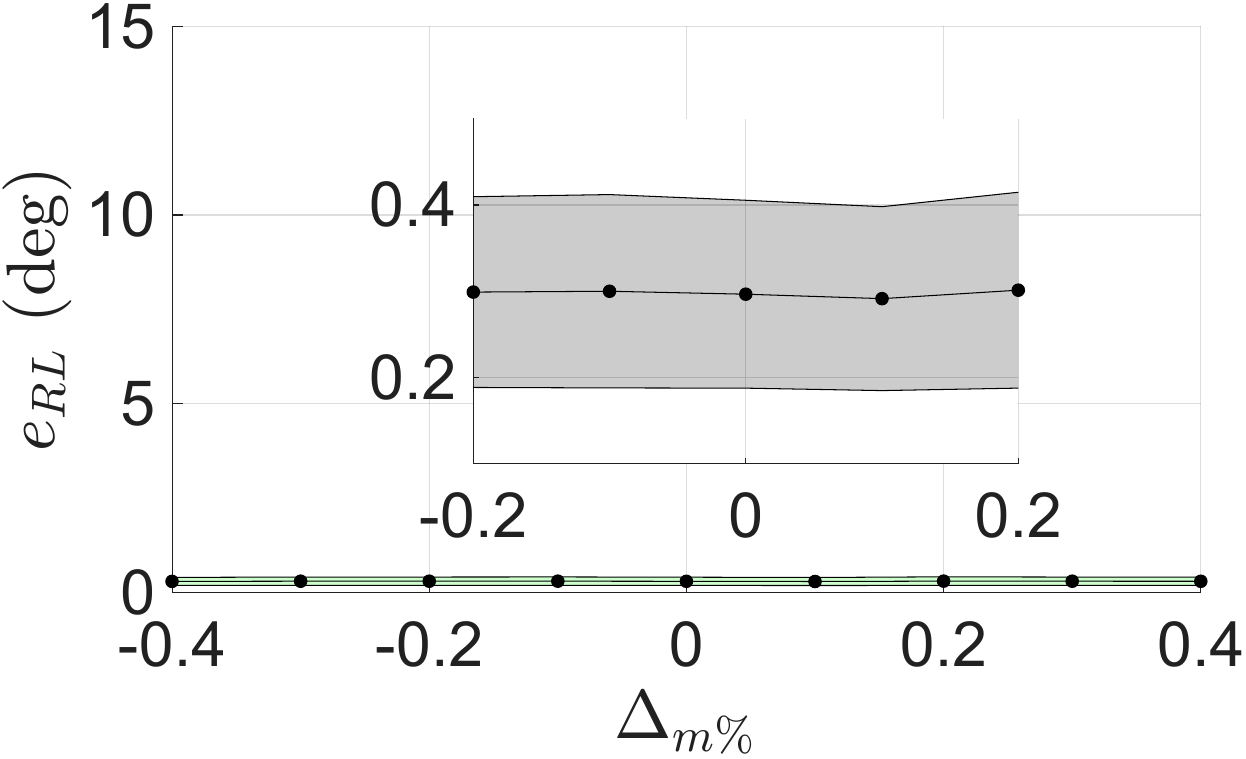}\\
\includegraphics[width=0.24\textwidth]{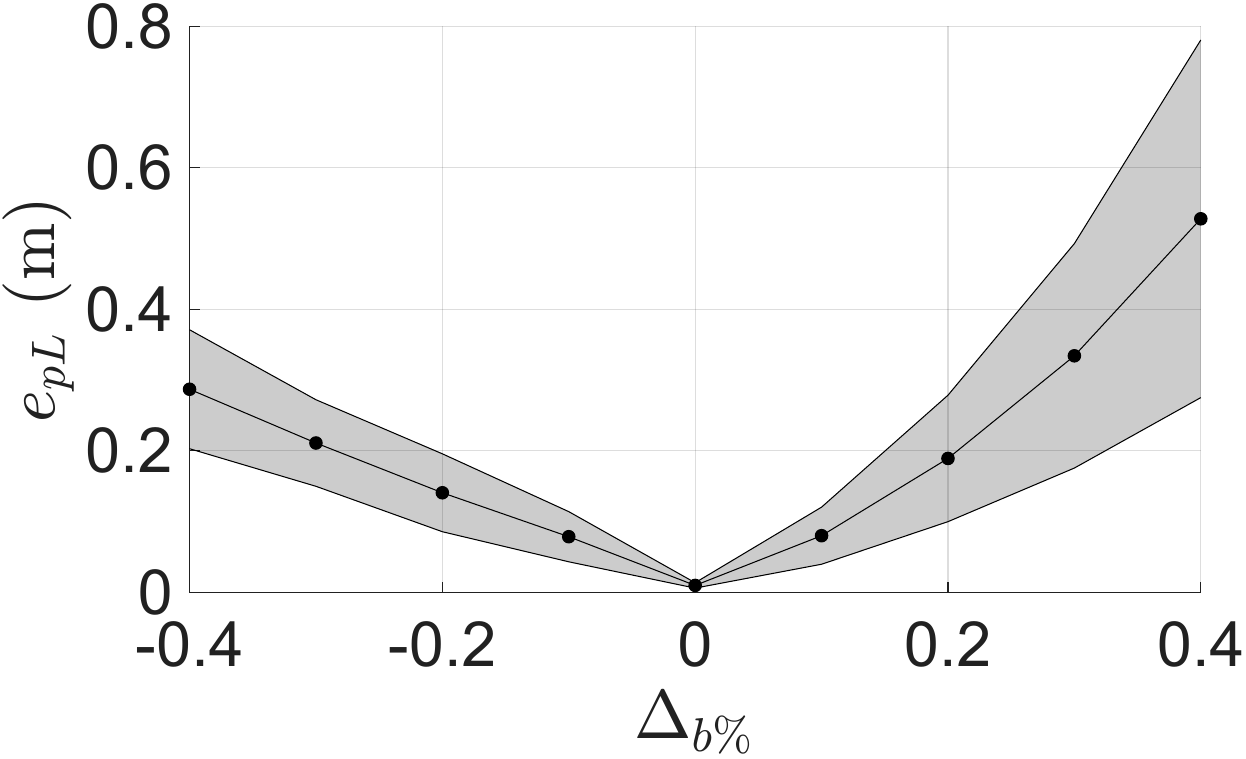} \includegraphics[width=0.24\textwidth]{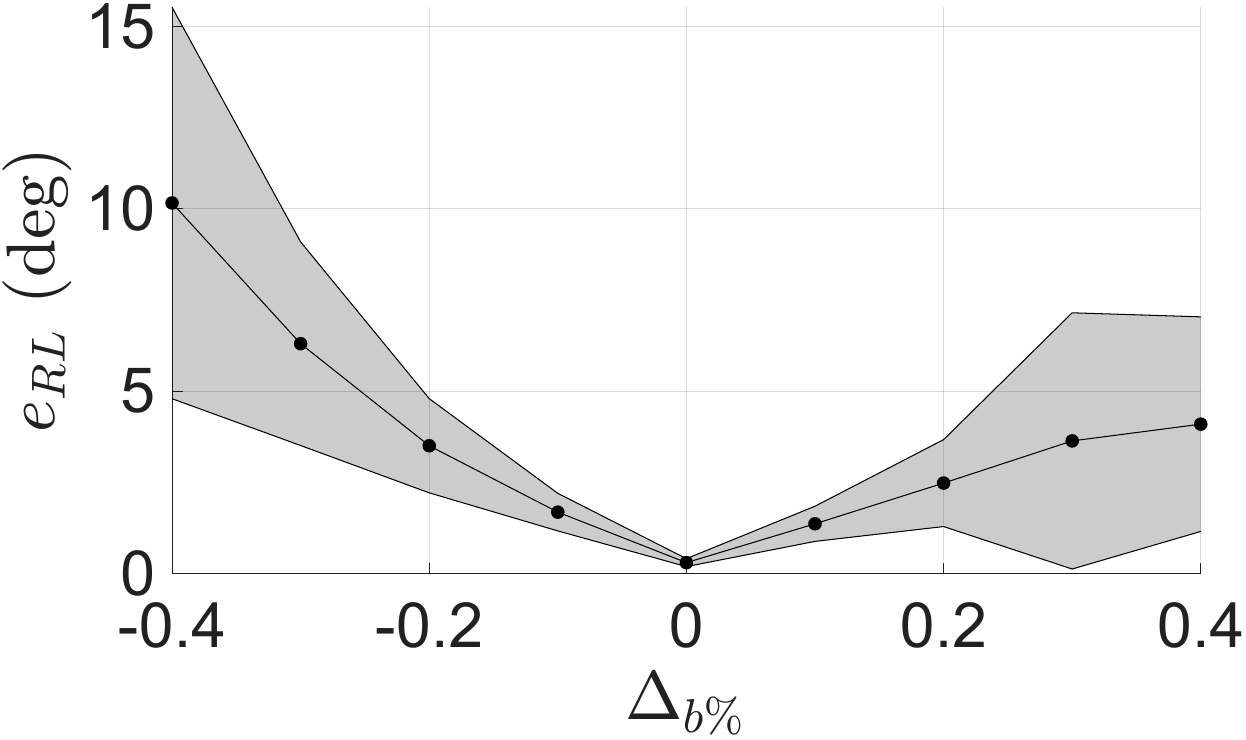}\\
\includegraphics[width=0.24\textwidth]{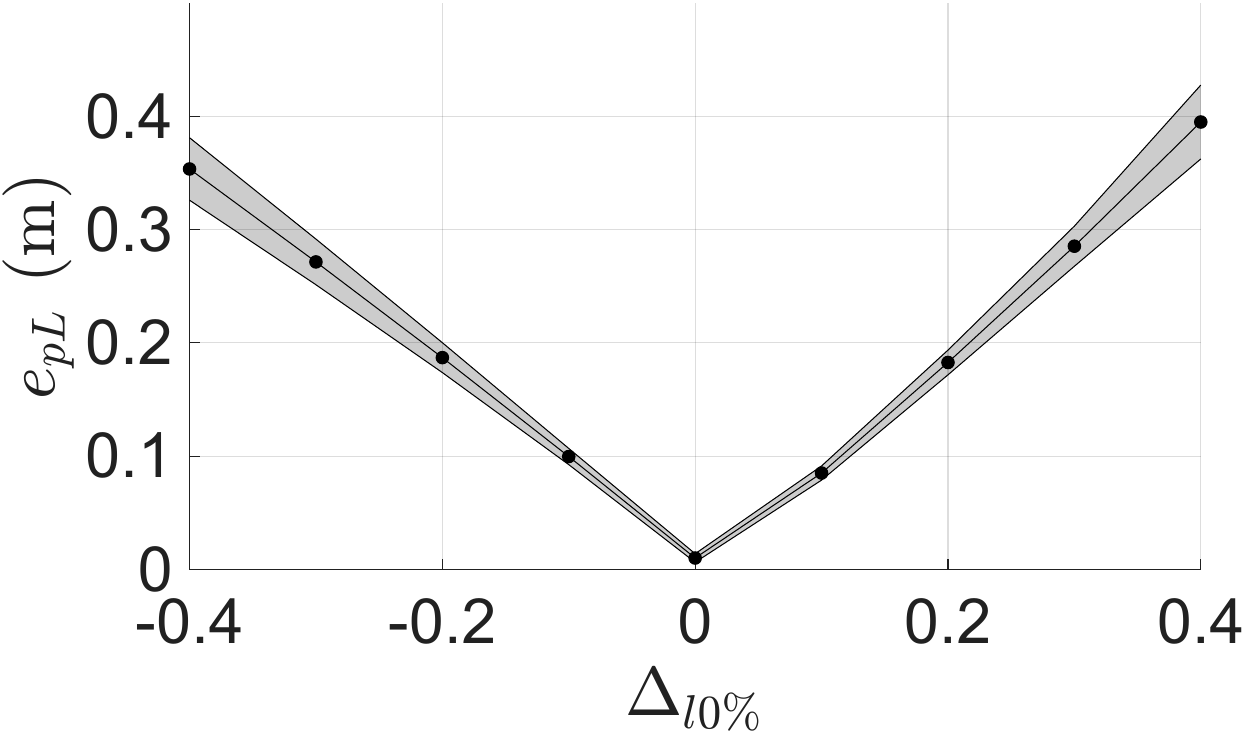}  \includegraphics[width=0.24\textwidth]{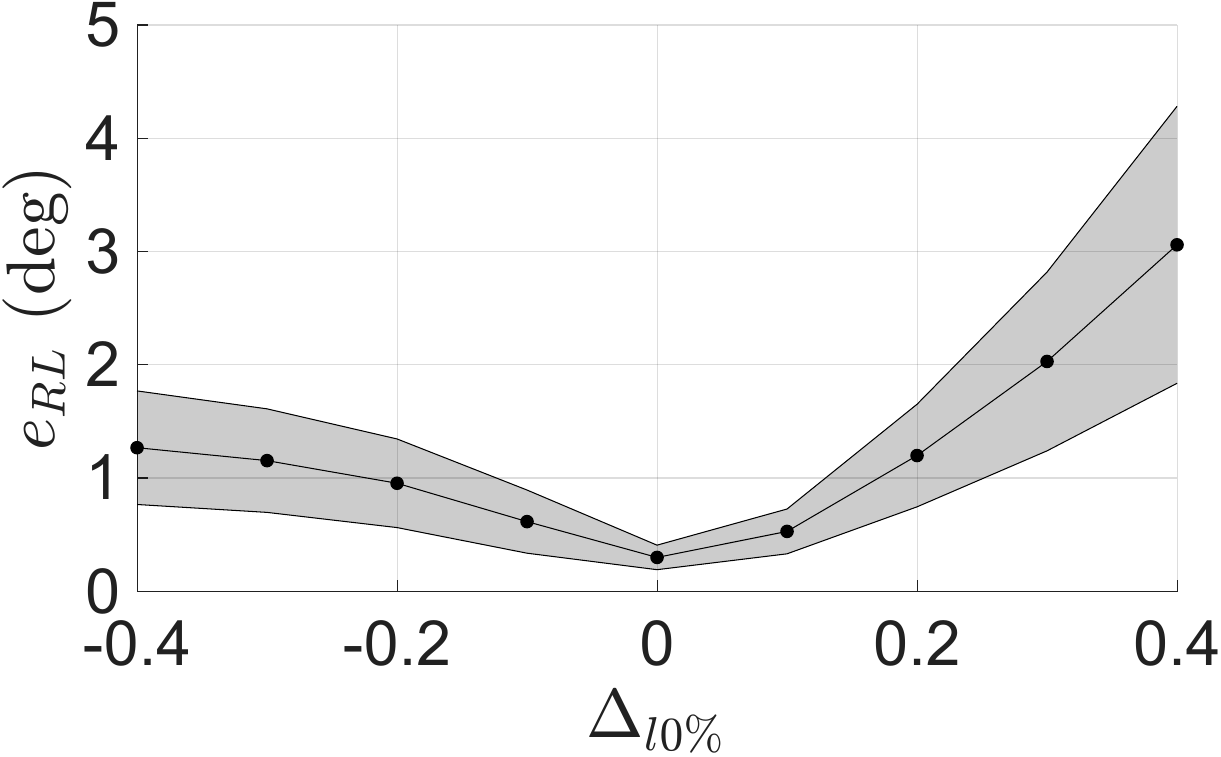}\\
\includegraphics[width=0.24\textwidth]{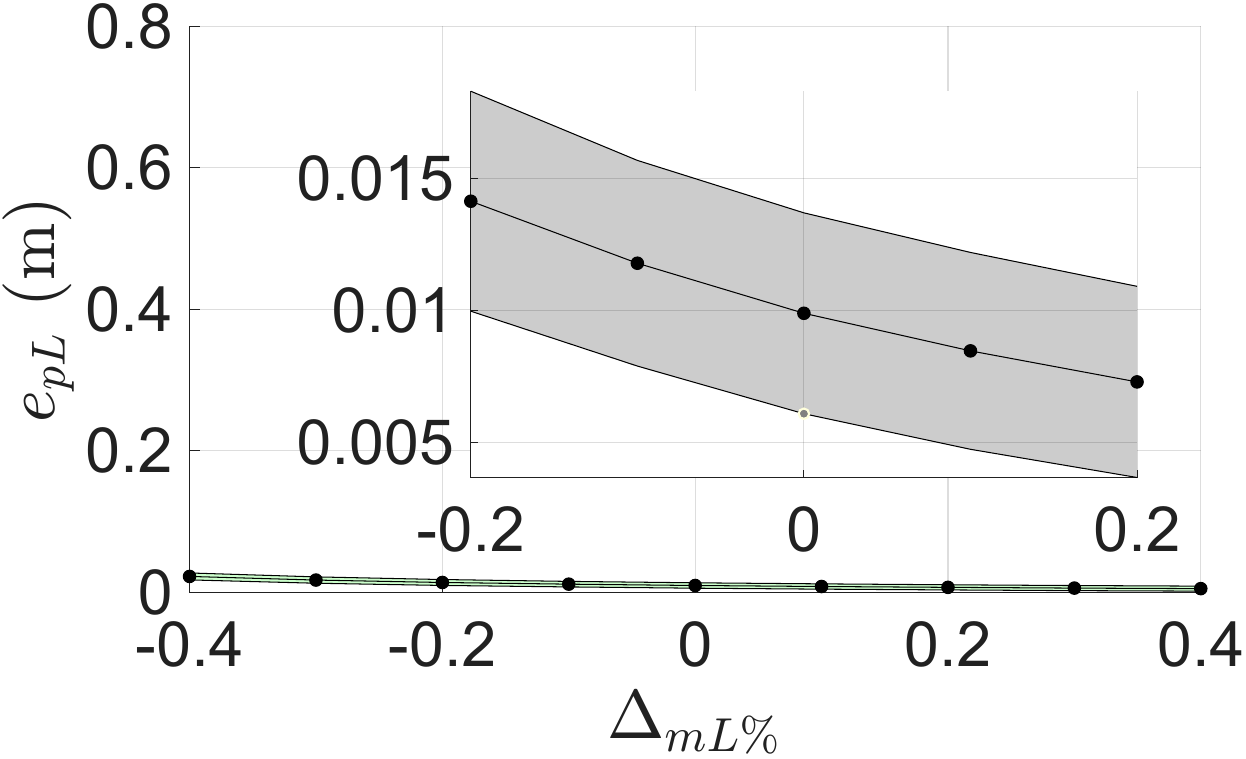} \includegraphics[width=0.24\textwidth]{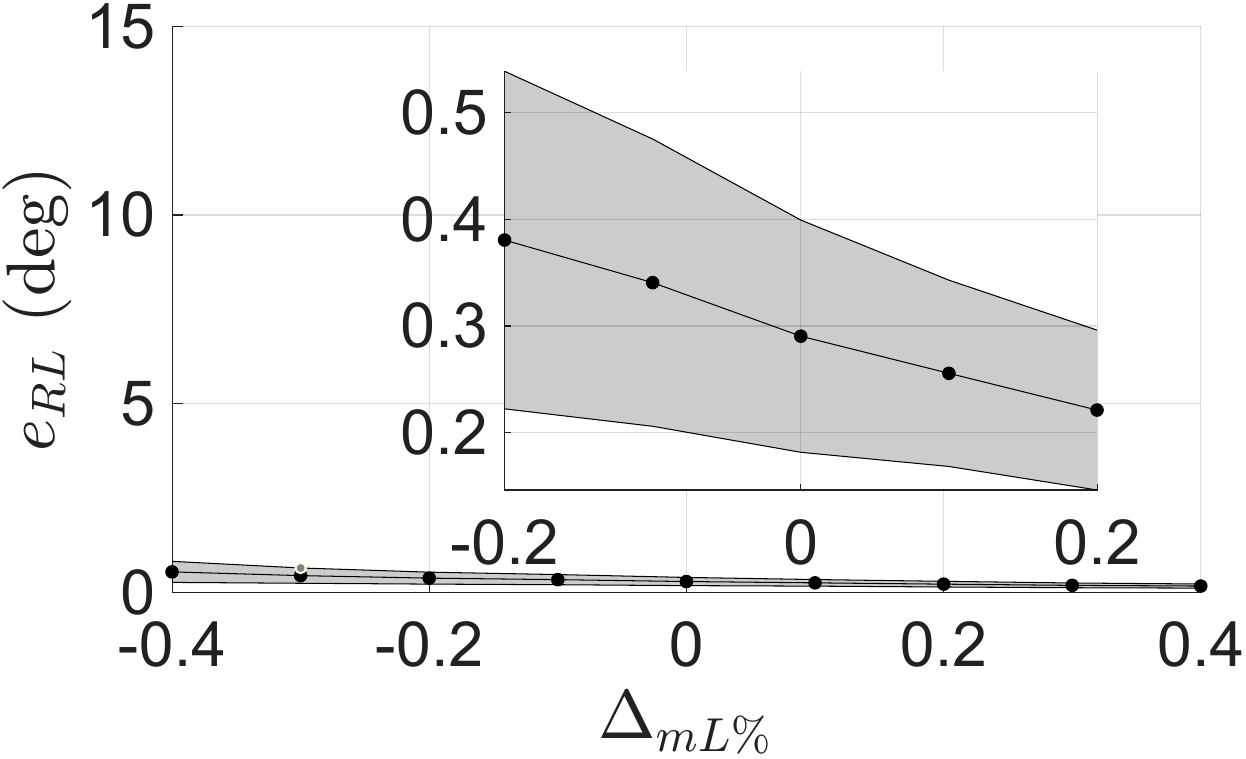}
     \caption{\textcolor{black}{Mean value (dots) and standard deviation (bands) of the load pose errors $\bm{e}_{pL}$  $\bm{e}_{RL}$ for a system with $N=6$ carriers and a relative parameter error as displayed in the x-axes of the plots.} }\label{fig:uncert}
\end{figure}

\textcolor{black}{From the results, we observe that, as expected, an error in the carriers' mass does not reflect an increased load pose error, because the control loop on the carriers' trajectories is closed on their state with an integral term. }

\textcolor{black}{Instead, an error on the cable length or attachment point is not corrected by the method, as no feedback control is performed on the load's state. Hence, the greater the parameter error is, the greater the norm of the load pose errors. The interested reader can find in \cite{girardello2025trajectory} a method proposed by the authors that closes the loop on the load's pose exploiting an online optimization. }

\textcolor{black}{Interestingly, the load pose error caused by an uncertain load mass $m_L$ is not symmetric: if the nominal value of $m_L$ is larger than its actual value, the resulting pose error decreases. We explain this by observing that a large nominal value of the load mass causes the trajectories of the carriers to be narrower, leading to lower tracking errors for the carriers. This is because the effect of the internal force becomes negligible compared to the external wrench in \eqref{eq:forces_mu}. Indeed, the decreasing trend of the load pose errors when increasing the nominal load mass disappears if the simulator is modified to let the cable end-points track the desired trajectories exactly. Eventually, note how the greatest load pose errors caused by $\Delta_{\%mL}\neq0$ are an order of magnitude smaller than those caused by the other considered uncertainties.}

\subsection{Role of the Hamiltonian cycle in case of cable detachment}

\begin{figure}[t]
     \includegraphics[width=0.48\textwidth,trim={3cm 9cm 3cm 8.5cm},clip]{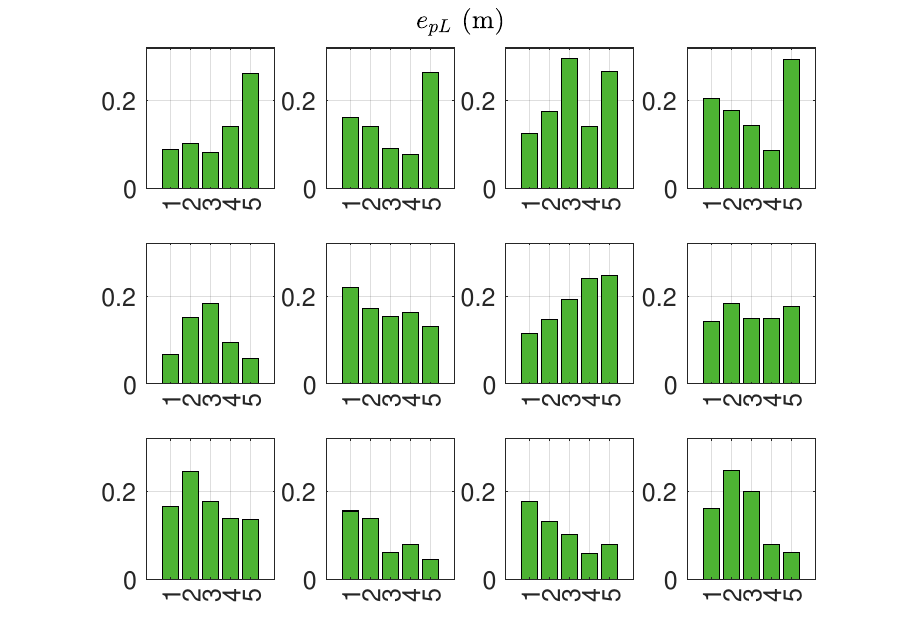}
      \includegraphics[width=0.48\textwidth,trim={3cm 9cm 3cm 8.5cm},clip]{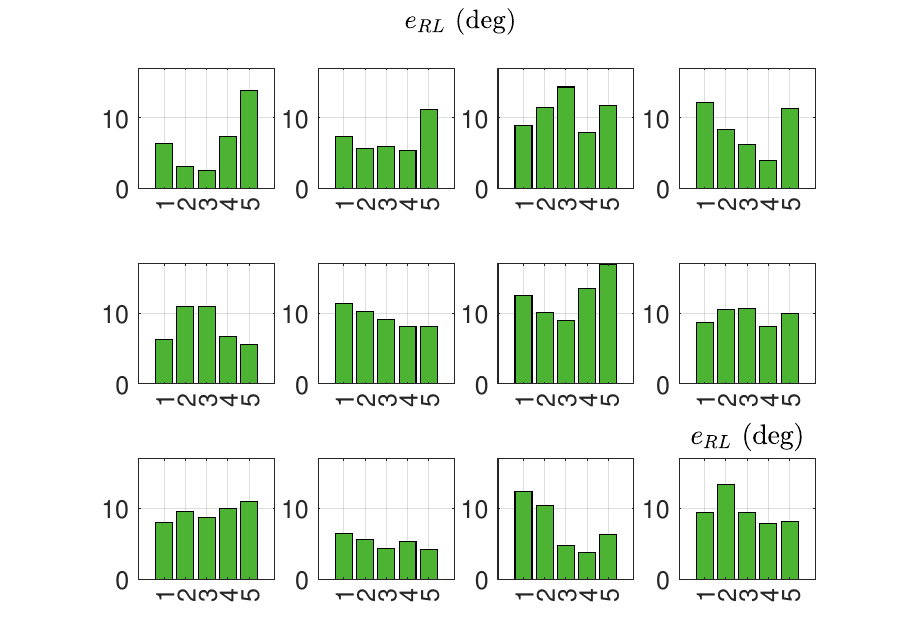}
      \caption{$\bm{e}_{pL}$ (on the bottom) and $\bm{e}_{RL}$ (on the top) for each of the 12 Hamiltonian cycles in a system with $n=5$ carriers. Each bar corresponds to the value of the error when one of the carriers is faulty: it is detached from the system.}\label{fig:faulty}
 \end{figure}

We simulated the case in which one cable breaks and the remaining $n-1$ carriers keep moving along the nominal trajectories. 
We carried out 60 simulation runs on a system with 5 carriers. Of the 60 runs, there are 5 for each of the 12 Hamiltonian cycles. For each cycle, 5 tests have been carried out, in which each of the cables detaches. The results, reported in Fig.~\ref{fig:faulty}, show that, despite a little variability, the errors are overall similar and do not seem to justify a specific choice of the cycle. 
\textcolor{black}{Additional simulation results performed on a planar horizontal load with symmetrically distributed cables confirmed that the error induced on the load by a cable detachment does not depend on the specific cable when $\alpha_i$ are the same for all robots.}

\subsection{\textcolor{black}{Translating load}}
\textcolor{black}{Let us assume that one wishes to apply the same proposed method to a moving load. Note that for a moving load, it may still be desirable to apply the proposed method when it is not feasible for the carriers to simply follow the motion of the load, for instance, if the load moves too slowly for them to keep a desired lower-bound forward speed. In general, the application is not straightforward and requires considerable adjustments of the proposed method, as the authors also elaborate in \cite{girardello2025trajectory}. However, the proposed method can be applied under the simplifying assumptions that (i) the load moves through consecutive static poses, and that (ii) its orientation remains unchanged. Based on the first assumption,~\eqref{eq:f_t_form} holds; based on the second assumption,~\eqref{eq:fidot_mudot} holds, because the grasp matrix is constant, being dependent on the orientation of the load but not its position. The load desired trajectory is a third-order polynomial one along the X-direction, moving by 2 meters in 10 seconds and then remaining static for the rest of the time. In figure~\ref{fig:moving}, we report the results obtained applying the proposed method to a time-varying $\pL$ in \eqref{eq:kinematics}.}

\begin{figure}[t]
    \centering
    \includegraphics[width=0.9\linewidth]{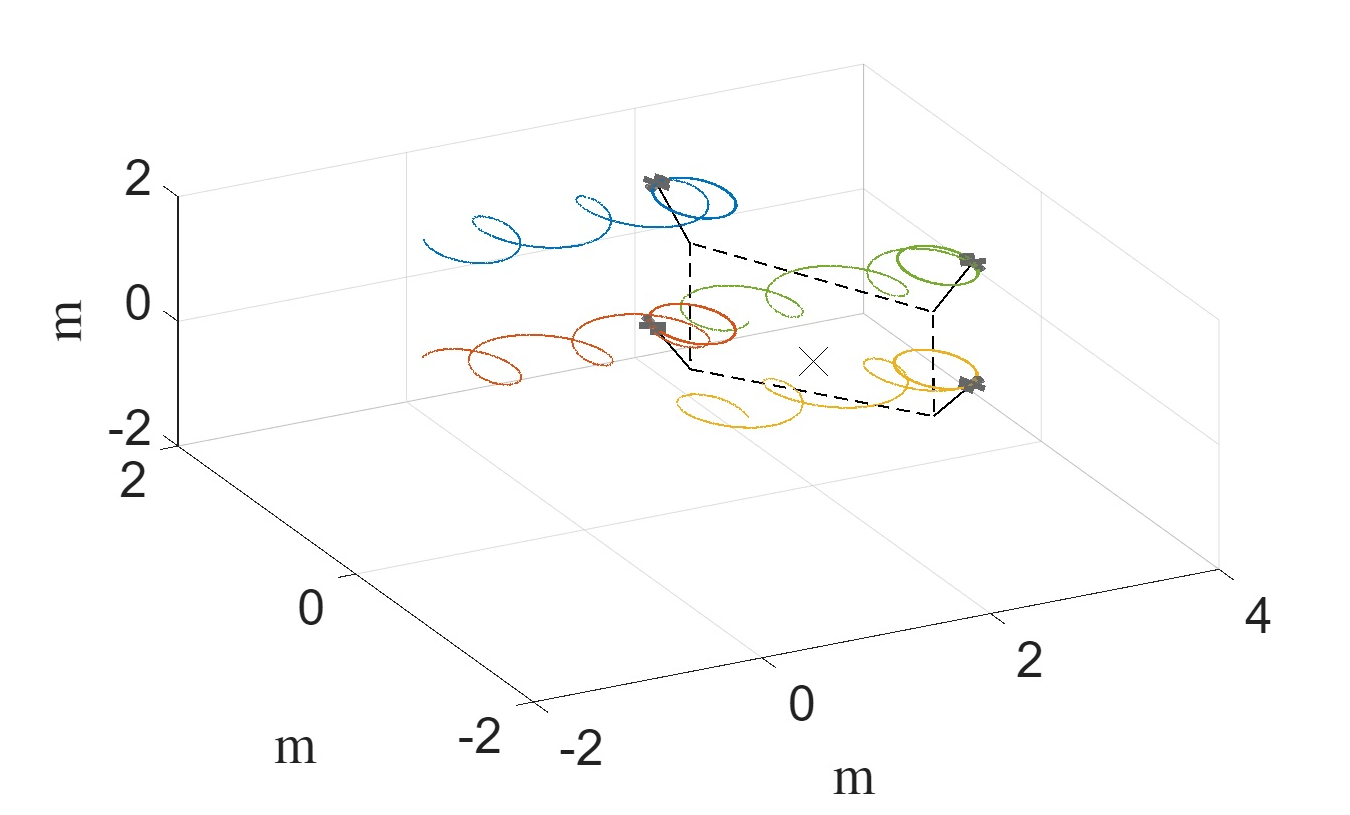}
    \includegraphics[width=0.8\linewidth]{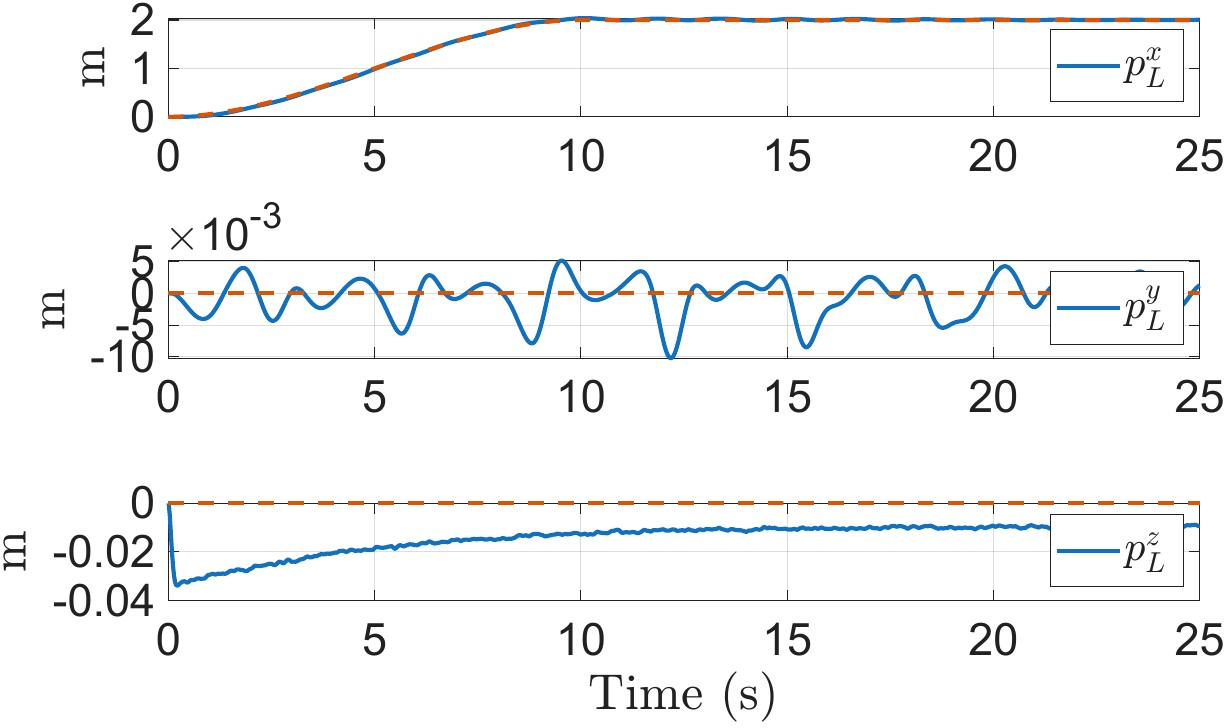}
    \includegraphics[width=0.8\linewidth]{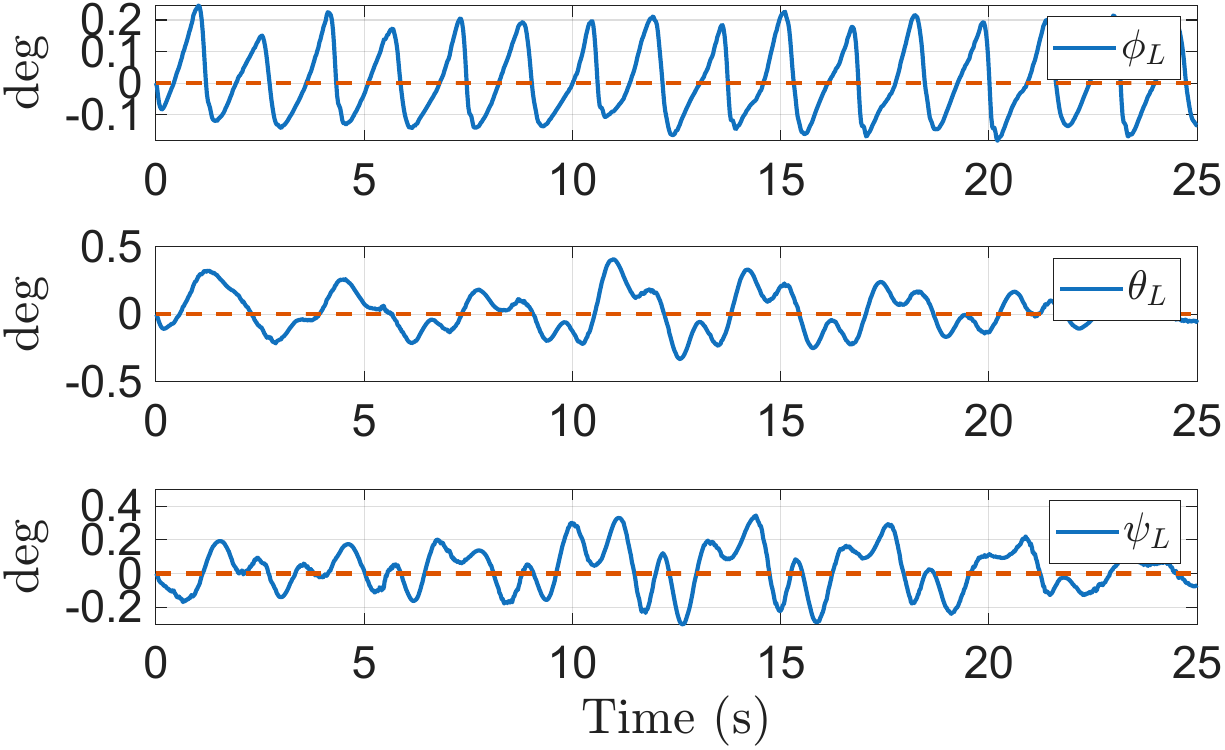}
    \caption{\textcolor{black}{Application of the proposed method to a translating load. The carriers steer the load along the desired trajectory and keep it at the desired final pose without perturbing its orientation. A small error in the load's altitude is caused by the elastic cables' elongations, unknown to the carriers. }}
    \label{fig:moving}
\end{figure}

\subsection{\textcolor{black}{Simulations with fixed-wing vehicles}}

\begin{figure*}[t!]
    \centering
\includegraphics[width=\linewidth, trim={0 1cm 0 1cm},clip]{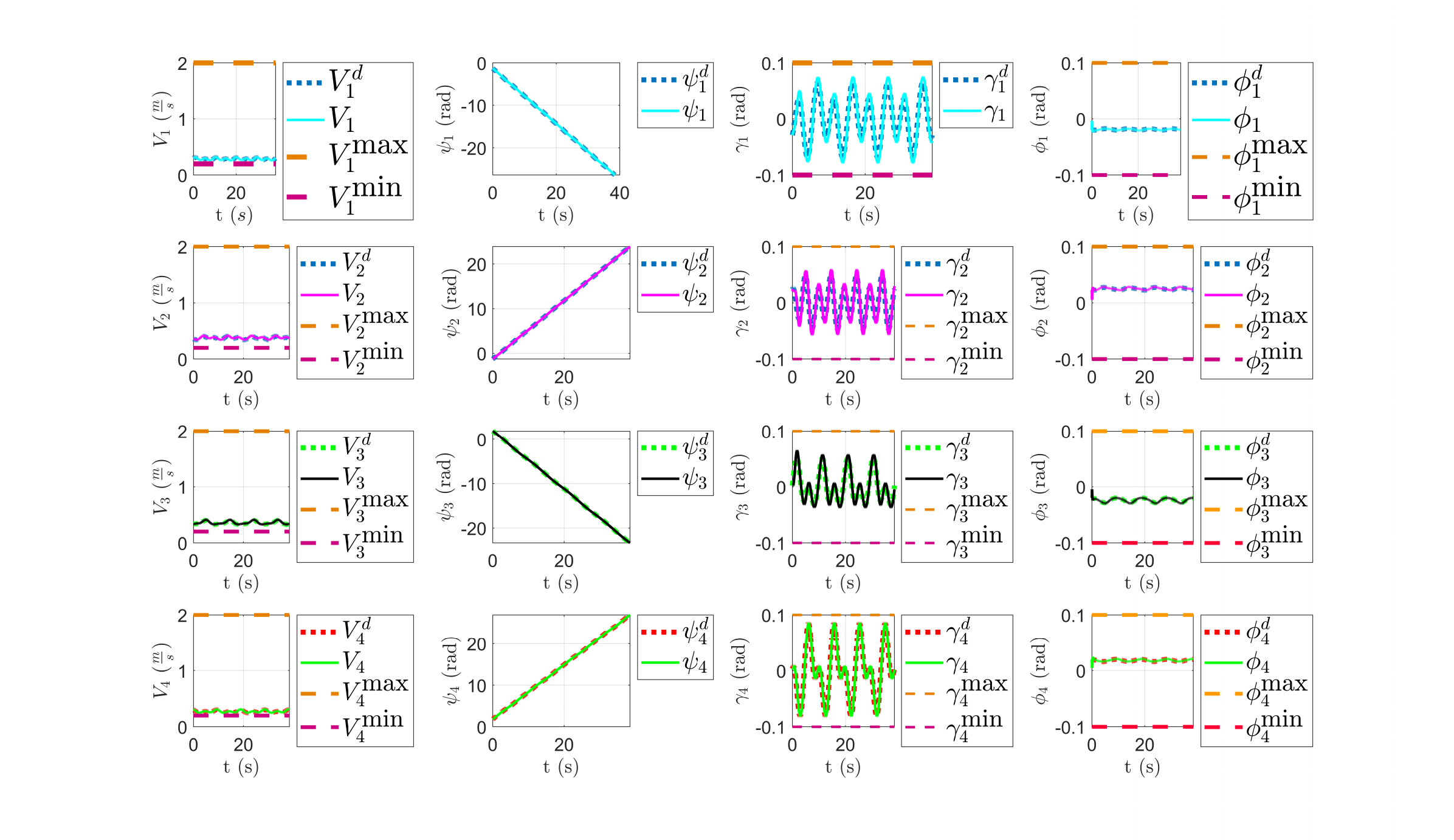}
\caption{\textcolor{black}{Time evolution of the norm of the velocity $V_{i}$, heading angle $\psi_{i}$, flight-path angle $\gamma_{i}$ and banking angle $\phi_{i}$ of the $i^{\textrm{th}}$ fixed-wing UAV in solid line while tracking planned trajectories, displayed as dotted lines.}}
\label{fig:4UAVs_tracking}
 
\end{figure*}
\begin{figure}
        \centering
        \includegraphics[width=\linewidth ]{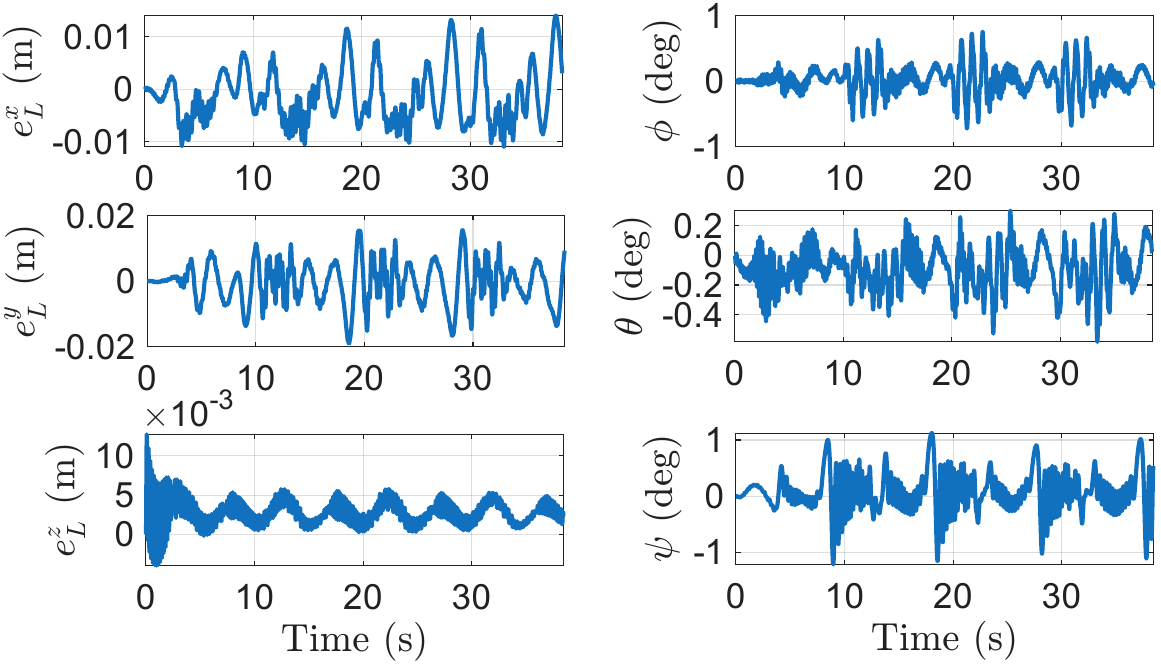}
        \caption{\textcolor{black}{Position (left column) and attitude (right column) error of the load. The proposed method allows keeping the load pose minimally perturbed while UAVs perform non-stop flight.}}
        \label{fig:payload_tracking_error}
\end{figure}
\textcolor{black}{In this section, we report simulation results obtained using a modified path planner and simulator where each carrier is modeled as a kinematics-only 3D Dubins airplane~\cite{owen2015implementing}.
Denote $V_i(A,\xi,t)$, $\psi_i(A,\xi,t)$, $\gamma_i(A,\xi,t)$, $\phi_{i}(A,\xi,t)$ respectively the norm of the velocity, heading angle, flight-path angle and banking angle of the $i^{\textrm{th}}$ UAV, $i \in [1,\dots, 4]$, parameterized with respect to the amplitude $A$ and the frequency $\xi$. We chose $\phi_{2n+2} = \frac{\pi}{2}$, $\phi_{2n+1} = 0$, $n = [0,1]$.
The UAV kinematics is 
\begin{equation}
    \label{Dubins}
       \begin{bmatrix}  \dot{p}_{R_{i}} \\\dot{\psi}_{i}\end{bmatrix}
       =\begin{bmatrix} V_{i}\cos{\psi_i}\cos{\gamma_{i}},\\
       V_{i}\sin{\psi_i}\cos{\gamma_{i}},\\
         -V_{i}\sin{\gamma_{i}}
        \\ \cfrac{g}{V_{i}}\tan{\phi_{i}}\end{bmatrix}
\end{equation}
The path planner is defined as follows:
\begin{mini!}|s|
{A,\xi}{\int_{0}^{t_f} \sum_{i=1}^{4} \kappa_i\dot{V}_i^{2}}
{}{}\label{costfcn}
\addConstraint{V_{i}^{\textrm{min}}\leq V_i \leq V_{i}^{\textrm{max}}}\label{diseq:BoundedNormAcc}
\addConstraint{\phi_{i}^{\textrm{min}}\leq \phi_{i}\leq\phi_{i}^{\textrm{max}}, \ \gamma_i^{\textrm{min}}\leq\gamma_i\leq\gamma_i^{\textrm{max}}} \label{diseq:BoundedBankAngle}
\addConstraint{{A^\textrm{min}\leq A\leq A^\textrm{max}}}\label{diseq:BoundedAmplitude}
\addConstraint{\xi = \frac{2\pi}{t_f}.}\label{eq:ClosedPaths}
\end{mini!}
The dependency of the norm of the velocity, banking angle, and flight-path angle of each aerial robot with respect to $A$, $\xi$, $t$ is not made explicit for seek of simplicity.
The cost function \eqref{costfcn} aims at minimizing the energy consumption of each flying carrier within a given period $t_f$. $\kappa_i \in \mathbb{R}$ represents weights. Constraints \eqref{diseq:BoundedNormAcc} impose a bound on the norm of the velocity of the $i^\textrm{th}$ UAV. \eqref{diseq:BoundedBankAngle}  are constraints representing physical limits of each flying carrier \cite{owen2015implementing}. The amplitude of the internal forces is constrained between a minimum and a maximum value as in \eqref{diseq:BoundedAmplitude}, whereas the frequency of the internal forces is constrained to be equal to $\frac{2\pi}{t_f}$ in order to obtain closed pseudo-elliptical-like trajectories. The paths are obtained by parameterizing (1) with respect to $A,\xi$ and plugging into the resulting expression the optimal values of both amplitude and frequency of the internal forces returned by the planner. The criteria so far described is applied to four non-stop fixed-wing UAVs sustaining a homogeneous sphere of radius \SI{1.5}{\meter}, mass \SI{1}{\kilogram} and inertia moment $0.9\bm{I}_{3} \textrm{ kgm}^2$ at desired constant equilibrium $\pLEq = [1/,1/,1/,]^\text{T}$, $\rotMatLEq = \bm{I}_3$. The length of each massless inextensible cable is equal to $\SI{1}{\meter}$. The Hamiltonian cycle chosen is $1 \to 2 \to 3 \to 4 \to 1$. The weights of the cost function $\kappa_i$ are equal to 1 $\forall i \in [1,\dots,4]$. The minimum and maximum values are summarized in the following: $\phi_i^{\textrm{max}} = \gamma_{i}^{\textrm{max}} = -\phi_{i}^{\textrm{min}} = -\gamma_{i}^{\textrm{min}} = \SI{0.1}{\radian}, V_{i}^{\textrm{max}} = \SI{2}{\meter/\second}, V_{i}^{\textrm{min}} = \SI{0.2}{\meter/\second}, A^{\textrm{min}} = \SI{0.3}{\newton}, A^{\textrm{max}} = \SI{4}{\newton}.$ We validated the applicability of the proposed method by building a more realistic simulator in MATLAB Simulink. The model of each aerial robot approximates the behavior of a closed-loop system consisting of an autopilot controller and a 3D kinematics-only fixed-wing model \cite{owen2015implementing}. The cables are modelled as linear springs with a
rigidity of $K_{c} = 500$ N/m, damping coefficient $B_c = {0.1}$ Ns/m
and a rest length of 1 m. Cables' tensions are assumed to affect only the payload, not the aerial vehicles. As a consequence, cable force compensation is not accounted for in the UAV's control design. 
The results are shown in Figure  \ref{fig:4UAVs_tracking},  and \ref{fig:payload_tracking_error}.}

\section{Experimental Results}\label{sec:exp}
\begin{figure}[t]
    \centering
\includegraphics[width=0.47\columnwidth]{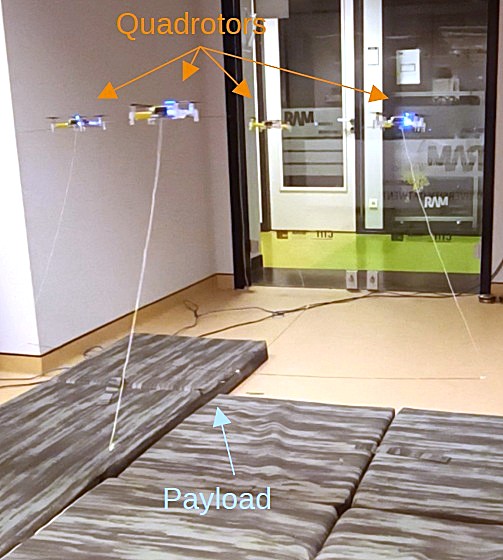}\quad\includegraphics[width=0.47\columnwidth]{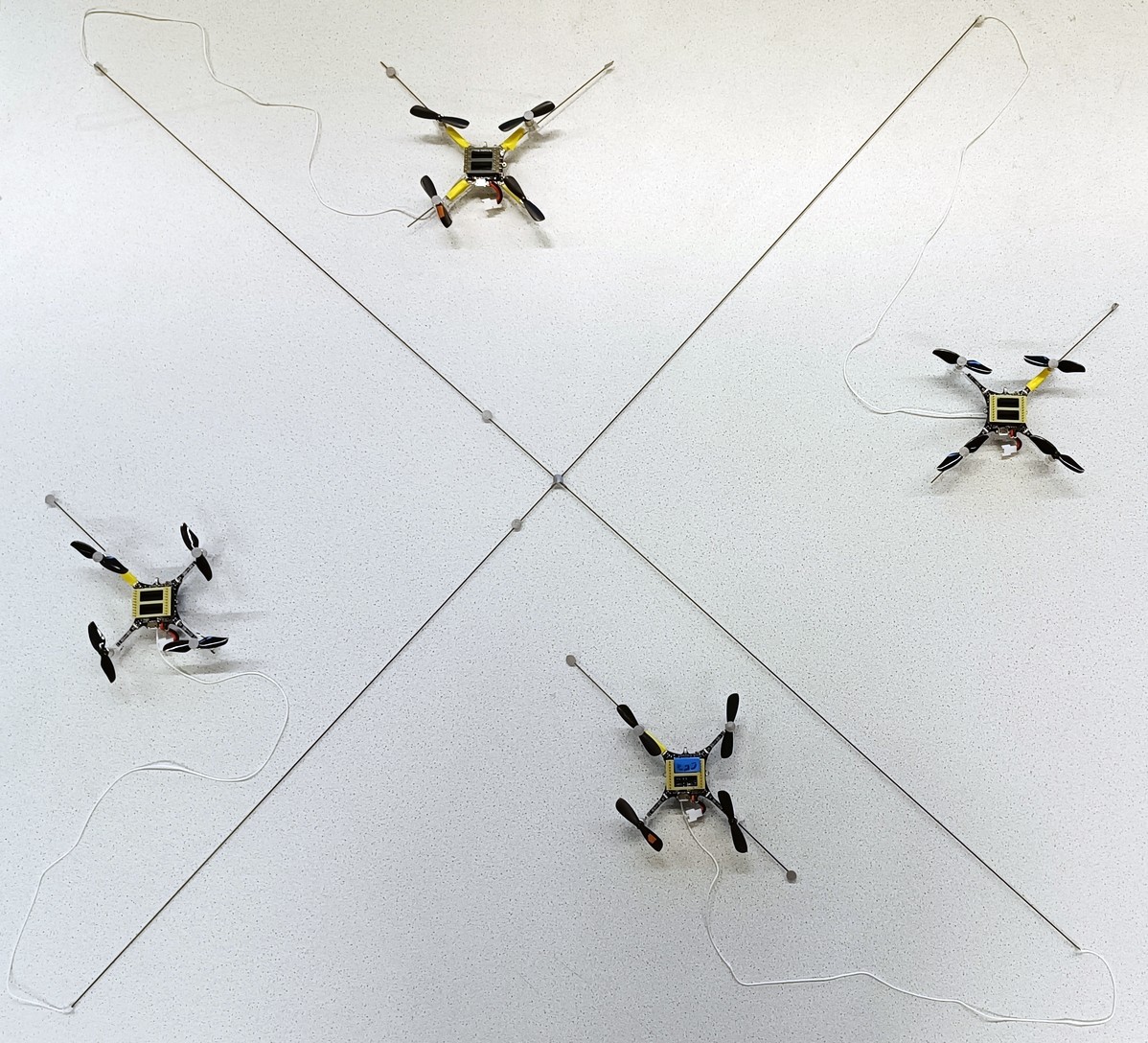}
\caption{\textcolor{black}{Experimental setup. On the right, a closer look at the robotic platform. 
Two carbon fiber bars glued to each other form the payload. Four cables are attached to the endpoints of the bars.}}\label{fig:setup}
\label{fig:setup}
\end{figure}
\begin{figure*}[!h]
    \centering
    \includegraphics[width=0.9\textwidth]  {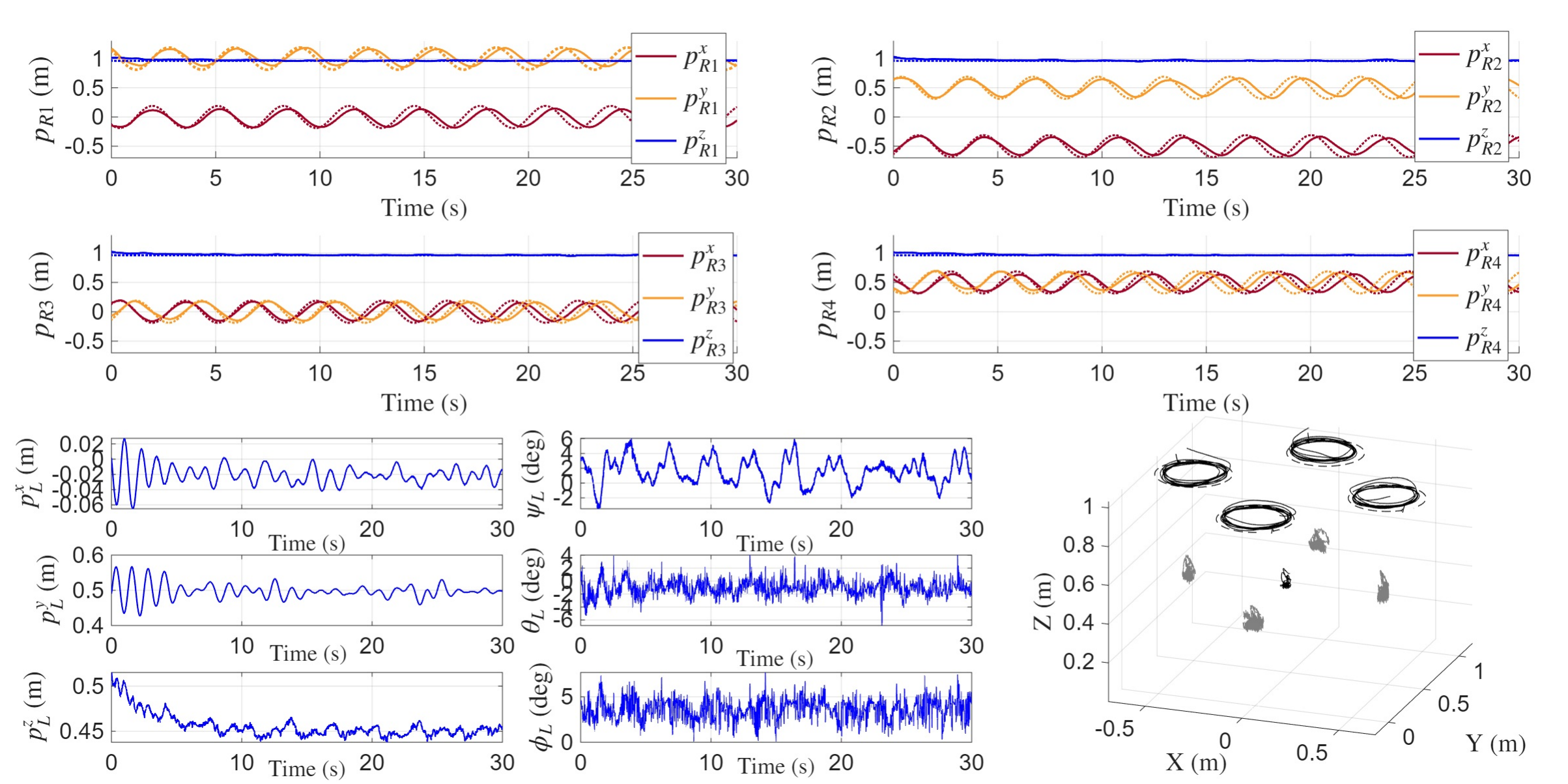} \\
    {(a) With the proposed method}
    \includegraphics[width=0.88\textwidth]{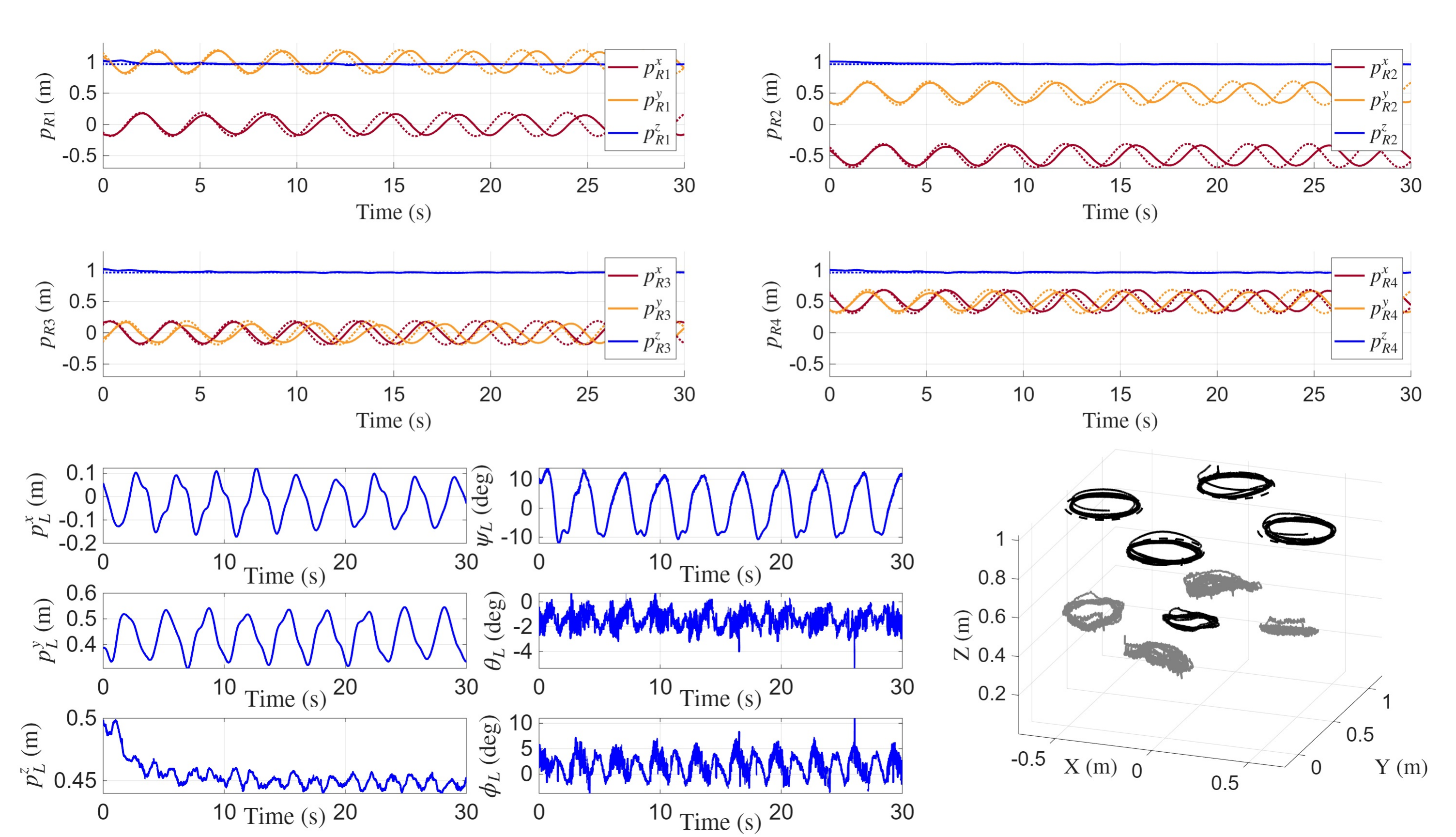}\\
    {(b) Without the proposed method}
   \caption{\textcolor{black}{Experimental results showing the position of the four UAVs (the dotted lines of the same color represent reference values), the position and attitude of the load, and an aggregated plot with the UAV reference (dotted lines) and actual (black solid lines) positions, $\pL$  in black, and the position of cable anchoring points on the load displayed in grey. (a) Proposed method: the load is kept at a static pose while the robots move. (b) One of the quadrotor trajectories is delayed: the quadrotors fail to keep the object static.}}
    \label{fig:exp_2_plot}
\end{figure*}
\begin{figure}[!h]
    \centering
    \includegraphics[width=\linewidth]{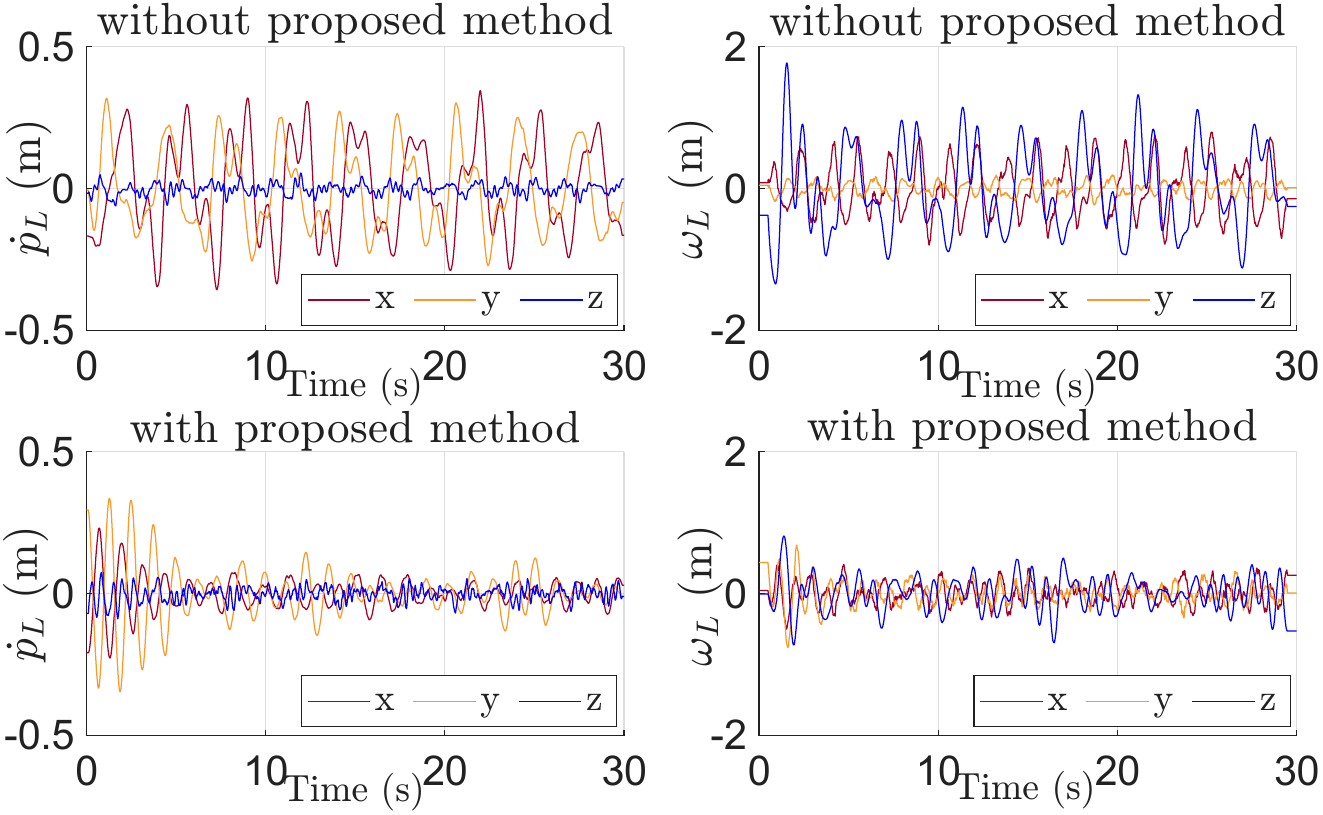}
    \caption{\textcolor{black}{Velocity of the suspended object without the proposed method (top) and with it (bottom). The proposed method effectively keeps the object velocities close to zero.}}
    \label{fig:exp_vel}
\end{figure}

This section outlines experiments conducted to validate the proposed manipulation methodology using real UAVs. The objective is to demonstrate that UAVs, by tracking the trajectories generated with the proposed method, can maintain motion while minimally disturbing the pose of the suspended load, effectively keeping it stationary.

Given the complexity of conducting experiments with continuously moving UAVs, such as fixed-wing or convertible UAVs, the tests were performed using quadrotors. The mechatronics and experimental efforts required to implement tests with fixed-wing UAVs constitute significant research in their own right and are therefore deferred to future work.

The experimental setup is shown in Fig.~\ref{fig:setup}, \textcolor{black}{where four quadrotors were tethered to the four endpoints of a cross-shaped payload via four cotton cables, each 0.50 m in length. Two 1.5\,mm-diameter and 1\,m-length carbon fiber bars perpendicularly glued to each other at the center point form the 0.001\,kg payload.}    Crazyflie 2.1 quadrotors were used, with the Crazyswarm software framework employed for data communication, multi-drone control, and state estimation~\cite{crazyswarm}. An offboard computer transmitted pre-computed reference positions (calculated offline using the proposed method) to the UAVs via radio communication. An external motion capture system (Optitrack, NaturalPoint, Inc.) provided positional feedback for UAV position tracking. 

\textcolor{black}{The payload's reference was $\pLEq=[0,\,0.5,\, 0.5] ~$\,m  and zero roll, pitch and yaw}. Experimental results, shown in Fig.~\ref{fig:exp_2_plot}(a), revealed an average load position error 
\textcolor{black}{ of $0.056$\,m with a standard deviation of $0.009$\,m. The average yaw error was $1.96$\,deg, the pitch error $1.33$\,deg, and the roll error $3.65$\,deg.} Disparities in cable lengths contributed to the roll error. It is important to note that no feedback on the load's state was implemented, so zero error was not expected. \textcolor{black}{The success of the experiment is instead that the load's pose remained nearly constant throughout the task, with standard deviations of $1.34$\,deg,  $0.91$\,deg, and $1.15$\,deg for yaw, pitch, and roll angles, respectively.}

\textcolor{black}{To further support the proposed trajectory generation method, we carried out a similar experiment where the trajectory of one of the quadrotors is delayed by half of its period, while the others' remain unchanged. The results, collected in Fig.~\ref{fig:exp_2_plot}(b), show that the load moves considerably. Figure~\ref{fig:exp_vel} collects the load's linear and angular velocity in the two experiments, i.e., with and without the proposed method; one can see how the proposed method is effective in keeping the object's velocities close to zero.}

The experimental data confirm that, \textcolor{black}{with the proposed method}, the UAVs maintained loitering flight with nonzero forward speed while holding the load at a constant pose.
\subsection{Multimedia Material}
 Multimedia videos showing  simulations and experimental results can be watched at \url{https://youtu.be/mt2vVczUwW4}  

\section{Conclusions and Future Work}\label{sec:conclusion}
This work demonstrated the compatibility of non-stop carrier flights with keeping the pose of cable-suspended loads constant for any $n\geq3$ non-stop flying carriers. We proposed and formally derived a novel method to generate coordinated trajectories, allowing the non-stop carriers to hold the load in a static equilibrium. We reported numerical and experimental results in support of the method.

Relevant future work concerns the design of an optimal planning algorithm that generates smooth non-stop trajectories accounting for obstacle avoidance. 
A  more realistic simulator will be developed to start assessing the practical applicability of the manipulation method. From an analytical point of view, the identification of families of trajectories compatible with the considered problem and the rigorous study of how the method's parameters affect such trajectories remain an open and interesting point. \textcolor{black}{Another interesting direction for extending the work is studying the most general internal-force allocation that does not restrict the internal force of each cable on a 2D plane.}  Eventually, experimental tests will be carried out on non-stop flying carriers after considering possible necessary \textcolor{black}{hardware and software} vehicle adaptations. \textcolor{black}{The energy efficiency of the proposed manipulation concept will then be assessed experimentally.}

\bibliographystyle{ieeetr}
\bibliography{biblio.bib}

@STRING{tra		= "IEEE Trans. on Robotics and Automation"}

@STRING{tr		= "IEEE Trans. on Robotics"}

@STRING{tm		= "IEEE/ASME Trans. on Mechatronics"}

@STRING{ral			= "IEEE Robotics and Automation Letters"}

@STRING{jfr		= "Journal of Field Robotics"}

@STRING{ar		= "Autonomous Robots"}

@STRING{jnls			= "Journal of Nonlinear Science"}

@STRING{icra		= "{IEEE} Int. Conf. on Robotics and Automation"}

@STRING{iros		= "IEEE/RSJ Int. Conf. on Intelligent Robots and Systems"}

@STRING{med		= "Mediterranean Conf. on Control and Automation"}

@STRING{isrr		=  "Int. Symp. on Robotics Research"}

@STRING{icuas	= "Int. Conf. on Unmanned Aircraft Systems"}

@article{williams2009dynamics,
  title={Dynamics of towed payload system using multiple fixed-wing aircraft},
  author={Williams, Paul and Ockels, Wubbo},
  journal={Journal of guidance, control, and dynamics},
  volume={32},
  number={6},
  pages={1766--1780},
  year={2009}
}

@article{quenneville2023experimental,
  title={Experimental demonstration of the lifting capability of a towed payload using multiple fixed-wing uavs},
  author={Quenneville, Samuel and Th{\'e}rien, Francis and Verrette, Jessy and Rancourt, David and Walsh, Alex and Bigu{\'e}, Jean-Philippe Lucking and Engineer, Robotics Lead and Feyel, Philippe},
  year={2023}
}

@article{foss2026energy,
  title={Energy-Efficient Collaborative Transport of Tether-Suspended Payloads via Rotating Equilibrium},
  author={Foss, Eric and Tai, Andrew and Bosio, Carlo and Mueller, Mark W},
  journal={arXiv preprint arXiv:2603.06955},
  year={2026}
}

@article{girardello2025trajectory,
  title={Trajectory control of a suspended load with non-stopping flying carriers},
  author={Girardello, Sofia and Michieletto, Giulia and Cenedese, Angelo and Franchi, Antonio and Gabellieri, Chiara},
  journal={arXiv preprint arXiv:2510.11413},
  year={2025}
}

@article{Zelazo2015,
  author    = {Daniel Zelazo and Antonio Franchi and Heinrich H. B{\"{u}}lthoff and Paolo Robuffo Giordano},
  title     = {Decentralized Rigidity Maintenance Control with Range Measurements for Multi-Robot Systems},
  journal   = {The International Journal of Robotics Research},
  year      = {2015},
  volume    = {34},
  number    = {1},
  pages     = {105--128},
  doi       = {10.1177/0278364914546173}
}

@inProceedings{gabellieri2024existence,
author  = {C. Gabellieri and A. Franchi},
title   = {On the Existence of Static Equilibria of a Cable-Suspended Load with Non-stopping Flying Carriers},
booktitle = "2024 "#icuas,
year	  = {2024},
address   = {Chania, Greece},
month     = {June},
pages     = {638-644},
doi       = {10.1109/ICUAS60882.2024.10556930},	
}

@inproceedings{crazyswarm,
  author    = {James A. Preiss* and
               Wolfgang  H\"onig* and
               Gaurav S. Sukhatme and
               Nora Ayanian},
  title     = {Crazyswarm: {A} large nano-quadcopter swarm},
  booktitle = "2017 ICRA",
  pages     = {3299--3304},
  year      = {2017},
  doi       = {10.1109/ICRA.2017.7989376},
}

@article{ruggiero2018aerial,
  title={Aerial manipulation: A literature review},
  author={Ruggiero, Fabio and Lippiello, Vincenzo and Ollero, Anibal},
  journal=ral,
  volume={3},
  number={3},
  pages={1957--1964},
  year={2018},
}

@INPROCEEDINGS{masone2016,
  author={Masone, Carlo and Bülthoff, Heinrich H. and Stegagno, Paolo},
  booktitle="2016 "#iros, 
  title={Cooperative transportation of a payload using quadrotors: A reconfigurable cable-driven parallel robot}, 
  year={2016},
  volume={},
  number={},
  pages={1623-1630},
  doi={10.1109/IROS.2016.7759262}
}

@Article{drones8020035,
AUTHOR = {Estevez, Julian and Garate, Gorka and Lopez-Guede, Jose Manuel and Larrea, Mikel},
TITLE = {Review of Aerial Transportation of Suspended-Cable Payloads with Quadrotors},
JOURNAL = {Drones},
VOLUME = {8},
YEAR = {2024},
NUMBER = {2},
ARTICLE-NUMBER = {35},
DOI = {10.3390/drones8020035}
}

@inproceedings{sreenath2013trajectory,
  title={Trajectory generation and control of a quadrotor with a cable-suspended load-a differentially-flat hybrid system},
  author={Sreenath, Koushil and Michael, Nathan and Kumar, Vijay},
  booktitle="2013 "#icra,
  pages={4888--4895},
  year={2013}
}

@inproceedings{pereira2016slung,
  title={Slung load transportation with a single aerial vehicle and disturbance removal},
  author={Pereira, Pedro O and Herzog, Manuel and Dimarogonas, Dimos V},
  booktitle="2016 24th "#med,
  pages={671--676},
  year={2016},
}

@article{sanalitro2020full,
  title={Full-pose manipulation control of a cable-suspended load with multiple UAVs under uncertainties},
  author={Sanalitro, Dario and Savino, Heitor J and Tognon, Marco and Cort{\'e}s, Juan and Franchi, Antonio},
  journal=ral,
  volume={5},
  number={2},
  pages={2185--2191},
  year={2020},
}

@article{goodman2022geometric,
  title={Geometric control of two quadrotors carrying a rigid rod with elastic cables},
  author={Goodman, Jacob and Colombo, Leonardo},
  journal=jnls,
  volume={32},
  number={5},
  pages={65},
  year={2022}
}

@article{ollero2021past,
  title={Past, present, and future of aerial robotic manipulators},
  author={Ollero, Anibal and Tognon, Marco and Suarez, Alejandro and Lee, Dongjun and Franchi, Antonio},
  journal=tr,
  volume={38},
  number={1},
  pages={626--645},
  year={2021},
}

@article{jiang2012inverse,
  title={The inverse kinematics of cooperative transport with multiple aerial robots},
  author={Jiang, Qimi and Kumar, Vijay},
  journal=tr,
  volume={29},
  number={1},
  pages={136--145},
  year={2012},
}

@inproceedings{fink2011planning,
  title={Planning and control for cooperative manipulation and transportation with aerial robots},
  author={Fink, Jonathan and Michael, Nathan and Kim, Soonkyum and Kumar, Vijay},
  booktitle= "14th "#isrr,
  pages={643--659},
  year={2011},
  organization={Springer}
}

@article{tognon2018aerial,
  title={Aerial co-manipulation with cables: The role of internal force for equilibria, stability, and passivity},
  author={Tognon, Marco and Gabellieri, Chiara and Pallottino, Lucia and Franchi, Antonio},
  journal=ral,
  volume={3},
  number={3},
  pages={2577--2583},
  year={2018},
}

@article{michael2011cooperative,
  title={Cooperative manipulation and transportation with aerial robots},
  author={Michael, Nathan and Fink, Jonathan and Kumar, Vijay},
  journal=ar,
  volume={30},
  pages={73--86},
  year={2011},
}

@article{bernard2011autonomous,
  title={Autonomous transportation and deployment with aerial robots for search and rescue missions},
  author={Bernard, Markus and Kondak, Konstantin and Maza, Ivan and Ollero, Anibal},
  journal=jfr,
  volume={28},
  number={6},
  pages={914--931},
  year={2011},
}

@ARTICLE{9508879,
  author={Mohammadi, Keyvan and Sirouspour, Shahin and Grivani, Ali},
  journal=tm, 
  title={Passivity-Based Control of Multiple Quadrotors Carrying a Cable-Suspended Payload}, 
  year={2022},
  volume={27},
  number={4},
  pages={2390-2400},
  doi={10.1109/TMECH.2021.3102522}}

@INPROCEEDINGS{8995928,
  author={Chen, Ti and Shan, Jinjun},
  booktitle={2019 IEEE Int. Conf. on Unmanned Systems}, 
  title={Cooperative Transportation of Cable-suspended Slender Payload Using Two Quadrotors}, 
  year={2019},
  volume={},
  number={},
  pages={432-437},
  doi={10.1109/ICUS48101.2019.8995928}
}

@inproceedings{pereira2018asymmetric,
  title={Asymmetric collaborative bar stabilization tethered to two heterogeneous aerial vehicles},
  author={Pereira, Pedro O and Roque, Pedro and Dimarogonas, Dimos V},
  booktitle="2018 "#icra,
  pages={5247--5253},
  year={2018},
}

@inproceedings{bernard2009generic,
  title={Generic slung load transportation system using small size helicopters},
  author={Bernard, Markus and Kondak, Konstantin},
  booktitle="2009 "#icra,
  pages={3258--3264},
  year={2009},
}

@article{leutenegger2016flying,
  title={Flying robots},
  author={Leutenegger, Stefan and H{\"u}rzeler, Christoph and Stowers, Amanda K and Alexis, Kostas and Achtelik, Markus W and Lentink, David and Oh, Paul Y and Siegwart, Roland},
  journal={Springer Handbook of Robotics},
  pages={623--670},
  year={2016},
  publisher={Springer}
}

@article{wahba2024efficient,
  title={Efficient Optimization-based Cable Force Allocation for Geometric Control of a Multirotor Team Transporting a Payload},
  author={Wahba, Khaled and H{\"o}nig, Wolfgang},
  journal=ral,
  year={2024},
  publisher={IEEE}
}

@incollection{owen2015implementing,
  title={Implementing Dubins airplane paths on fixed-wing uavs},
  author={Owen, Mark and Beard, Randal W and McLain, Timothy W},
  booktitle={Handbook of unmanned aerial vehicles},
  pages={1677--1701},
  year={2015},
  publisher={Springer}
}

@inproceedings{gassner2017dynamic,
  title={Dynamic collaboration without communication: Vision-based cable-suspended load transport with two quadrotors},
  author={Gassner, Michael and Cieslewski, Titus and Scaramuzza, Davide},
  booktitle="2017" #icra,
  pages={5196--5202},
  year={2017},
}

@article{li2021cooperative,
  title={Cooperative transportation of cable suspended payloads with mavs using monocular vision and inertial sensing},
  author={Li, Guanrui and Ge, Rundong and Loianno, Giuseppe},
  journal=ral,
  volume={6},
  number={3},
  pages={5316--5323},
  year={2021},
}

@article{gabellieri2023equilibria,
  title={Equilibria, Stability, and Sensitivity for the Aerial Suspended Beam Robotic System Subject to Parameter Uncertainty},
  author={Gabellieri, Chiara and Tognon, Marco and Sanalitro, Dario and Franchi, Antonio},
  journal=tr,
  year={2023},
}

@article{yoshikawa1999virtual,
  title={Virtual truss model for characterization of internal forces for multiple finger grasps},
  author={Yoshikawa, Tsuneo},
  journal=tra,
  volume={15},
  number={5},
  pages={941--947},
  year={1999},
}

@book{soifer2008mathematical,
  author    = {Alexander Soifer},
  title     = {The Mathematical Coloring Book},
  year      = {2008},
  publisher = {Springer-Verlag},
  isbn      = {978-0-387-74640-1}
}
\begin{IEEEbiography}[{\includegraphics[width=1in,height=1.25in,clip,keepaspectratio]{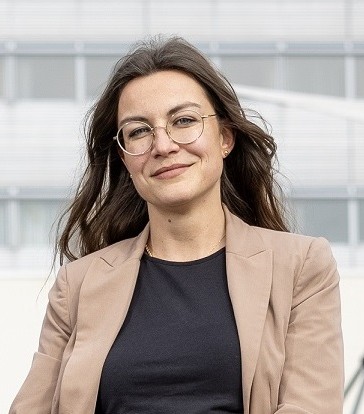}}]{Chiara Gabellieri} received her Ph.D. in Information Engineering in 2021, her MSc. in Robotics and Automation Engineering with honors in 2017, and her BSc. in Bioengineering in 2014 from the University of Pisa. She is an Assistant Professor in the Robotics and Mechatronics group at the University of Twente, in the Netherlands, where she also received the Marie Skłodowska-Curie postdoctoral fellowship in 2022. She was at LAAS-CNRS in Toulouse, France, in 2018, and a visiting Ph.D. student at the German Aerospace Center (DLR) in 2020. She is an Associate Editor (AE) for the IEEE RA-L.  She works on modeling and control for aerial robotic manipulation. 
\end{IEEEbiography}
\begin{IEEEbiography}[{\includegraphics[width=1in,height=1.25in,clip,keepaspectratio]{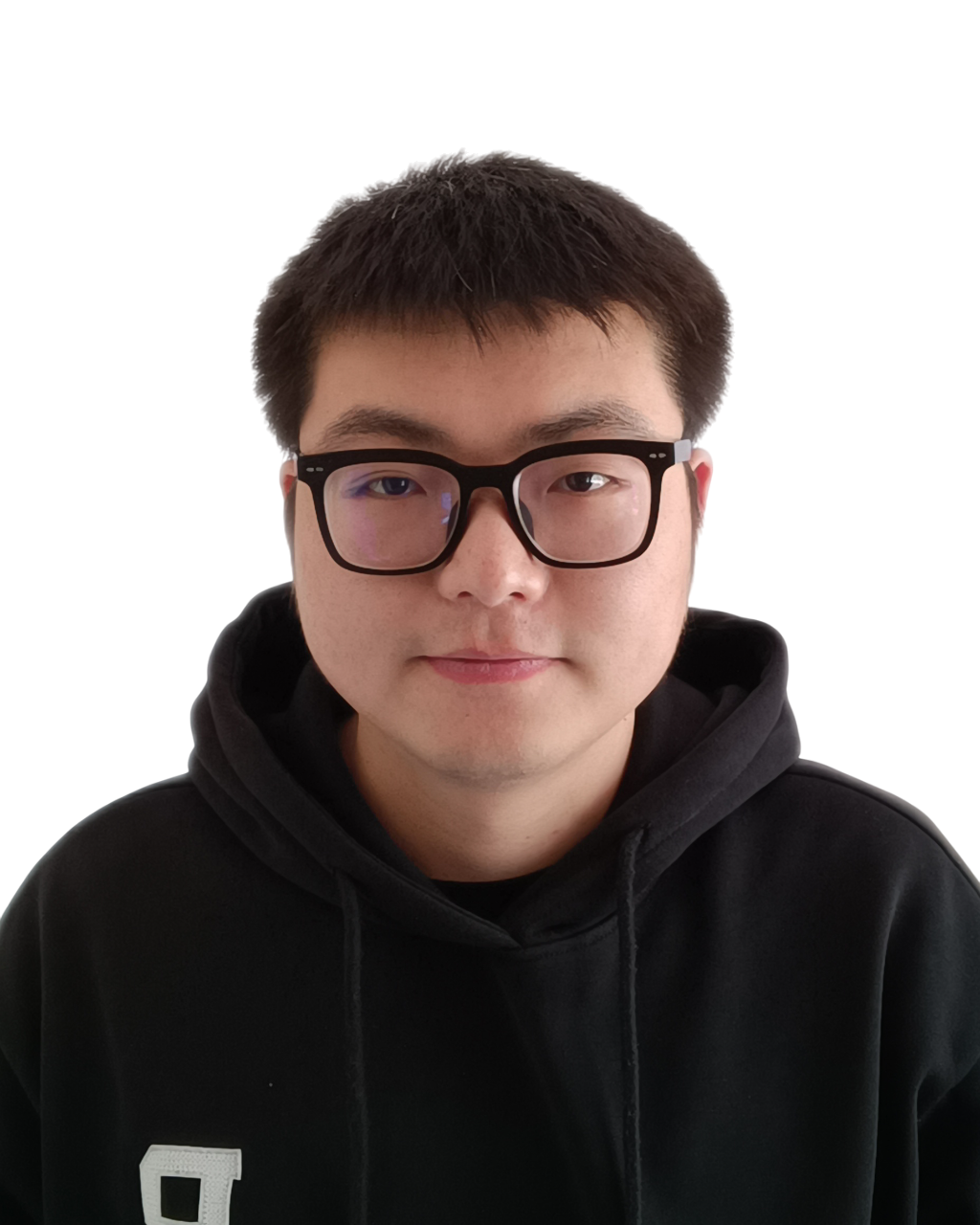}}]{Yaolei Shen} received the B.Eng. and M.Eng.~degree in mechanical engineering from Northwestern Polytechnical University, Xi'an, China, in 2019 and 2022, respectively.  He is currently a Ph.D. candidate in the Robotics and Mechatronics group at the University of Twente, the Netherlands. 

His research interests are modeling and control for aerial robotic manipulation.
\end{IEEEbiography}
\begin{IEEEbiography}[{\includegraphics[width=1in,height=1.25in,clip,keepaspectratio]{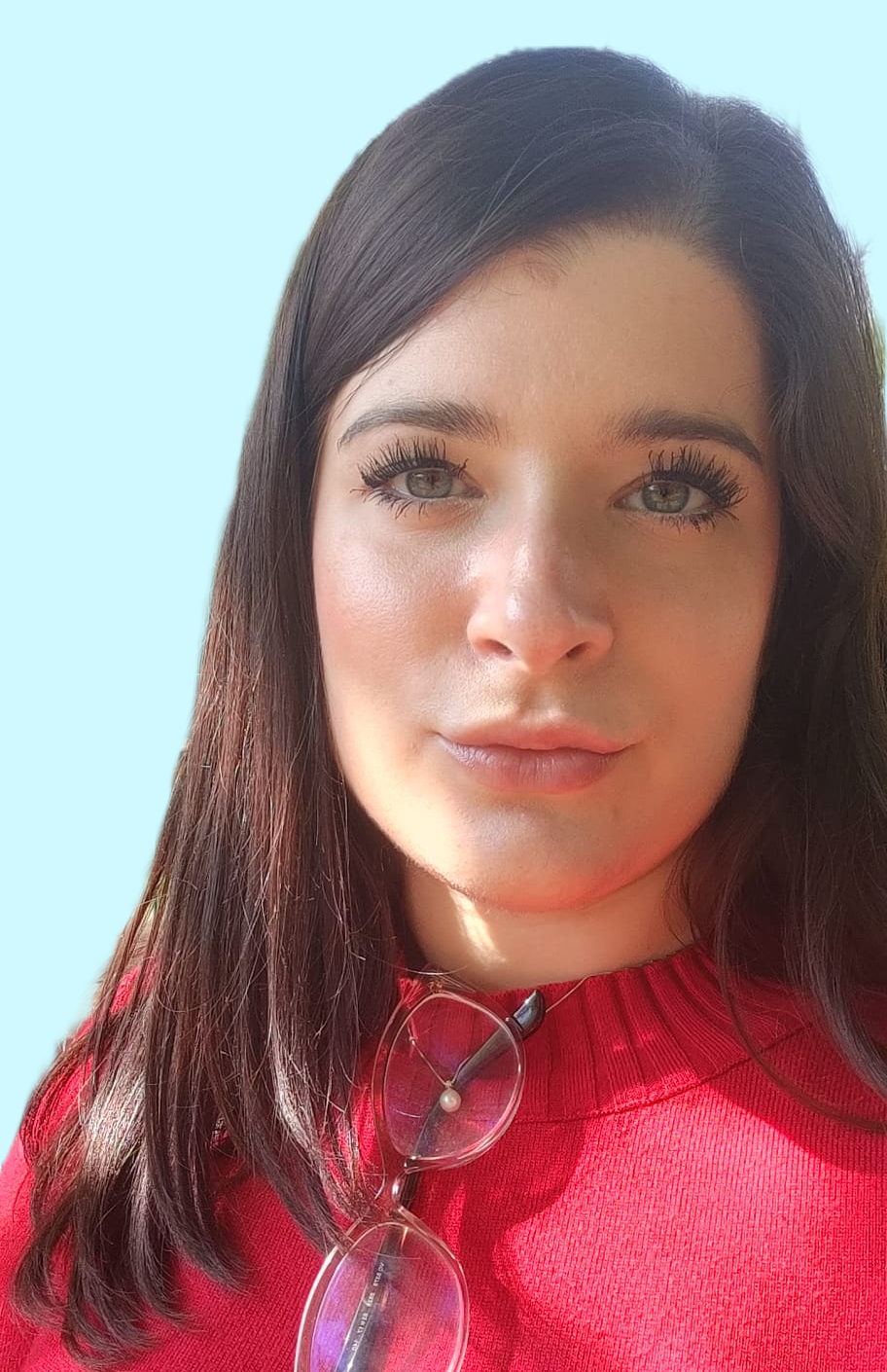}}] {Martina Paolucci} received the BSc. in Ingegneria Informatica e Automatica and MSc. in Control Engineering from Sapienza University of Rome, Italy, in 2020 and 2024, respectively.
\end{IEEEbiography}
\begin{IEEEbiography}[{\includegraphics[width=1in,height=1.25in,clip,keepaspectratio]{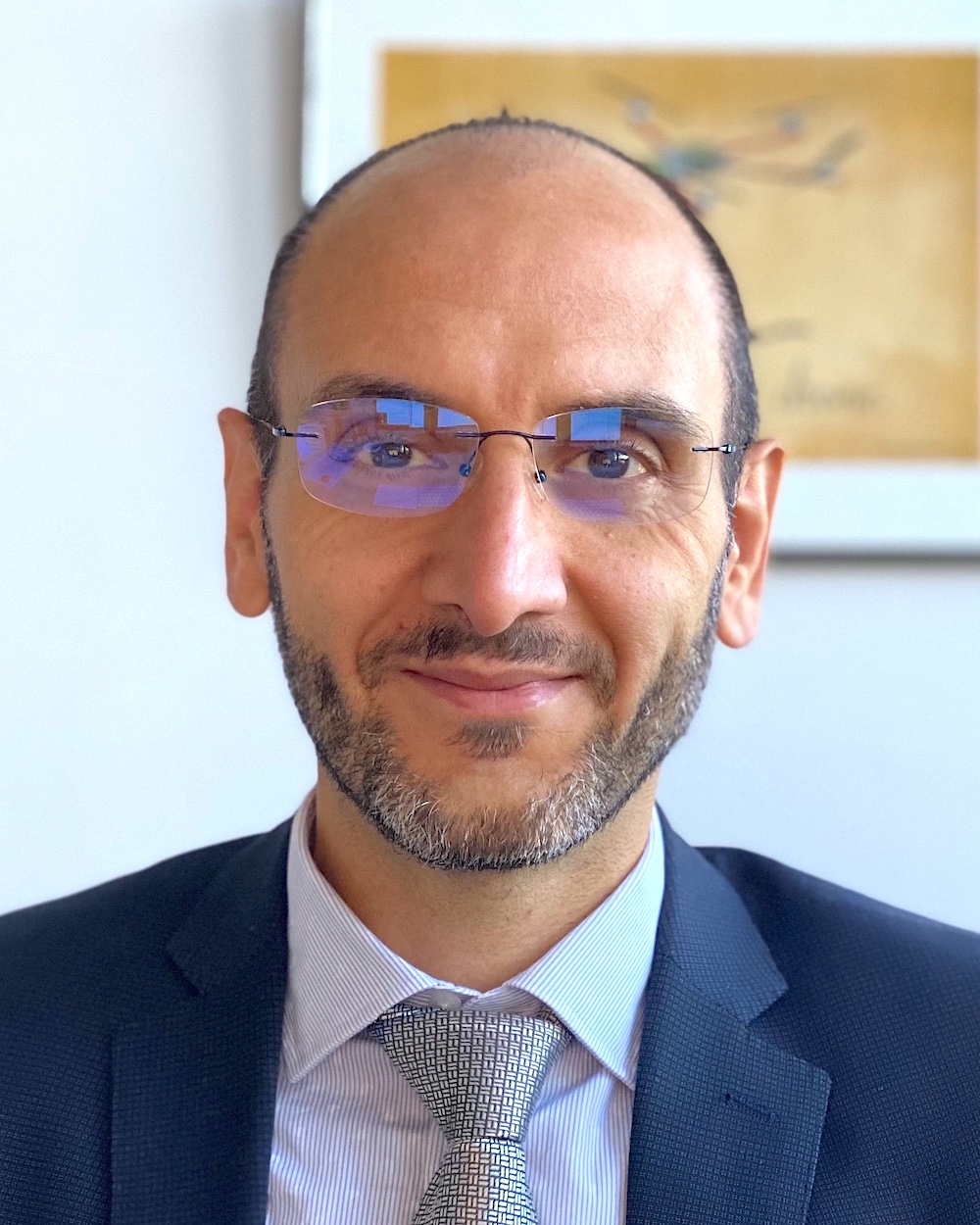}}]{Antonio Franchi} (F'23--SM'16--M'07) received the Ph.D.~degree in system engineering from the Sapienza University of Rome, Rome, Italy in 2010, and the HDR~degree in Science, from the  National Polytechnic Institute of Toulouse in 2016. From 2014 to 2019 he was a tenured CNRS researcher at LAAS-CNRS, Toulouse, France. From 2019 to 2021 he was an Associate Professor and then since 2022 he has been a Full Professor in aerial robotics control, both in the Robotics and Mechatronics group at the University of Twente. Enschede, The Netherlands. Since 2023 he is also a Full Professor in the Department of Computer, Control and Management Engineering at Sapienza University of Rome.
He co-authored more than 180 publications in peer-reviewed international journals, books, and conferences on design, control, and estimation for robotic systems applied to multi-robot systems and aerial robots.
\end{IEEEbiography}
\end{document}